\documentclass[11pt]{article} % For LaTeX2e

\usepackage[utf8]{inputenc} % allow utf-8 input
\usepackage[T1]{fontenc}    % use 8-bit T1 fonts
\usepackage[%dvips,
            CJKbookmarks=true,
            bookmarksnumbered=true,
            bookmarksopen=true,
            colorlinks=true,
            citecolor=red,
            linkcolor=blue,
            anchorcolor=red,
            urlcolor=blue
            ]{hyperref}
\usepackage{url}            % simple URL typesetting
\usepackage{booktabs}       % professional-quality tables
\usepackage{amsfonts}       % blackboard math symbols
\usepackage{nicefrac}       % compact symbols for 1/2, etc.
\usepackage{microtype}      % microtypography
\usepackage{xcolor}         % colors
\usepackage{amsmath,amssymb,amsthm}
\usepackage{algorithm,algorithmic}
\usepackage{dsfont}
\usepackage{enumitem}
\usepackage{multicol}
\usepackage{multirow}
\usepackage{graphicx}
\usepackage{natbib}

\makeatletter

\allowdisplaybreaks

\newcommand{\rf}{\mathrm{ref}}
\newcommand{\argmax}{\mathrm{argmax}}
\newcommand{\argmin}{\mathrm{argmin}}

\newcommand{\ind}{\mathds{1}}

\newcommand{\bmu}{\boldsymbol{\mu}}
\newcommand{\bphi}{\boldsymbol{\phi}}

\newtheorem{theorem}{Theorem}[section]
\newtheorem{lemma}[theorem]{Lemma}
\newtheorem{corollary}[theorem]{Corollary}

\newtheorem{assumption}[theorem]{Assumption}
\newtheorem{example}[theorem]{Example}

\theoremstyle{definition}
\newtheorem{definition}[theorem]{Definition}

\theoremstyle{remark}
\newtheorem{remark}[theorem]{Remark}

\title{Provably Efficient Offline-to-Online Value Adaptation with General Function Approximation}

\author{%
Shangzhe Li \\ UNC Chapel Hill \and Weitong Zhang\thanks{Correspondence to: \texttt{\{shangzhe, weitongz\}@unc.edu}} \\ UNC Chapel Hill
}
\date{}

\begin{document}

\maketitle

\begin{abstract}
  We study value adaptation in offline-to-online reinforcement learning under general function approximation. Starting from an imperfect offline pretrained $Q$-function, the learner aims to adapt it to the target environment using only a limited amount of online interaction. We first characterize the difficulty of this setting by establishing a minimax lower bound, showing that even when the pretrained $Q$-function is close to optimal $Q^\star$, online adaptation can be no more efficient than pure online RL on certain hard instances. On the positive side, under a novel structural condition on the offline-pretrained value functions, we propose O2O-LSVI, an adaptation algorithm with problem-dependent sample complexity that provably improves over pure online RL. Finally, we complement our theory with neural-network experiments that demonstrate the practical effectiveness of the proposed method.
\end{abstract}

\section{Introduction}
Reinforcement learning (RL) has long been regarded as an effective framework for solving complex sequential decision-making problems across a wide range of domains, including robotic control \citep{tang2025deep}, healthcare \citep{CORONATO2020101964}, and language model post-training \citep{ouyang2022training}. Standard online RL paradigms typically rely on extensive interactions with the environment to optimize the policy. 
To reduce the need for online interaction, offline RL has emerged as an important paradigm in which the algorithm learns from a dataset collected by an existing behavioral policy instead of directly accessing the environment \citep{fujimoto2021minimalist,kostrikov2021offline,kumar2020conservative}. However, the performance of the purely offline RL approaches is usually limited by the coverage and quality of the collected dataset \citep{gulcehre2022empirical,zhan2022offline}. In particular, when data coverage is poor or when the online environment shifts from the offline environment, the learned value function can become highly inaccurate over parts of the state-action space. 

To address this issue, recent empirical approaches consider the offline-to-online adaptation setting, where an (offline) pretrained RL agent can improve through a limited amount of online interaction~\citep{yu2023actor,nakamoto2023cal,huang2025offline,ball2023efficient,zhou2024efficient}. On the theoretical side, prior works have analyzed online policy adaptation starting from a suboptimal initial policy \citep{xie2021policy}, as well as hybrid RL settings that combine offline data with online exploration \citep{li2023reward,song2022hybrid} or finetune a pretrained transition kernel \citep{qu2024hybrid}. However, current theoretical progress still leaves a significant gap from empirical implementations, which typically leverage offline-pretrained value functions or policies in a model-free manner~\citep{nakamoto2023cal,zhou2024efficient}. This gap motivates the following question:
\begin{center}
    \textit{How can we theoretically characterize the adaptation of value functions in offline-to-online RL?}
\end{center}

To fully characterize this problem, three aspects must be addressed. First, it is necessary to understand the inherent difficulty of adapting an arbitrary offline value function through online interaction. Second, it is important to identify the structural conditions under which an algorithm can provably benefit from the offline-pretrained value function, compared with learning from scratch through purely online RL. Third, it is essential to characterize how the performance of the adaptation algorithm is impact by the potential misspecification of the offline-pretrained value function.

In this paper, we answer this question by developing a model-free adaptation framework in which the algorithm starts with a pretrained state-action value function ($Q$-function) obtained through offline pretraining and adapts its policy using online interactions. Under this framework, we reveal the fundamental difficulty of general offline-to-online adaptation algorithms through an information-theoretic lower bound. Importantly, we also propose an offline-to-online adaptation algorithm, \textsc{O2O-LSVI}, which enjoys provably efficient sample complexity under a common structural condition on the pretrained state-action value function. Our contributions can be summarized as follows:
\begin{itemize}[leftmargin=*]
\item We formulate offline-to-online adaptation as a value-function adaptation framework and establish an information-theoretic lower bound, showing that without structural conditions on the offline-pretrained value functions, adaptation can be fundamentally as hard as training from scratch using online interaction, even when the pretrained function is sufficiently close to the optimal $Q$-function.  
\item We propose a novel structural condition, called ${\beta}$-separable value gaps, that reflects the empirical success of offline-to-online RL algorithms. Under condition of ${\beta}$-separable value gaps, the offline-pretrained $Q$-function is close to the optimal $Q$-function in some regions, while exhibiting significantly larger errors in others. 
\item We propose an offline-to-online adaptation algorithm called \textsc{O2O-LSVI} for ${\beta}$-separable value gaps. We show that under general function approximation, \textsc{O2O-LSVI} is provably more efficient than purely online RL. The key innovations are the algorithmic design and new analysis technique of a novel criterion that determines whether to use optimistic estimation or instead leverage the pretrained value function. Empirically, we demonstrate that the \textsc{O2O-LSVI} can be efficiently adapt a various class of offline-pretrained value functions with significant performance gain.
\end{itemize}

\section{Related Works}

\noindent\textbf{Theoretical Analysis on Offline and Online RL.} There has been extensive theoretical study of both online and offline reinforcement learning. For RL with direct online interaction with the environment, prior work has established regret and sample complexity guarantees for tabular MDPs \citep{he2021nearly,zanette2019tighter,JMLR:v11:jaksch10a}, MDPs with linear function approximation \citep{jin2023provably,he2021logarithmic,zhang2021reward,zhang2024achieving}, and MDPs with general function approximation \citep{agarwal2023vo,zhao2024nearly,wang2020reinforcement}. For RL with an offline dataset, existing theory primarily focuses on pessimism based methods and establishes suboptimality guarantees under various coverage conditions \citep{yin2021towards,yin2022near,xie2021bellman,uehara2021pessimistic}. There is also a line of work on hybrid RL in tabular MDPs, where the learner has access to both offline data, potentially collected under dynamics shift, and online interaction, with the goal of understanding when and how offline data can improve the sample efficiency of online learning \citep{li2023reward,qu2024hybrid,song2022hybrid,xie2021policy}.

\noindent\textbf{RL with General Function Approximation.} A series of theoretical work has studied reinforcement learning under general function approximation. To understand the fundamental statistical limits of this setting and shed light on the behavior of practical RL algorithms, prior works have introduced a range of complexity measures for general function classes, including Bellman Rank~\citep{jiang2017contextual}, Witness Rank~\citep{sun2019model}, Eluder Dimension~\citep{10.5555/2999792.2999864}, Bellman Eluder Dimension~\citep{jin2021bellman}, Decision Estimation Coefficient (DEC)~\citep{foster2021statistical}, Admissible Bellman Characterization~\citep{chen2022general}, Generalized Eluder Dimension~\citep{agarwal2023vo}, and Generalized Eluder Coefficient (GEC)~\citep{zhong2022gec}.%, among others.

\noindent\textbf{Offline-to-Online RL.} Many prior works have investigated how to efficiently leverage offline data in online reinforcement learning, as well as how to fine-tune offline-pretrained policies through online interaction in the offline-to-online RL setting. In the first line of work, existing methods initialize the replay buffer with offline data \citep{ball2023efficient,song2022hybrid,zhou2023offline,NEURIPS2024_d9251dc2}, achieving both empirical improvements and provable theoretical guarantees. The second line of work~\citep{nakamoto2023cal,yu2023actor,zhou2024efficient} study how to directly adapt or calibrate an offline-pretrained Q-function to the online environment in order to demonstrate substantial empirical gains. Our work is primarily concerned with this second setting, in which an imperfect offline-pretrained Q-function is given and the objective is to adapt it using online interactions. \looseness=-1

\section{Preliminaries}
We consider the setting of episodic Markov Decision Processes (MDPs) $\langle\mathcal{S},\mathcal{A},H,P,r\rangle$, where $\mathcal{S}$ and $\mathcal{A}$ are the state and action spaces. We denote $H$ as the horizon length. $P=\{P_h\}_{[H]}$ is the transition dynamics where $P_h:\mathcal{S}\times\mathcal{A}\rightarrow\Delta_{\mathcal{S}}$ is the transition probability for the future state given current state and action. The reward is denoted by $r=\{r_h\}_{[H]}$ where $r_h:\mathcal{S}\times\mathcal{A}\rightarrow[0,1]$. We use a fixed initial state $s_1$; an initial-state distribution is handled by averaging. For stochastic rewards, $r_h(s,a)$ denotes the conditional mean and $R_h^k\in[0,1]$ denotes the observed reward.

A policy of an agent interacting with this MDP is $\pi:\mathcal{S}\times[H]\rightarrow\mathcal{A}$, where $a=\pi_h(s)$ is an action the policy takes at timestep $h$. For each $h\in[H]$, there exists a state value function $V^\pi_{h}:\mathcal{S}\rightarrow\mathbb{R}$ under policy $\pi$, representing the expected future rewards from the current state $s$, defined as:
\begin{align*}\textstyle{
    V^\pi_{h}(s)=\mathbb{E}\left[\sum_{h'=h}^Hr_{h'}\left(s_{h'},\pi_{h'}(s_{h'})\right)| s_h=s\right],\qquad \forall s\in\mathcal{S},h\in[H].}
\end{align*}
Furthermore, we can also define the state-action value function $Q^\pi_{h}(s,a):\mathcal{S}\times\mathcal{A}\rightarrow\mathbb{R}$ under a policy $\pi$, which depicts the expected future rewards given current state-action pairs $(s,a)$:
\begin{align*}\textstyle{
    Q^\pi_{h}(s,a)=r_h(s,a)+\mathbb{E}\left[\sum_{h'=h+1}^Hr_{h'}\left(s_{h'},\pi_{h'}(s_{h'})\right)\,\middle|\,s_h{=}s,a_h{=}a\right],\ \forall (s,a, h)\in\mathcal{S}\times\mathcal{A} \times [H].}
\end{align*}
Within this formulation, the optimal policy $\pi^\star$ can obtain the optimal value function $V^\star_h(s)=\sup_\pi V^\pi_h(s)$ for every $h\in[H]$ and $s\in\mathcal{S}$. The Bellman equation is therefore written by
\begin{align*}
Q^\pi_h(s,a)=r_h (s,a)+[P_h V^\pi_{h+1}](s, a),\quad V_h^\pi(s)=Q_h^{\pi}(s,\pi_h(s)),\quad V_{H+1}^\pi(s)=0,
\end{align*}
for a policy $\pi$ and for all $(s,a,h)\in\mathcal{S}\times\mathcal{A}\times[H]$, where we use the shorthand \[[P_h V_{h+1}](s,a):=\mathbb{E}_{s'\sim P_h(\cdot|s,a)}V_{h+1}(s').\] Similarly, the Bellman optimality equation is:
\begin{align*}
Q^\star_h(s,a)=r_h(s, a)+[P_h V^\star_{h+1}](s,a),\quad V_h^\star(s)=\max_{a\in\mathcal{A}}Q_h^{\star}(s,a),\quad V_{H+1}^\star(s)=0,
\end{align*}
for all $(s,a,h)\in\mathcal{S}\times\mathcal{A}\times[H]$. Through the online interaction, the agent would like to minimize the pseudo regret defined by the cumulative gap between optimal value function and the value function of the output policy $\pi_k$ for $k \in [K]$ rounds, i.e.,
\begin{align}
\label{eqn:regret}\textstyle{
\mathrm{Regret}(K)=\sum_{k=1}^K\left[V^{\star}_{1}(s_1)-V^{\pi^k}_{1}(s_1)\right].}
\end{align}

\noindent \textbf{RL with General Function Approximation.} We consider model-free reinforcement learning with general function approximation. Specifically, we assume the Q-function in the algorithm is within the general function class $\mathcal{F}$. Additionally, we follow the definition of generalized Eluder dimension \citep{agarwal2023vo} to measure the complexity of exploration in the MDP:
\begin{definition}[Generalized eluder dimension]
\label{def:generalized-eluder}
Let $\lambda>0$ and $h\in[H]$, a sequence of state-action pairs $Z_h=\{z_{i,h}=(s_h^i,a_h^i)\}_{i\in[K]}$ be given. The generalized eluder dimension of a function class $\mathcal{F}_h:\mathcal{S}\times\mathcal{A}\rightarrow[0,H]$ with respect to $\lambda$ is defined by $\dim_{K}(\mathcal{F}_h):=\sup_{|Z_h|=K}\dim(\mathcal{F}_h,Z_h)$, where:
\begin{align*}
\dim(\mathcal{F}_h,Z_h)&:=\textstyle{\sum_{i=1}^K}\min\left(H^2,D^2_{\mathcal{F}_h}(z_{i,h};z_{[i-1],h})\right),\\
D^2_{\mathcal{F}_h}(z_{i,h};z_{[i-1],h})&:=\sup_{f_1,f_2\in\mathcal{F}_h}\frac{\left(f_1(z_{i,h})-f_2(z_{i,h})\right)^2}{\sum_{l\in[i-1]}\left(f_1(z_{l,h})-f_2(z_{l,h})\right)^2+\lambda},
\end{align*}
we denote $\dim_{K}(\mathcal{F})=H^{-1}\sum_{h\in[H]}\dim_{K}(\mathcal{F}_h)$ for simplicity.
\end{definition}
\begin{remark}
The generalized Eluder dimension in Definition~\ref{def:generalized-eluder} is the unweighted ($\sigma\equiv1$) uncertainty-sum complexity. Its comparison with the standard Eluder dimension~\citep{NIPS2013_41bfd20a} involves scale-dependent logarithmic factors; we retain the generalized quantity directly. For simplicity of presentation and to maintain focus on offline-to-online adaptation, we do not use variance-weighted regression. Our analysis instead uses concentration for nonnegative predictable costs; pursuing sharper horizon dependence using variance-aware techniques~\citep{zhao2024nearly} is left to future work.
\end{remark}
We further assume the bonus oracle used in our algorithm under the general function approximation developed in \citep{zhao2024nearly}:
\begin{assumption}[Bonus oracle]\label{def:oracle}
We assume that a computable function $\bar D^2_{\mathcal{F}}(z,z_{[k-1]})$ satisfies \[D^2_{\mathcal F}(z,z_{[k-1]})\le\bar D^2_{\mathcal F}(z,z_{[k-1]})\le C D^2_{\mathcal F}(z,z_{[k-1]}),\]where $C$ is an absolute constant and $D^2_{\mathcal{F}}(z,z_{[k-1]})$ is the uncertainty measurement in Definition~\ref{def:generalized-eluder}.
\end{assumption}
The following covering numbers are useful when analyzing the general function approximation.
\begin{definition}[Covering numbers]
\label{def:covering}
For any $\kappa>0$, we define the following covering numbers:
\begin{enumerate}[leftmargin=*]
\item For each $h\in[H]$, there exists a $\kappa$-cover $(\mathcal{F}_h)_{\kappa}\subseteq\mathcal{F}_h$ with cardinality 
$\vert(\mathcal{F}_h)_{\kappa}\vert\leq \mathcal{N}_{\kappa}(\mathcal{F}_h)$ such that, for every $f\in\mathcal{F}_h$, one can find 
$f'\in(\mathcal{F}_h)_{\kappa}$ satisfying 
$\Vert f-f'\Vert_{\infty}\leq\kappa$. 
We further define the uniform covering number of $\mathcal{F}$ at scale $\kappa$ as $\mathcal{N}_\kappa(\mathcal{F})
:=\max_{h\in[H]}\mathcal{N}_{\kappa}(\mathcal{F}_h)$.
\item Let $\mathcal{B}:\mathcal{S}\times\mathcal{A}\rightarrow\mathbb{R}$ denote a bonus function class such that, for any $0\le k\le K$, 
$z_{[k]}\in(\mathcal{S}\times\mathcal{A})^k$, the oracle defined in Definition~\ref{def:oracle}, 
$\bar D_{\mathcal{F}}(\cdot;z_{[k]})$, belongs to $\mathcal{B}$.

\item For the bonus class $\mathcal{B}$, there exists an $\kappa$-cover $(\mathcal{B})_{\kappa}\subseteq\mathcal{B}$ with size 
$\vert(\mathcal{B})_{\kappa}\vert\leq\mathcal{N}_{\kappa}(\mathcal{B})$ such that, for every $b\in\mathcal{B}$, there exists 
$b'\in(\mathcal{B})_{\kappa}$ satisfying $\Vert b-b'\Vert_\infty\leq \kappa$.
\end{enumerate}
All covers used by the theoretical algorithm are finite, fixed before online interaction, and have deterministic tie-breaking. Write $(\mathcal F_h)_\kappa$ and $(\mathcal B)_\kappa$ for their chosen representatives.
\end{definition}

The following remark reduces the notation in general function approximations into linear MDPs~\citep{jin2023provably} where the episode MDP with horizon $H$ if the transition kernel and reward are
\begin{align}\textstyle{
P_h(\cdot| s,a)=\sum_{i=1}^d \boldsymbol \phi^{(i)}(s,a)\boldsymbol \mu_h^{(i)}(\cdot), \quad r_h(s,a)=\boldsymbol\phi(s,a)^\top\boldsymbol\theta_h \in [0, 1],\quad \boldsymbol \phi(s, a), \boldsymbol\theta \in \mathbb R^d} \notag 
\end{align}
\begin{remark}[Reduction to linear MDPs]
In a $d$-dimensional linear MDP, $Q_h^\pi=\langle\mathbf w_h^\pi,\boldsymbol\phi\rangle$~\citep{jin2023provably}. For the usual bounded linear classes with $\|\phi(s,a)\|_2\le1$ and coefficient norm at most $W=\mathcal O(H\sqrt d)$, taking the scalar ridge $\lambda=H^2$ gives $\dim_K(\mathcal F)=\widetilde{\mathcal O}(d)$, $\log\mathcal N_\kappa(\mathcal F)=\widetilde{\mathcal O}(d)$, and $\log\mathcal N_\kappa(\mathcal B)=\widetilde{\mathcal O}(d^2)$ for the exact uncertainty oracle. The latter depends on the empirical feature Gram matrix; its square is bounded above by
$\phi(s,a)^\top(\Sigma_{k-1,h}+\lambda I/(4W^2))^{-1}\phi(s,a)$,
where $\Sigma_{k-1,h}=\sum_{i<k}\phi(s_h^i,a_h^i)\phi(s_h^i,a_h^i)^\top$. Oracle computation remains subject to Assumption~\ref{def:oracle}.
\end{remark}

\section{Offline-to-Online Value Adaptation: Setup and Hardness}
We study offline-to-online value adaptation, where the agent's only input from offline pretraining is a state-action value function $Q_{\rf}$, and adaptation proceeds through online interaction.

This framework is \emph{pretrain-agnostic}: unlike approaches that require knowledge of the source transition kernel~\citep{qu2024hybrid}, behavioral-policy coverage~\citep{xie2021policy}, or access to the offline dataset~\citep{song2022hybrid}, it does not require these details of how $Q_{\rf}$ was obtained. This abstraction accommodates practical offline-to-online RL~\citep{yu2023actor,ball2023efficient,huang2025offline} and domain adaptation.

Approximation error, environment shifts, and incomplete offline coverage can each render $Q_\rf$ inaccurate. Examples~\ref{ex:misspecified-reference}, \ref{ex:environment-shifts}, and \ref{ex:missing-coverage} illustrate how such structured reference errors can arise in practice. However, accuracy alone is insufficient: without additional structural assumptions, even a uniformly accurate reference may fail to improve worst-case online sample efficiency, as formalized by the following theorem.

\begin{theorem}[Minimax lower bound]
\label{thm:lb}
For any horizon $H\ge3$, dimension $d\ge3$, tolerance $\zeta\in(0,\tfrac12)$, and accuracy $0<\epsilon\le c_0\zeta H$, where $c_0>0$ is an absolute constant, there exist a pretrained value function $Q_\rf$ and a finite class of $d$-dimensional linear MDPs $\mathcal M_\zeta$ such that
\begin{align}
 |Q_{\rf,h}(s,a)-Q_h^\star(s,a;M)|\le\zeta H,
 \qquad\forall h,s,a,\ M\in\mathcal M_\zeta.
\end{align}
Any algorithm given $Q_\rf$ requires at least $\Omega(H^3d^2/\epsilon^2)$ online episodes to achieve $\epsilon$-suboptimal expected average regret, namely,
\begin{align}\textstyle{
 \mathbb E_{M,\mathrm{Alg}}\!\left[
  \frac1K\sum_{k=1}^K\bigl(V_1^\star(s_1)-V_1^{\pi^k}(s_1)\bigr)\right]\le\epsilon.}
\end{align}
The initial state is fixed, and the expectation is over the uniform distribution on this finite hard family and the learner's interaction randomness. The family may depend on $\epsilon$; the reference is common to all its members.
\end{theorem}
The construction of the hard instance and the proof are deferred to Appendix~\ref{sec:construction} and the following remarks would help understanding this negative result.
\begin{remark}
Theorem~\ref{thm:lb} yields an $\Omega(H^3d^2/\epsilon^2)$ sample complexity for the value adaptation, which can also be achieved by the purely online algorithms~\citep{he2023nearly}. This suggests that in principle it's fundamental hard to benefit from the offline pretrained state-value function in the online process, without special structural condition on the property of $Q_{\rf}$. In addition, in the tabular case, a similar result is also obtained in \citet{qu2024hybrid}, who establish the minimax lower bound $\Omega(H^3SA/\epsilon^2)$ under dynamics-shift transfer with tabular MDPs, showing that source-environment data cannot improve worst-case target-environment sample efficiency. This result also aligns with the prior minimax lower bound for online RL \citep{zhou2021nearly} for linear MDPs. 
\end{remark}
\begin{remark}
\label{rem:misspecification-diff}
One of the related, but different setting on misspecified function approximation is studied in \citet{du2019good}. In particular, Theorem~4.1 of \citet{du2019good} states that even if the $Q$-function misspecification level is small, online RL still requires $\Omega(2^H)$ trajectories to find the optimal policy. However, it worth highlighting that the settings of misspecification and value adaptation is fundamentally different. In particular, the ground truth optimal action-value function typically lies outside the assumed function class in the misspecified setting, so the learner faces an inherent realizability error. In the value adaptation setting, the target $Q$-function still lies in the function class so the algorithm may still downgrade to the purely online learning process. However, the pretrained function $Q_{\rf}$ may be biased relative to the optimal function $Q^\star$. As a result, Theorem~\ref{thm:lb} suggests that the this pretrained function might not bring fundamental benefits to the online RL procedure. 
\end{remark}

\section{Offline-to-Online Adaptation under Separable Value Gaps}
Although Theorem~\ref{thm:lb} shows that offline-to-online value adaptation cannot achieve improved sample efficiency in the most general setting, many practical RL tasks are far less challenging than these pathological instances. Motivated by the strong empirical performance of recent offline-to-online RL methods \citep{yu2023actor,ball2023efficient,huang2025offline}, we therefore turn to a more structured setting in which the value function satisfies a \textit{separable gap condition}. This structure allows us to leverage prior information about the discrepancy between the offline pretrained $Q$-function and the optimal $Q$-function. Under this setting, we develop an algorithm presented in Algorithm~\ref{alg:main-gfa} for adapting an offline pretrained $Q$-function $Q_{\rf}$ through online interaction under general function approximation, and show that it attains provably better sample complexity than pure online RL.

\subsection{Quantifying the Imperfection of $Q_\rf$}

To characterize the imperfection of the offline-pre-trained $Q_\rf$, we introduce two additional structural conditions. We begin with a notion of separable value gaps, which captures the minimum nonzero discrepancy between the reference function and the
optimal action-value function.

\begin{definition}[${\beta}$-separable value gaps]
\label{def:beta-sep}
Given the optimal action-value function $Q^\star$, we say that $Q_{\rf}$ is \emph{${\beta}$-separable} for some ${\beta}\in(0,H]$, if for any $h \in [H]$, $Q^\star_h(s,a)\neq Q_{\rf;h}(s,a)$, it holds that $\vert Q^\star_h(s,a)-Q_{\rf;h}(s,a)\vert \geq {\beta}$ for all $(s,a,h)\in\mathcal{S}\times\mathcal{A}\times[H]$.
\end{definition}

Intuitively, the ${\beta}$-separable value gaps indicates that the offline pretrained function $Q_{\rf}$ either coincides exactly with $Q^\star$ or differs from it by at least ${\beta}$. 
This structural condition reflects many empirical settings. For instance, in many RL tasks such as AntMaze in MuJoCo, different initial state might lead to distinct trajectories. As a result, the offline-pretrained $Q_{\rf}$ might estimate precisely for the covered initial states, but yields a significant estimation error for the uncovered cases. 

\begin{remark}
A similar setting has been studied in \citet{qu2024hybrid}, dubbed as ${\beta}$-separable MDPs, where the ${\beta}$-separation is defined at the TV distance between the source and the target MDPs. Our ${\beta}$-separable value gaps can be viewed as a ``model-free'' counterpart of \citet{qu2024hybrid} and reflects more general offline-to-online cases, since the shift of the transition kernel does not naturally guarantees a shift of the optimal value function as well as the optimal policy, a ${\beta}$-separable MDPs might be too pessmistic to quantify the shift. In addition, the gaps in value functions reflects the fact that the offline pretraining can still generate a biased $Q_{\rf}$ because of the imperfect coverage of the behavioral policy, even with the same transition kernel.
\end{remark}

% This condition is
% closely related to the notion of ${\beta}$-separable MDPs in
% \citet{qu2024hybrid}, where the separation is defined at the level of
% transition dynamics to characterize the degree of dynamics shift
% between two MDPs. Our definition can be viewed as an analogous
% value-based extension of their separability condition: if we interpret
% $Q_{\rf}$ as the optimal action-value function of an unknown source MDP
% $\mathcal{M}_{\rf}$, then this condition quantifies the induced
% optimality gap between the source MDP $\mathcal{M}_{\rf}$ and the
% target MDP $\mathcal{M}$. {\color{red} TODO: additional discussion on
% this definition on previous literature and provide examples.}
% Motivated by the fact that offline pre-training may produce a biased
% reference function, our goal is to design an online adaptation
% algorithm that, given access to a ${\beta}$-separable $Q_\rf$, can
% recover an $\epsilon$-suboptimal policy with improved sample
% complexity relative to pure online RL.

Given the $\beta$-separable value gaps, we then quantify the amount of the state-action pairs where the pretrained $Q_{\rf}$ is inaccurate:

\begin{assumption}[Bounded accuracy coverage]
\label{asm:bounded-coverage}
We assume the $Q_{\rf}$ satisfies the ${\beta}$-separable gap and for any policy $\pi$, the inaccuracy coverage of $Q_{\rf;h}$ is bounded by
\begin{align*}
\mathbb{E}_{(s,a)\sim d_h^\pi}\!\left[\ind\!\left(Q_{\rf;h}(s,a)\neq Q_h^\star(s,a)\right)\right]
\leq \rho,
\qquad \forall h\in[H],
\end{align*}
for some $\rho\in[0,1]$, where $d_h^\pi$ denotes the state-action
occupancy measure of policy $\pi$ at stage $h$.
\end{assumption}

Assumption~\ref{asm:bounded-coverage} bounds the visitation probability of inaccurate state-action pairs by $\rho$ at each stage under every policy; it is not a table-entry fraction. In the most general case where $Q_{\rf}$ is never accurate, Assumption~\ref{asm:bounded-coverage} holds for ${\beta} \rightarrow 0$ and $\rho = 1$. A similar assumption is made on in~\citet{qu2024hybrid} where only $B$ state-action pairs are changed in the tabular MDP setting over the whole $SA$ pairs. In the general function setting, we refer to this changing factor as $\rho$ for the potential infinite state-action pairs. 

We provide the following examples to illustrate how these errors can nevertheless satisfy the structural conditions formalized in Definition~\ref{def:beta-sep} and Assumption~\ref{asm:bounded-coverage}.
We start from the first example targeting the \textbf{\textit{misspecified function approximation}} where the reference value function $Q_{\rf}$.
% : every nonzero discrepancy is at least $\beta>0$, and the inaccurate region has occupancy at most $\rho$ at each stage under every policy.

\begin{example}[Misspecified reference function]
\label{ex:misspecified-reference}
% \normalfont
Consider an linear MDP with misspecified reward denoted by the residual $\xi$ on the reward as
% a a nonlinear reward residual $\xi$ written by
% Consider an MDP with linear transitions and a nonlinear reward residual:
\begin{align*}
P_h(s'| s,a) =\langle\boldsymbol\phi(s,a),\boldsymbol\mu_h(s')\rangle,\quad r_h(s,a) =\langle\boldsymbol\phi(s,a),\boldsymbol\theta_h\rangle+\xi_h(s,a)\in[0,1].
\end{align*}
Although $\xi$ cannot be captured by the linear MDP structure, one can still leverage feature $\phi(\cdot)$ and parameters $\mu(\cdot), \theta(\cdot)$ together with the true continuation value (in this idealized existence construction) to construct an idealized reference $Q_{\rf}$ by $Q_{\rf;h}(s,a)
=\left\langle\boldsymbol\phi(s,a),\boldsymbol\theta_h+
\textstyle{\int_{\mathcal S}}V_{h+1}^\star(s')\,\mathrm d \boldsymbol\mu_h(s')\right\rangle$. 

It can be verified that the reference function $Q_{\rf}$ satisfies $Q_h^\star(s, a)-Q_{\rf;h}(s, a)=\xi_h(s, a)$. In the case that the misspecification $\xi_h(s, a) \ge \beta$ is constrained in a subset $\mathcal S'_h\times\mathcal A$, one can easily verify that this is $\beta$-separation with their coverage by $\rho$ controlled by the visitation of subset $\mathcal S'_h\times\mathcal A$. We defer the detailed verification into Appendix~\ref{app:misspecified-reference}.

% The Bellman equation gives $Q_h^\star-Q_{\rf;h}=\xi_h$.
% Suppose $\xi_h$ vanishes outside $\mathcal S'_h\times\mathcal A$,
% satisfies $|\xi_h(s,a)|\geq\beta$ within this region, and
% $\sup_\pi\Pr_\pi(s_h\in\mathcal S'_h)\leq\rho/H$ for every $h$.
% Then both $\beta$-separation and 
% {\color{red}
% $\rho$-coverage hold.
% Such regions can be constructed by appending an independent
% Bernoulli$(\rho/H)$ flag to each next state, starting unflagged,
% and restricting the residual to flagged states.
% Using flag-independent features preserves the linear transition
% factorization. Here, the reference is misspecified, while the
% richer online function class need not be.}
\end{example}

Then the next example describe a common \textbf{\textit{environment shift}} where the agent suffers larger penalties for suboptimal actions at the `critical states' in target environment. 

\begin{example}[Environment shifts]
\label{ex:environment-shifts}
Let the source MDP denoted by $M=(P,r)$ with it's optimal $Q$ function as $Q_\rf=Q^{\star,M}$ and a deterministic optimal policy $\pi^\star$ with reward $r(s,a)\in[r_{\min},1]$ outside the zero-reward absorbing state, with $r_{\min}>0$. Denote a set of `critical states' $\mathcal S' \subset \mathcal S$ where the agent's suboptimal action will lead to an absorbing state $s_0$ with $r(s_0) = 0$ in the target MDP $P'$ with probability $\eta$. The target kernel is
\begin{align*}
    P'_h(\cdot| s,a)
    =\begin{cases}
        (1-\eta)P_h(\cdot| s,a)+\eta\delta_{s_0},
        &s \in \mathcal S', a \neq \pi^*(s),\\
        P_h(\cdot| s,a),&\text{otherwise},
    \end{cases}
\end{align*}
It can be verified that that for the `critical states', the shift of $Q$ function is controlled by 
\begin{align*}
    Q_{\rf;h}(s,a)-Q_h^{\star,M'}(s,a)
    =\eta\ind\{s\in\mathcal S',\ a\ne\pi_h^\star(s)\}[P_hV_{h+1}^{\star,M}](s,a),
\end{align*}
\end{example}

Then the next example provide a non-trivial case of missing offline coverage especially when pessimistic is applied in offline RL as common practice~\citep{xie2021bellman}.
\begin{example}[Missing offline coverage]
\label{ex:missing-coverage}
Consider an MDP with rewards in $[r_{\min},1]$, where $r_{\min}>0$, and behavioral support $\mathcal C_h=\operatorname{supp}(d_h^\mu)$.
When the optimal action is supported in every covered state in $\mathcal C_h$, under infinite offline data we can build the pessimistic estimation
\begin{align*}
Q_{\rf;h}(s,a) = \ind[(s, a) \in \mathcal C_h] \cdot Q_h^\star(s,a)
\end{align*}
Every uncovered pair has error at least $\beta=r_{\min}$, while $\max_a Q_{\rf;h}(s,a)=V_h^\star(s)$ on covered states.
Remarkably, this error does not propagate through the unseen states and we do not need to assume the covered set $\mathcal C_h$ to be isolated in the whole state-space.
\end{example}

\subsection{Proposed Algorithm}
Below we state our proposed algorithm, \textsc{O2O-LSVI}, follows the least-squares value iteration (LSVI) framework \citep{jin2023provably, he2021logarithmic, zhao2024nearly}. The algorithm consists of two main phases in each iteration. First, a reverse planning phase performs backward least-squares value updates to construct confidence bounds and update the state--action value estimates. Second, an episode-collecting phase interacts with the environment by executing the greedy policy induced by $Q_h^k$, thereby collecting new trajectories for subsequent updates.
\subsubsection{Least-Square Regression}
In \textsc{O2O-LSVI}, we estimate the Q-function via least-squares regression. In the original LSVI-UCB framework \citep{jin2023provably}, the Q-function is estimated using least-squares regression under linear function approximation. In our setting, we extend this estimation procedure to general function approximation by leveraging a generic function class $\mathcal{F}_h$, following the paradigm developed in prior works \citep{wang2020reinforcement, agarwal2023vo, zhao2024nearly}. 

Specifically, at each iteration $k$ and stage $h$, we regress toward the Bellman targets to construct upper and lower confidence approximations of $Q_h^\star$. We employ two function estimators, denoted by $\hat f_h^k$ and $\check f_h^k$, which induce the upper and lower confidence bounds ${\hat Q}_h^k$ and ${\check Q}_h^k$, respectively. The exploration bonus $b_h^k$ is constructed using a confidence radius $\alpha_k$ together with the bonus oracle in Assumption~\ref{def:oracle}. For the theoretical guarantee, regression is over the fixed cover $(\mathcal F_h)_\kappa$ and the oracle is rounded to $\bar D_{h,\kappa}^k\in(\mathcal B)_\kappa$ within $\kappa$ in sup norm; adding $\kappa$ to the rounded oracle preserves an upper bonus. This finite-resolution convention controls the actual hard-switching target class without assuming the switch is continuous. We note that the proposed algorithm can be naturally extended to integrate variance-aware techniques within \citet{agarwal2023vo, zhao2024nearly}, potentially enabling tighter analyses, yet here we retain unweighted regression and a nonnegative-cost concentration argument.

% unlike recent approaches that incorporate variance-aware weighted regression \citep{agarwal2023vo, zhao2024nearly}, our method relies on standard (unweighted) least-squares regression. While variance-weighted regression can yield sharper confidence sets and more refined statistical guarantees, our primary contribution lies elsewhere; therefore, we adopt the simpler least-squares formulation for clarity of exposition. Nevertheless, 
\subsubsection{Exploration Condition}

A key design in \textsc{O2O-LSVI} is when to directly leverage the offline pretrained $Q$-function $Q_{\rf}$ and when to perform UCB-based exploration. Intuitively, if we are sufficiently confident that $Q_{\rf}$ is accurate at a state--action pair $(s,a)\in\mathcal{S}\times\mathcal{A}$, we exploit it; otherwise, we resort to UCB exploration, similar to standard online model-free RL.

Formally, for $(s,a,h,k)\in\mathcal{S}\times\mathcal{A}\times[H]\times[K]$, we trust $Q_{\rf}$ if the following inclusion holds:
\begin{align}
\label{eqn:inclusion}
    \left[{\check Q}_h^k(s,a), {\hat Q}_h^k(s,a)\right] \subseteq \left[Q_{\rf;h}(s,a)-\tfrac{1}{2}{\beta},\, Q_{\rf;h}(s,a)+\tfrac{1}{2}{\beta}\right].
\end{align}
In the well-specified case, it can be proved that $Q_h^\star(s,a)\in[{\check Q}_h^k(s,a), {\hat Q}_h^k(s,a)]$. Combined with Eq.\eqref{eqn:inclusion}, this implies $Q_h^\star(s,a)\in[Q_{\rf;h}(s,a)-\tfrac{1}{2}{\beta},\, Q_{\rf;h}(s,a)+\tfrac{1}{2}{\beta}]$. Under the ${\beta}$-separation condition, it follows that $Q_h^\star(s,a)=Q_{\rf;h}(s,a)$ thus $Q_{\rf}$ is accurate. In the algorithm, we explicitly check \eqref{eqn:inclusion}; if it holds, we use its interval projection, and otherwise we continue to apply UCB-type exploration. Here $\Pi_{[l,u]}(q):=\min\{u,\max\{l,q\}\}$. Projection is inactive on the well-specified confidence event and safeguards the misspecified case without requiring its error level. 

\begin{algorithm}[t]
\caption{O2O-LSVI}
\label{alg:main-gfa}
\begin{algorithmic}[1]
\REQUIRE Reference $Q_\rf$, gap $\beta$, number of iterations $K$, fixed covers $(\mathcal F_h)_\kappa,(\mathcal B)_\kappa$, confidence radius $\alpha_k$.
\STATE Set $\kappa=(KH)^{-1}$, $\lambda=H^2$, and $H_h=H-h+1$.
\FOR{iterations $k=1,2,\dots,K$}
\STATE Set $s_1^k=s_1$ and $V_{H+1}^k=\check V_{H+1}^k=0$.
\FOR{step $h=H,H-1,\dots,1$}
\STATE $\hat f_h^k\leftarrow\argmin_{f\in(\mathcal F_h)_\kappa}\sum_{i<k}(f(s_h^i,a_h^i)-R_h^i-V_{h+1}^k(s_{h+1}^i))^2$
\STATE $\check f_h^k\leftarrow\argmin_{f\in(\mathcal F_h)_\kappa}\sum_{i<k}(f(s_h^i,a_h^i)-R_h^i-\check V_{h+1}^k(s_{h+1}^i))^2$
\STATE Choose $\bar D_{h,\kappa}^k\in(\mathcal B)_\kappa$ with $\|\bar D_{h,\kappa}^k-\bar D_{\mathcal F_h}(\cdot;z_{[k-1],h})\|_\infty\le\kappa$.
\STATE $b_h^k(s,a)\leftarrow\alpha_k(\bar D_{h,\kappa}^k(s,a)+\kappa)$
\STATE $\check Q_h^k(s,a)\leftarrow[\check f_h^k(s,a)-b_h^k(s,a)]_0^{H_h}$
\STATE $\hat Q_h^k(s,a)\leftarrow\max\{\check Q_h^k(s,a),[\hat f_h^k(s,a)+b_h^k(s,a)]_0^{H_h}\}$
\IF{$[\check Q_h^k(s,a),\hat Q_h^k(s,a)]\subseteq[Q_{\rf;h}(s,a)-\beta/2,Q_{\rf;h}(s,a)+\beta/2]$}
\STATE $Q_h^k(s,a)\leftarrow\Pi_{[\check Q_h^k(s,a),\hat Q_h^k(s,a)]}(Q_{\rf;h}(s,a))$
\ELSE
\STATE $Q_h^k(s,a)\leftarrow\hat Q_h^k(s,a)$
\ENDIF
\STATE $V_h^k(s)\leftarrow\max_aQ_h^k(s,a)$, $\check V_h^k(s)\leftarrow\max_a\check Q_h^k(s,a)$
\ENDFOR
\FOR{step $h=1,2,\dots,H$}
\STATE Take $a_h^k\leftarrow\argmax_aQ_h^k(s_h^k,a)$ and observe $R_h^k,s_{h+1}^k$.
\ENDFOR
\ENDFOR
\STATE Return the episode-level uniform mixture $\bar\pi$ of $\pi^1,\ldots,\pi^K$.
\end{algorithmic}
\end{algorithm}

\section{Theoretical Analysis}
In this section, we present the theoretical analysis for the \textsc{O2O-LSVI}. We start with the Bellman completeness assumption in RL with general function approximation:
\begin{assumption}[Bellman completeness]
\label{asm:bellman-completeness}
Let $\mathcal F=\{\mathcal F_h\}_{h=1}^H$, where each $f_h\in\mathcal F_h$ maps $\mathcal S\times\mathcal A$ into $[0,H-h+1]$. For every $V:\mathcal S\to[0,H-h]$, there exists $f\in\mathcal F_h$ such that
$f(s,a)=r_h(s,a)+[P_hV](s,a)$ for all $(s,a)$. We use the common output bound $L=H$.
\end{assumption}
Assumption~\ref{asm:bellman-completeness} is standard assumption used in the theoretical analysis of model-free RL with general function approximation~\citep{wang2020reinforcement,agarwal2023vo,zhao2024nearly}.
\subsection{Well-Specified $\beta$-Separable Value Gaps}
In the following theorem, we bound the regret of \textsc{O2O-LSVI} under well-specified $\beta$-separable value gaps in Definition~\ref{def:beta-sep} where $|Q_{\rf,h}(s,a)-Q_h^\star(s,a)|<\beta$ directly yields $Q_{\rf,h}(s,a)=Q_h^\star(s,a)$.

% we provide a detailed theoretical analysis of Algorithm~\ref{alg:main-gfa}. 
% In particular, we establish an upper bound on the regret defined in Eq.~\eqref{eqn:regret} when an offline pre-trained $Q$-function $Q_{\rf}$ satisfying Definition~\ref{def:beta-sep} is available. 
% Furthermore, we derive a regret bound under misspecification, where the algorithm only has access to $\tilde Q_{\rf}$ instead of the true $Q_{\rf}$. 
% Here, $\tilde Q_{\rf}$ is a noisy approximation of $Q_{\rf}$ with a maximum perturbation magnitude bounded by $\tau$. The misspecified setting better reflects practical scenarios, as offline pre-trained $Q$-functions are typically obtained via offline RL and may therefore exhibit a non-negligible suboptimality gap with respect to $Q^\star$. For our theoretical analysis, we make an additional assumption regarding the function classes:

% \begin{remark}
%     {\color{red} TODO: Add remarks on Assumption~\ref{asm:bounded-coverage}.}
% \end{remark}

% We first provide an upper bound of the regret for Algorithm~\ref{alg:main-gfa}:
\begin{theorem}
\label{thm:upper}
Under Assumptions~\ref{def:oracle}, \ref{asm:bellman-completeness}, and~\ref{asm:bounded-coverage}, use the fixed finite covers of Definition~\ref{def:covering} in Algorithm~\ref{alg:main-gfa}. For $\delta\in(0,\tfrac15)$, let $\kappa=(KH)^{-1}$, $\lambda=H^2$,
\[
 d_{\mathcal F}:=\dim_K(\mathcal F),\quad
 \mathcal N_{\mathcal F}:=\mathcal N_\kappa(\mathcal F),\quad
 \mathcal N_{\mathcal B}:=\mathcal N_\kappa(\mathcal B),\quad
 \alpha_k=c_\alpha H\sqrt{\log(8Hk^2\mathcal N_{\mathcal F}^3\mathcal N_{\mathcal B}/\delta)},
\]
for a sufficiently large absolute constant $c_\alpha$. With probability at least $1-5\delta$,
\begin{align*}
 \mathrm{Regret}(K)=\widetilde{\mathcal O}\Big(\!\min\Big\{\!
\frac{H^4d_{\mathcal F}\log(\mathcal N_{\mathcal F}\mathcal N_{\mathcal B})}{\beta}
+H^2\sqrt{\rho Kd_{\mathcal F}\log(\mathcal N_{\mathcal F}\mathcal N_{\mathcal B})},  H^2\sqrt{Kd_{\mathcal F}\log(\mathcal N_{\mathcal F}\mathcal N_{\mathcal B})}\Big\}\Big).
\end{align*}
% The additive one in $d_{\mathcal F}$ and the $e$ inside the logarithm only avoid degenerate zero-complexity expressions.
\end{theorem}
The regret result in Theorem~\ref{thm:upper} can be immediately translated into the following sampling complexity bound, as stated in the following corollary:
\begin{corollary}
\label{cor:sampling-complexity}
With the notation and assumptions of Theorem~\ref{thm:upper}, the episode-level uniform mixture returned by Algorithm~\ref{alg:main-gfa} is $\epsilon$-optimal with probability at least $1-5\delta$ using
\begin{align*}
 \widetilde{\mathcal O}\left(\min\left\{
 \frac{H^4d_{\mathcal F}\log(\mathcal N_{\mathcal F}\mathcal N_{\mathcal B})}{\beta\epsilon}
 +\frac{\rho H^4d_{\mathcal F}\log(\mathcal N_{\mathcal F}\mathcal N_{\mathcal B})}{\epsilon^2},\quad
 \frac{H^4d_{\mathcal F}\log(\mathcal N_{\mathcal F}\mathcal N_{\mathcal B})}{\epsilon^2}
 \right\}\right)
\end{align*}
online episodes. The mixture policy selects one learned policy uniformly at the start of an evaluation episode and follows it throughout that episode.
\end{corollary}
Several remarks are important to highlight our understanding of the performance of adaptation. First we connect our result in general function approximation to the linear MDPs. 
\begin{remark}
In the $d$-dimensional linear MDPs~\citep{jin2023provably}, it can be verified that 
% $\dim_K(\mathcal{F})=\mathcal{\tilde O}(d)$, $\mathcal{N}_\kappa(\mathcal{F})=\mathcal{\tilde O}(d)$, and $\mathcal{N}_\kappa(\mathcal{B})=\mathcal{\tilde O}(d)$. Under this specialization, 
the sampling complexity of Algorithm~\ref{alg:main-gfa} reduces to $\mathcal{\tilde O}(\min\{(H^4d^3)/({\beta}\epsilon)+(\rho H^4 d^{3})/\epsilon^2,(H^4d^{3})/\epsilon^2\})$. In the worst case, this yields a sampling complexity of $\mathcal{\tilde O}(H^4d^{3}/\epsilon^2)$, which matches the known result of the LSVI-UCB~\citep{jin2023provably}. The dependence on $\epsilon$ of the upper bound matches the minimax lower bound in Theorem~\ref{thm:lb}. In contrast, the sampling complexity upper bound still exhibits an $\mathcal{\tilde O}(dH)$ gap to the minimax lower bound in Theorem~\ref{thm:lb}. Closing this horizon-and-dimension gap may require a variance-aware extension of Algorithm~\ref{alg:main-gfa}, along the lines of \citet{zhou2021nearly, zhao2024nearly, he2023nearly}; this is not established here.
\end{remark}
The following remark demonstrates when the \textsc{O2O-LSVI} can improve the sample complexity
\begin{remark}
\label{rm:beta-eps}
In the linear MDP setting, the worst-case sample complexity $\widetilde{\mathcal O}(H^4d^3/\epsilon^2)$ matches that of LSVI-UCB~\citep{jin2023provably,wang2020reinforcement}. Up to logarithmic factors, the ratio of our bound to this baseline is $\min\{\epsilon/\beta+\rho,1\}$:
\begin{itemize}[leftmargin=*, noitemsep]
\item When $\epsilon\gtrsim\beta$, the online branch gives $\widetilde{\mathcal O}(H^4d^3/\epsilon^2)$, so the guarantee avoids an additional adaptation penalty when moderate accuracy suffices.
\item When $\rho>0$ and $\epsilon\lesssim\rho\beta$, the bound becomes $\widetilde{\mathcal O}(\rho H^4d^3/\epsilon^2)$. By Assumption~\ref{asm:bounded-coverage}, $\rho$ bounds the stagewise visitation probability of the inaccurate region under every policy. A smaller $\rho$ therefore gives a $\rho$-factor improvement over standard online RL in this regime.
\end{itemize}
When $\rho=0$, the adaptation term is instead $\widetilde{\mathcal O}(H^4d^3/(\beta\epsilon))$.
\end{remark}

While Remark~\ref{rm:beta-eps} targets for linear MDPs, similar claims can be made to other structures thanks to the general function approximations. We extend this comparison to~\citet{qu2024hybrid} in tabular MDPs. 
\begin{remark}
The result in Corollary~\ref{cor:sampling-complexity} shares similarities with the theoretical analysis of HTRL \citep{qu2024hybrid}, which studies policy transfer between two tabular MDPs under a ${\beta}$-separable dynamics difference condition, analogous to our Definition~\ref{def:beta-sep}. They establish a sampling complexity of $\mathcal{\tilde O}\big(\min\{H^3SA/\epsilon^2, H^3|B|/\epsilon^2 + H^2S^2A/(\sigma{\beta})^2\}\big)$ for HTRL and provide an interpretation similar to that in Remark~\ref{rm:beta-eps}. In contrast, our adaptation-overhead term scales as $\mathcal O(1/(\beta\epsilon))$, giving inverse-linear rather than inverse-quadratic dependence on our value-gap parameter for fixed $\epsilon$. The gap definitions and accuracy dependence differ, so this comparison does not imply uniform dominance. Moreover, the term $|B| < SA$ in \citet{qu2024hybrid} result characterizes the size of the region where the dynamics shift occurs. Our formulation generalizes this core idea beyond the tabular setting to continuous state–action spaces under a bounded coverage assumption.
\end{remark}
\subsection{Additional Analysis and Empirical Results}
The preceding analysis assumes a well-specified $\beta$-separable reference. To account for imperfect offline pretraining, Appendix~\ref{app:misspecified-gaps} extends our guarantees to an observed reference $\tilde Q_\rf$ whose uniform deviation from a $\beta$-separable reference is at most $\tau$. For $0<\tau\le\beta/2$, Theorem~\ref{thm:upper-mis} shows that this misspecification introduces an additional $\widetilde{\mathcal O}(HK\tau)$ term in the regret bound, without requiring the algorithm to know $\tau$.

To complement these theoretical results, we also evaluate our method on MuJoCo AntMaze and examine its certification mechanism in a controlled linear MDP, with details provided in Appendix~\ref{app:empirical-results}. On AntMaze, the neural implementation achieves final scores comparable to or higher than those of the evaluated baselines. In the linear MDP, reference switching reduces mean final cumulative pseudo-regret by $15.6\%$ relative to LSVI-UCB. The algorithm eventually accepts accurate reference entries while rejecting misspecified ones, illustrating the behavior of the certification mechanism alongside its finite-sample benefit.

\section{Conclusion}
In this paper, we studied value adaptation in offline-to-online reinforcement learning under general function approximation. We showed through a minimax lower bound that, in the worst case, adapting an imperfect pretrained $Q$-function can be no easier than learning from scratch. Under a novel structural condition on pretrained values, however, adaptation can be substantially more sample efficient. Motivated by this insight, we proposed O2O-LSVI, which achieves improved problem-dependent guarantees over pure online RL and performs effectively in practice. These results clarify both the limitations and the potential of offline pretraining for accelerating online RL.

% \section*{Acknowledgment} We appreciate the discussion with Dr. Laixi Shi from JHU.
\bibliography{references}

@article{qu2024hybrid,
  title={Hybrid transfer reinforcement learning: Provable sample efficiency from shifted-dynamics data},
  author={Qu, Chengrui and Shi, Laixi and Panaganti, Kishan and You, Pengcheng and Wierman, Adam},
  journal={arXiv preprint arXiv:2411.03810},
  year={2024}
}

@inproceedings{he2023nearly,
  title={Nearly minimax optimal reinforcement learning for linear markov decision processes},
  author={He, Jiafan and Zhao, Heyang and Zhou, Dongruo and Gu, Quanquan},
  booktitle={International Conference on Machine Learning},
  pages={12790--12822},
  year={2023},
  organization={PMLR}
}

@article{jin2023provably,
  title={Provably efficient reinforcement learning with linear function approximation},
  author={Jin, Chi and Yang, Zhuoran and Wang, Zhaoran and Jordan, Michael I},
  journal={Mathematics of Operations Research},
  volume={48},
  number={3},
  pages={1496--1521},
  year={2023},
  publisher={INFORMS}
}

@inproceedings{he2021logarithmic,
  title={Logarithmic regret for reinforcement learning with linear function approximation},
  author={He, Jiafan and Zhou, Dongruo and Gu, Quanquan},
  booktitle={International Conference on Machine Learning},
  pages={4171--4180},
  year={2021},
  organization={PMLR}
}

@inproceedings{agarwal2023vo,
  title={VOQL: Towards Optimal Regret in Model-free RL with Nonlinear Function Approximation},
  author={Agarwal, Alekh and Jin, Yujia and Zhang, Tong},
  booktitle={The Thirty Sixth Annual Conference on Learning Theory},
  pages={987--1063},
  year={2023},
  organization={PMLR}
}

@inproceedings{NIPS2013_41bfd20a,
 author = {Russo, Daniel and Van Roy, Benjamin},
 booktitle = {Advances in Neural Information Processing Systems},
 editor = {C.J. Burges and L. Bottou and M. Welling and Z. Ghahramani and K.Q. Weinberger},
 pages = {},
 publisher = {Curran Associates, Inc.},
 title = {Eluder Dimension and the Sample Complexity of Optimistic Exploration},
 volume = {26},
 year = {2013}
}

@article{zhao2024nearly,
  title={A nearly optimal and low-switching algorithm for reinforcement learning with general function approximation},
  author={Zhao, Heyang and He, Jiafan and Gu, Quanquan},
  journal={Advances in Neural Information Processing Systems},
  volume={37},
  pages={94684--94735},
  year={2024}
}

@article{wang2020reinforcement,
  title={Reinforcement learning with general value function approximation: Provably efficient approach via bounded eluder dimension},
  author={Wang, Ruosong and Salakhutdinov, Russ R and Yang, Lin},
  journal={Advances in Neural Information Processing Systems},
  volume={33},
  pages={6123--6135},
  year={2020}
}

@inproceedings{zhou2021nearly,
  title={Nearly minimax optimal reinforcement learning for linear mixture markov decision processes},
  author={Zhou, Dongruo and Gu, Quanquan and Szepesvari, Csaba},
  booktitle={Conference on Learning Theory},
  pages={4532--4576},
  year={2021},
  organization={PMLR}
}

@book{lattimore2020bandit,
  title={Bandit algorithms},
  author={Lattimore, Tor and Szepesv{\'a}ri, Csaba},
  year={2020},
  publisher={Cambridge University Press}
}

@InProceedings{yu2023actor,
  title = 	 {Actor-Critic Alignment for Offline-to-Online Reinforcement Learning},
  author =       {Yu, Zishun and Zhang, Xinhua},
  booktitle = 	 {Proceedings of the 40th International Conference on Machine Learning},
  pages = 	 {40452--40474},
  year = 	 {2023},
  editor = 	 {Krause, Andreas and Brunskill, Emma and Cho, Kyunghyun and Engelhardt, Barbara and Sabato, Sivan and Scarlett, Jonathan},
  volume = 	 {202},
  series = 	 {Proceedings of Machine Learning Research},
  month = 	 {23--29 Jul},
  publisher =    {PMLR}
}

@article{huang2025offline,
  title={Offline-to-Online Reinforcement Learning with Classifier-Free Diffusion Generation},
  author={Huang, Xiao and Liu, Xu and Zhang, Enze and Yu, Tong and Li, Shuai},
  journal={arXiv preprint arXiv:2508.06806},
  year={2025}
}

@inproceedings{ball2023efficient,
  title={Efficient online reinforcement learning with offline data},
  author={Ball, Philip J and Smith, Laura and Kostrikov, Ilya and Levine, Sergey},
  booktitle={International Conference on Machine Learning},
  pages={1577--1594},
  year={2023},
  organization={PMLR}
}

@inproceedings{jiang2017contextual,
  title={Contextual decision processes with low bellman rank are pac-learnable},
  author={Jiang, Nan and Krishnamurthy, Akshay and Agarwal, Alekh and Langford, John and Schapire, Robert E},
  booktitle={International Conference on Machine Learning},
  pages={1704--1713},
  year={2017},
  organization={PMLR}
}

@inproceedings{sun2019model,
  title={Model-based rl in contextual decision processes: Pac bounds and exponential improvements over model-free approaches},
  author={Sun, Wen and Jiang, Nan and Krishnamurthy, Akshay and Agarwal, Alekh and Langford, John},
  booktitle={Conference on learning theory},
  pages={2898--2933},
  year={2019},
  organization={PMLR}
}

@inproceedings{10.5555/2999792.2999864,
author = {Russo, Daniel and Roy, Benjamin Van},
title = {Eluder dimension and the sample complexity of optimistic exploration},
year = {2013},
publisher = {Curran Associates Inc.},
address = {Red Hook, NY, USA},
booktitle = {Proceedings of the 27th International Conference on Neural Information Processing Systems - Volume 2},
pages = {2256–2264},
numpages = {9},
location = {Lake Tahoe, Nevada},
series = {NIPS'13}
}

@article{jin2021bellman,
  title={Bellman eluder dimension: New rich classes of rl problems, and sample-efficient algorithms},
  author={Jin, Chi and Liu, Qinghua and Miryoosefi, Sobhan},
  journal={Advances in neural information processing systems},
  volume={34},
  pages={13406--13418},
  year={2021}
}

@article{foster2021statistical,
  title={The statistical complexity of interactive decision making},
  author={Foster, Dylan J and Kakade, Sham M and Qian, Jian and Rakhlin, Alexander},
  journal={arXiv preprint arXiv:2112.13487},
  year={2021}
}

@article{chen2022general,
  title={A general framework for sample-efficient function approximation in reinforcement learning},
  author={Chen, Zixiang and Li, Chris Junchi and Yuan, Angela and Gu, Quanquan and Jordan, Michael I},
  journal={arXiv preprint arXiv:2209.15634},
  year={2022}
}

@article{zhong2022gec,
  title={Gec: A unified framework for interactive decision making in mdp, pomdp, and beyond},
  author={Zhong, Han and Xiong, Wei and Zheng, Sirui and Wang, Liwei and Wang, Zhaoran and Yang, Zhuoran and Zhang, Tong},
  journal={arXiv preprint arXiv:2211.01962},
  year={2022}
}

@article{nakamoto2023cal,
  title={Cal-ql: Calibrated offline rl pre-training for efficient online fine-tuning},
  author={Nakamoto, Mitsuhiko and Zhai, Simon and Singh, Anikait and Sobol Mark, Max and Ma, Yi and Finn, Chelsea and Kumar, Aviral and Levine, Sergey},
  journal={Advances in Neural Information Processing Systems},
  volume={36},
  pages={62244--62269},
  year={2023}
}

@article{song2022hybrid,
  title={Hybrid rl: Using both offline and online data can make rl efficient},
  author={Song, Yuda and Zhou, Yifei and Sekhari, Ayush and Bagnell, J Andrew and Krishnamurthy, Akshay and Sun, Wen},
  journal={arXiv preprint arXiv:2210.06718},
  year={2022}
}

@article{zhou2023offline,
  title={Offline data enhanced on-policy policy gradient with provable guarantees},
  author={Zhou, Yifei and Sekhari, Ayush and Song, Yuda and Sun, Wen},
  journal={arXiv preprint arXiv:2311.08384},
  year={2023}
}

@inproceedings{NEURIPS2024_d9251dc2,
 author = {Tan, Kevin and Fan, Wei and Wei, Yuting},
 booktitle = {Advances in Neural Information Processing Systems},
 doi = {10.52202/079017-3815},
 editor = {A. Globerson and L. Mackey and D. Belgrave and A. Fan and U. Paquet and J. Tomczak and C. Zhang},
 pages = {120038--120077},
 publisher = {Curran Associates, Inc.},
 title = {Hybrid Reinforcement Learning Breaks Sample Size Barriers In Linear MDPs},
 volume = {37},
 year = {2024}
}

@article{zhou2024efficient,
  title={Efficient online reinforcement learning fine-tuning need not retain offline data},
  author={Zhou, Zhiyuan and Peng, Andy and Li, Qiyang and Levine, Sergey and Kumar, Aviral},
  journal={arXiv preprint arXiv:2412.07762},
  year={2024}
}

@article{du2019good,
  title={Is a good representation sufficient for sample efficient reinforcement learning?},
  author={Du, Simon S and Kakade, Sham M and Wang, Ruosong and Yang, Lin F},
  journal={arXiv preprint arXiv:1910.03016},
  year={2019}
}

@article{he2021nearly,
  title={Nearly minimax optimal reinforcement learning for discounted MDPs},
  author={He, Jiafan and Zhou, Dongruo and Gu, Quanquan},
  journal={Advances in Neural Information Processing Systems},
  volume={34},
  pages={22288--22300},
  year={2021}
}

@article{JMLR:v11:jaksch10a,
  author  = {Thomas Jaksch and Ronald Ortner and Peter Auer},
  title   = {Near-optimal Regret Bounds for Reinforcement Learning},
  journal = {Journal of Machine Learning Research},
  year    = {2010},
  volume  = {11},
  number  = {51},
  pages   = {1563--1600},
}

@inproceedings{zanette2019tighter,
  title={Tighter problem-dependent regret bounds in reinforcement learning without domain knowledge using value function bounds},
  author={Zanette, Andrea and Brunskill, Emma},
  booktitle={International Conference on Machine Learning},
  pages={7304--7312},
  year={2019},
  organization={PMLR}
}

@article{zhang2021reward,
  title={Reward-free model-based reinforcement learning with linear function approximation},
  author={Zhang, Weitong and Zhou, Dongruo and Gu, Quanquan},
  journal={Advances in Neural Information Processing Systems},
  volume={34},
  pages={1582--1593},
  year={2021}
}

@article{zhang2024achieving,
  title={Achieving constant regret in linear markov decision processes},
  author={Zhang, Weitong and Fan, Zhiyuan and He, Jiafan and Gu, Quanquan},
  journal={Advances in Neural Information Processing Systems},
  volume={37},
  pages={130694--130738},
  year={2024}
}

@article{yin2021towards,
  title={Towards instance-optimal offline reinforcement learning with pessimism},
  author={Yin, Ming and Wang, Yu-Xiang},
  journal={Advances in neural information processing systems},
  volume={34},
  pages={4065--4078},
  year={2021}
}

@article{yin2022near,
  title={Near-optimal offline reinforcement learning with linear representation: Leveraging variance information with pessimism},
  author={Yin, Ming and Duan, Yaqi and Wang, Mengdi and Wang, Yu-Xiang},
  journal={arXiv preprint arXiv:2203.05804},
  year={2022}
}

@article{xie2021bellman,
  title={Bellman-consistent pessimism for offline reinforcement learning},
  author={Xie, Tengyang and Cheng, Ching-An and Jiang, Nan and Mineiro, Paul and Agarwal, Alekh},
  journal={Advances in neural information processing systems},
  volume={34},
  pages={6683--6694},
  year={2021}
}

@article{uehara2021pessimistic,
  title={Pessimistic model-based offline reinforcement learning under partial coverage},
  author={Uehara, Masatoshi and Sun, Wen},
  journal={arXiv preprint arXiv:2107.06226},
  year={2021}
}

@article{li2023reward,
  title={Reward-agnostic fine-tuning: Provable statistical benefits of hybrid reinforcement learning},
  author={Li, Gen and Zhan, Wenhao and Lee, Jason D and Chi, Yuejie and Chen, Yuxin},
  journal={Advances in Neural Information Processing Systems},
  volume={36},
  pages={55582--55615},
  year={2023}
}

@article{xie2021policy,
  title={Policy finetuning: Bridging sample-efficient offline and online reinforcement learning},
  author={Xie, Tengyang and Jiang, Nan and Wang, Huan and Xiong, Caiming and Bai, Yu},
  journal={Advances in neural information processing systems},
  volume={34},
  pages={27395--27407},
  year={2021}
}

@article{tang2025deep,
  title={Deep reinforcement learning for robotics: A survey of real-world successes},
  author={Tang, Chen and Abbatematteo, Ben and Hu, Jiaheng and Chandra, Rohan and Mart{\'\i}n-Mart{\'\i}n, Roberto and Stone, Peter},
  journal={Annual Review of Control, Robotics, and Autonomous Systems},
  volume={8},
  number={1},
  pages={153--188},
  year={2025},
  publisher={Annual Reviews}
}

@article{CORONATO2020101964,
title = {Reinforcement learning for intelligent healthcare applications: A survey},
journal = {Artificial Intelligence in Medicine},
volume = {109},
pages = {101964},
year = {2020},
issn = {0933-3657},
author = {Antonio Coronato and Muddasar Naeem and Giuseppe {De Pietro} and Giovanni Paragliola},
}

@article{ouyang2022training,
  title={Training language models to follow instructions with human feedback},
  author={Ouyang, Long and Wu, Jeffrey and Jiang, Xu and Almeida, Diogo and Wainwright, Carroll and Mishkin, Pamela and Zhang, Chong and Agarwal, Sandhini and Slama, Katarina and Ray, Alex and others},
  journal={Advances in neural information processing systems},
  volume={35},
  pages={27730--27744},
  year={2022}
}

@article{fujimoto2021minimalist,
  title={A minimalist approach to offline reinforcement learning},
  author={Fujimoto, Scott and Gu, Shixiang Shane},
  journal={Advances in neural information processing systems},
  volume={34},
  pages={20132--20145},
  year={2021}
}

@article{kostrikov2021offline,
  title={Offline reinforcement learning with implicit q-learning},
  author={Kostrikov, Ilya and Nair, Ashvin and Levine, Sergey},
  journal={arXiv preprint arXiv:2110.06169},
  year={2021}
}

@article{kumar2020conservative,
  title={Conservative q-learning for offline reinforcement learning},
  author={Kumar, Aviral and Zhou, Aurick and Tucker, George and Levine, Sergey},
  journal={Advances in neural information processing systems},
  volume={33},
  pages={1179--1191},
  year={2020}
}

@article{gulcehre2022empirical,
  title={An empirical study of implicit regularization in deep offline rl},
  author={Gulcehre, Caglar and Srinivasan, Srivatsan and Sygnowski, Jakub and Ostrovski, Georg and Farajtabar, Mehrdad and Hoffman, Matt and Pascanu, Razvan and Doucet, Arnaud},
  journal={arXiv preprint arXiv:2207.02099},
  year={2022}
}

@inproceedings{zhan2022offline,
  title={Offline reinforcement learning with realizability and single-policy concentrability},
  author={Zhan, Wenhao and Huang, Baihe and Huang, Audrey and Jiang, Nan and Lee, Jason},
  booktitle={Conference on Learning Theory},
  pages={2730--2775},
  year={2022},
  organization={PMLR}
}

@inproceedings{haarnoja2018soft,
  title={Soft actor-critic: Off-policy maximum entropy deep reinforcement learning with a stochastic actor},
  author={Haarnoja, Tuomas and Zhou, Aurick and Abbeel, Pieter and Levine, Sergey},
  booktitle={International conference on machine learning},
  pages={1861--1870},
  year={2018},
  organization={Pmlr}
}

@article{tarasov2023corl,
  title={CORL: Research-oriented deep offline reinforcement learning library},
  author={Tarasov, Denis and Nikulin, Alexander and Akimov, Dmitry and Kurenkov, Vladislav and Kolesnikov, Sergey},
  journal={Advances in Neural Information Processing Systems},
  volume={36},
  pages={30997--31020},
  year={2023}
}

@misc{fu2020d4rl,
    title={D4RL: Datasets for Deep Data-Driven Reinforcement Learning},
    author={Justin Fu and Aviral Kumar and Ofir Nachum and George Tucker and Sergey Levine},
    year={2020},
    eprint={2004.07219},
    archivePrefix={arXiv},
    primaryClass={cs.LG}
}
\bibliographystyle{plainnat}

\appendix
% \section{Technical Novelty}

% In the proof of the upper bound of Algorithm~\ref{alg:main-gfa}, the key technical novelty of this paper lies in upper bounding the cardinality of $\mathcal{M}_h$ (Lemma~\ref{lem:mh-upper}), the set of episodes $k\in[K]$ in which Algorithm~\ref{alg:main-gfa} executes Line 11 at horizon step $h\in[H]$. Our analysis is based on the quantity $\sum_{k\in\mathcal{M}_h}\hat Q_h^k(s_h^k,a_h^k)-\check Q_h^k(s_h^k,a_h^k)$. Specifically, we derive an upper bound on this quantity that depends on $|\mathcal{M}_h|$, as well as a lower bound that depends on both $|\mathcal{M}_h|$ and $\beta$. Since the lower bound cannot exceed the upper bound, this comparison yields an upper bound on $|\mathcal{M}_h|$ in terms of $\beta$. Prior work \citep{he2021logarithmic} has employed a related proof strategy in the gap dependent analysis of RL algorithms, while our result develops this idea in the $\beta$-separable value adaptation.

\section{Misspecified $\beta$-Separable Value Gaps}
\label{app:misspecified-gaps}
In many practical scenarios, it may still be unrealistic to obtain a reference $Q$-function $Q_{\rf}$ from offline RL pretraining that exactly matches $Q^\star$ on some region, as required by Definition~\ref{def:beta-sep}. Instead, the algorithm may only have access to a misspecified $\tilde Q_{\rf}$, which contains the error imperfect offline RL pretraining, as defined below

\begin{definition}[Misspecified $Q_\rf$]
\label{def:qref-mis}
\textsc{O2O-LSVI} does not have direct access to the true reference function $Q_{\rf}$, but instead observes a misspecified approximation $\tilde Q_{\rf}$. We assume the misspecification error is uniformly bounded as
\vspace{-0.1cm}
% \begin{align*}
$\eta_h(s,a):=\left| \tilde Q_{\rf,h}(s,a) - Q_{\rf,h}(s,a) \right| \le \tau , \forall (s,a,h)\in\mathcal{S}\times\mathcal{A}\times[H].$
% \end{align*}
\vspace{-0.2cm}
% where the error level satisfies $0<\tau H \le \tfrac{1}{2}{\beta}$.
\end{definition}
Based on the above definition, we can now characterize the regret upper bound of Algorithm~\ref{alg:main-gfa} under a misspecified reference $Q$-function, particularly without the knowledge of misspecification level $\tau$.

\begin{theorem}
\label{thm:upper-mis}
Under the assumptions and parameter choices of Theorem~\ref{thm:upper}, run Algorithm~\ref{alg:main-gfa} with the observed $\tilde Q_\rf$ in Definition~\ref{def:qref-mis}, including projection on its accepted-reference branch. For $0<\tau\le\beta/2$, with probability at least $1-5\delta$, the cumulative $\mathrm{Regret}(K)$ is bounded by
\begin{align*}
\widetilde{\mathcal O}\Big(\!\min\Big\{\!
\frac{H^4d_{\mathcal F}\log(\mathcal N_{\mathcal F}\mathcal N_{\mathcal B})}{\beta}
+H^2\sqrt{\rho Kd_{\mathcal F}\log(\mathcal N_{\mathcal F}\mathcal N_{\mathcal B})}, H^2\sqrt{Kd_{\mathcal F}\log(\mathcal N_{\mathcal F}\mathcal N_{\mathcal B})}\Big\}+HK\tau\Big).
\end{align*}
The algorithm does not require knowledge of $\tau$.
\end{theorem}
\begin{remark}
Theorem~\ref{thm:upper-mis} shows that when the offline-pretrained $Q$-function is misspecified with per-step error at most $\tau$, the regret upper bound incurs an additional $\tilde{\mathcal O}(HK\tau)$ term compared with the bound in the well-specified case given in Theorem~\ref{thm:upper}. This implies that the effect of misspecification accumulates over time and can not be easily detected through the trustworthy region in Algorithm~\ref{alg:main-gfa}.
\end{remark}
\begin{remark}
As discussed in Theorem~\ref{thm:lb}, it is fundamentally hard to improve the sample complexity using an arbitrarily misspecified pretrained function $\tilde Q_\rf$. For \textsc{O2O-LSVI}, Theorem~\ref{thm:upper-mis} further shows that the proposed method retains the $\rho$-dependent benefit in the regime of Remark~\ref{rm:beta-eps}, provided the requested accuracy is above the additional $\mathcal O(H\tau)$ misspecification floor. The bound includes an additional misspecification cost absent from the purely online guarantee; it does not by itself imply that the algorithm must perform worse. Together, these results identify separable value gaps as a sufficient structural condition for improved value transfer, not as a necessary condition for every possible transfer method. 
\end{remark}

\section{Proof of Theorem~\ref{thm:upper}}
\label{app:proof-upper}
\subsection{Sets}
For the detailed proof, let $\mathcal H_{k,h}$ denote the history immediately before the stage-$h$ transition in episode $k$, including $(s_h^k,a_h^k)$. Let $\mathcal H_{k,0}$ denote the history before episode $k$. All planning functions in episode $k$ are $\mathcal H_{k,0}$-measurable. Write
\[
 w_h^k(s,a):=\hat Q_h^k(s,a)-\check Q_h^k(s,a),\qquad
 c_h^k(s,a):=\min\{4b_h^k(s,a),H\}.
\]
Let $I_h^k(s,a)$ indicate that the inclusion test~\eqref{eqn:inclusion} fails. We use the following partition of visited pairs:
\begin{align}
 \mathcal M_h&:=\{k\in[K]:w_h^k(s_h^k,a_h^k)>\beta/4\},\label{eqn:mh-set}\\
 \mathcal V_h&:=\{k\in[K]\setminus\mathcal M_h:I_h^k(s_h^k,a_h^k)=1\}.\label{eqn:vh-set}
\end{align}
Thus $\mathcal M_h$ contains wide intervals, while $\mathcal V_h$ contains narrow intervals that fail the actual inclusion test. Some wide intervals may pass the test; none of the bounds below assumes that every visit to $\mathcal M_h$ uses the UCB branch. The threshold $\beta/4$ is only an analytical partition and does not change the algorithm's $\beta/2$ containment tolerance.

Throughout this appendix, abbreviate
\[
 d_h:=\dim_K(\mathcal F_h),\quad
 d_{\mathcal F}:=1+H^{-1}\sum_{h=1}^H d_h,\quad
 \ell:=\log(16HK/\delta),\quad \alpha:=\max_{k\le K}\alpha_k.
\]
In particular, $\alpha=\widetilde{\mathcal O}(H\sqrt{\log(e\mathcal N_{\mathcal F}\mathcal N_{\mathcal B})})$ and $\alpha^2\gtrsim H^2\ell$. Constants below may depend on the absolute oracle constant in Assumption~\ref{def:oracle}.
For sums along the realized trajectories, we occasionally suppress the
arguments $(s_h^k,a_h^k)$ in $c_h^k,w_h^k$, and the corresponding residuals.
All pointwise comparisons retain their state--action interpretation.
We will repeatedly use
\[
 \sum_{h=1}^H d_h\le Hd_{\mathcal F},\qquad
 \sum_{h=1}^H\sqrt{d_h}\le H\sqrt{d_{\mathcal F}},\qquad
 \kappa HK=1.
\]
The first follows from the displayed definition of $d_{\mathcal F}$, and
the second follows from Cauchy--Schwarz. The proof proceeds from regression
confidence to optimism, then to the two nonnegative-cost bounds. The main
theorem is proved by combining those bounds; proofs of the supporting
lemmas are supplied afterward in the existing order.

\subsection{Proof Sketch of Theorem~\ref{thm:upper}}
The first key lemma bounds the number of wide-interval visits, including terminal-stage uncertainty:
\begin{lemma}[Informal statement of Lemma~\ref{lem:mh-upper}]\label{lem:mh-upper-informal}
On the regression-confidence event, with an additional failure probability at most $\delta$, simultaneously for every $h\in[H]$,
\[
 |\mathcal M_h|\lesssim\min\{K,H^2\alpha^2d_{\mathcal F}/\beta^2\}.
\]
\end{lemma}
The second lemma uses the coverage condition only on narrow intervals that fail certification:
\begin{lemma}[Informal statement of Lemma~\ref{lem:up-t3}]\label{lem:up-t3-informal}
On the same confidence event, with an additional failure probability at most $\delta$,
\[
 \sum_{h=1}^H\sum_{k\in\mathcal V_h}c_h^k(s_h^k,a_h^k)
 \lesssim \alpha H\sqrt{d_{\mathcal F}(\rho K+\ell)}+\alpha.
\]
\end{lemma}
On a certified pair, $Q_h^k=Q_h^\star$, so its positive Bellman residual vanishes. Consequently, the regret is controlled by the conditional expectation of the nonnegative trajectory cost
$\sum_h I_h^kc_h^k$. Lemma~\ref{lem:nonnegative-predictable} converts its conditional expectations into its realized sum with an additive $\mathcal O(H^2\ell)$ term, rather than an unlocalized $\sqrt K$ term. The wide-interval contribution is $\mathcal O(H^2\alpha^2d_{\mathcal F}/\beta)$; the remaining contribution is controlled by Lemma~\ref{lem:up-t3-informal}. Independently, summing all bonuses gives the ordinary online bound. Both bounds hold on the same event and therefore can be minimized.

\subsection{High Probability Events}
\label{sec:detail-confidence}
For $\mathcal T_hV:=r_h+P_hV$, define
\begin{align*}
 \mathcal E^{\hat f}&:=\left\{\lambda+\sum_{i<k}
  (\hat f_h^k-\mathcal T_hV_{h+1}^k)^2(s_h^i,a_h^i)\le\alpha_k^2,
  \quad\forall h,k\right\},\\
 \mathcal E^{\check f}&:=\left\{\lambda+\sum_{i<k}
  (\check f_h^k-\mathcal T_h\check V_{h+1}^k)^2(s_h^i,a_h^i)\le\alpha_k^2,
  \quad\forall h,k\right\},\qquad
 \mathcal E:=\mathcal E^{\hat f}\cap\mathcal E^{\check f}.
\end{align*}
We establish these events for the continuation functions actually generated by
Algorithm~\ref{alg:main-gfa}, including its reference-switching operation.
Throughout, $\lambda=H^2$ and $\kappa=(KH)^{-1}$.

Fix an episode index $k$ and a stage $j\in[H]$. Conditional on the supplied
reference, each possible pair of regression estimates is chosen from
$(\mathcal F_j)_\kappa^2$, and the rounded bonus oracle is chosen from
$(\mathcal B)_\kappa$. Once these three functions and the deterministic radius
$\alpha_k$ are fixed, the clipping, ordering safeguard, inclusion test,
projection, and maximization over actions determine the resulting value
function uniquely. Thus, the number of possible upper continuation functions
is at most $\mathcal N_{\mathcal F}^2\mathcal N_{\mathcal B}$. The lower
continuation functions have no more choices, and the terminal continuation
is identically zero. Let $\mathcal W_{k,j}$ be a deterministic family
containing these functions, with $\mathcal W_{k,H+1}=\{0\}$. Then
\begin{equation}
 |\mathcal W_{k,j}|
 \le 2\mathcal N_{\mathcal F}^2\mathcal N_{\mathcal B}+1
 \le 3\mathcal N_{\mathcal F}^2\mathcal N_{\mathcal B}.
 \label{eq:detail-target-cardinality}
\end{equation}
All its members take values in the appropriate remaining-horizon range.
The family is fixed before observing the online data because the covers,
reference, and $\alpha_k$ are fixed. Although the algorithm selects a member
using the data, we can first obtain concentration simultaneously for every
member of the family. This argument does not require the inclusion test to
be continuous.

Fix $h,k$, a target $V\in\mathcal W_{k,h+1}$, and a candidate
$f\in(\mathcal F_h)_\kappa$. Write $g:=\mathcal T_hV$ and, for $i<k$, define
\begin{align*}
 Y_i&:=R_h^i+V(s_{h+1}^i),&
 \xi_i&:=Y_i-g(s_h^i,a_h^i),&
 \Delta_i&:=f(s_h^i,a_h^i)-g(s_h^i,a_h^i).
\end{align*}
Since $V$ and $f$ are fixed in this step, $\Delta_i$ is
$\mathcal H_{i,h}$-measurable, and the MDP sampling rule gives
\begin{align}
 \mathbb E[\xi_i\mid\mathcal H_{i,h}]
 &=\mathbb E[R_h^i+V(s_{h+1}^i)\mid\mathcal H_{i,h}]
       -g(s_h^i,a_h^i)\notag\\
 &=r_h(s_h^i,a_h^i)+[P_hV](s_h^i,a_h^i)
       -[\mathcal T_hV](s_h^i,a_h^i)=0.
 \label{eq:detail-noise-centered}
\end{align}
Also, $Y_i,g(s_h^i,a_h^i)\in[0,H-h+1]$, so $|\xi_i|\le H$.
The conditional exponential-moment bound proved in the proof of
Lemma~\ref{lem:hoeffding}, applied to the interval $[-H,H]$, yields
\begin{align}
 \mathbb E\!\left[
  \exp\!\left(\frac{\xi_i\Delta_i}{2H^2}
                 -\frac{\Delta_i^2}{8H^2}\right)
                 \,\middle|\,\mathcal H_{i,h}\right]
 &\le
 \exp\!\left(\frac{H^2}{2}
                  \frac{\Delta_i^2}{4H^4}
                 -\frac{\Delta_i^2}{8H^2}\right)=1.
 \label{eq:detail-noise-mgf}
\end{align}
Iterating conditional expectations over the stage-$h$ observations and
applying Markov's inequality, we obtain, with probability at least $1-\eta$,
\begin{align}
 \sum_{i<k}\left(\frac{\xi_i\Delta_i}{2H^2}
                     -\frac{\Delta_i^2}{8H^2}\right)
 &\le\log(1/\eta),\notag\\
 2\sum_{i<k}\xi_i\Delta_i
 &\le\frac12\sum_{i<k}\Delta_i^2+4H^2\log(1/\eta).
 \label{eq:detail-regression-noise}
\end{align}

To make the bound uniform over all targets and candidates, set
\[
 \eta_{k,h}:=
 \frac{\delta}{8Hk^2\mathcal N_{\mathcal F}^3\mathcal N_{\mathcal B}}.
\]
At each $h,k$, the number of pairs $(f,V)$ is at most
$3\mathcal N_{\mathcal F}^3\mathcal N_{\mathcal B}$ by
\eqref{eq:detail-target-cardinality}. Hence the total failure probability is
bounded by
\begin{align}
 \sum_{h=1}^H\sum_{k=1}^K
 3\mathcal N_{\mathcal F}^3\mathcal N_{\mathcal B}\eta_{k,h}
 &\le \frac{3\delta}{8}\sum_{k=1}^{\infty}\frac1{k^2}
 =\frac{\pi^2}{16}\delta<\delta.
 \label{eq:detail-confidence-union}
\end{align}
Consequently, \eqref{eq:detail-regression-noise} holds simultaneously for the
actual upper and lower targets and their selected regression minimizers.
In particular, no concentration inequality is applied after conditioning on
the selected target or on a future confidence event.

Let $\hat f$ be the empirical minimizer for a target $V$ in the preceding
finite family. By Assumption~\ref{asm:bellman-completeness},
$g=\mathcal T_hV\in\mathcal F_h$. Definition~\ref{def:covering} therefore
provides $g_\kappa\in(\mathcal F_h)_\kappa$ satisfying
$\|g_\kappa-g\|_\infty\le\kappa$. Set
$\varepsilon_i:=g_\kappa(s_h^i,a_h^i)-g(s_h^i,a_h^i)$ and now let
$\Delta_i:=\hat f(s_h^i,a_h^i)-g(s_h^i,a_h^i)$. By the definition of $\hat f$,
\begin{align}
 \sum_{i<k}(\Delta_i-\xi_i)^2
 &=\sum_{i<k}(\hat f(s_h^i,a_h^i)-Y_i)^2\notag\\
 &\le\sum_{i<k}(g_\kappa(s_h^i,a_h^i)-Y_i)^2
 =\sum_{i<k}(\varepsilon_i-\xi_i)^2.
 \label{eq:detail-ls-comparison}
\end{align}
Expanding the squares, canceling $\sum_{i<k}\xi_i^2$, and using
$|\varepsilon_i|\le\kappa$ gives
\begin{align}
 \sum_{i<k}\Delta_i^2
 &\le2\sum_{i<k}\xi_i\Delta_i
       +\sum_{i<k}\varepsilon_i^2
       -2\sum_{i<k}\xi_i\varepsilon_i\notag\\
 &\le2\sum_{i<k}\xi_i\Delta_i
       +(k-1)\kappa^2+2H(k-1)\kappa\notag\\
 &\le2\sum_{i<k}\xi_i\Delta_i+3\notag\\
 &\le\frac12\sum_{i<k}\Delta_i^2
       +4H^2\log(1/\eta_{k,h})+3.
 \label{eq:detail-ls-error-recursion}
\end{align}
The second inequality uses $|\xi_i|\le H$; the third uses
$\kappa=(KH)^{-1}$ and $k-1\le K$; and the last follows from the uniform
version of \eqref{eq:detail-regression-noise}. Moving the squared-error term
on the right to the left yields
\begin{align}
 \lambda+\sum_{i<k}\Delta_i^2
 &\le H^2+8H^2\log\!\left(
       \frac{8Hk^2\mathcal N_{\mathcal F}^3\mathcal N_{\mathcal B}}{\delta}
                         \right)+6
 \le\alpha_k^2,
 \label{eq:detail-regression-confidence}
\end{align}
where the last inequality holds for the sufficiently large absolute
constant $c_\alpha$ specified in Theorem~\ref{thm:upper}.
The argument applies to both targets $V_{h+1}^k$ and
$\check V_{h+1}^k$. Thus it proves, in particular, the stated bound
$\Pr(\mathcal E)\ge1-2\delta$.

We also give the details behind the linear uncertainty and covering bounds.
Let the coefficient set of the bounded linear class have radius
$W=\mathcal O(H\sqrt d)$, and suppose $\|\phi(s,a)\|_2\le1$.
The case $W=0$ is immediate, so take $W>0$. Let $\mathcal U$ be the set of
differences of two coefficient vectors; then
$\|v\|_2\le2W$ and $|\langle v,\phi(s,a)\rangle|\le H$ for
$v\in\mathcal U$. With
$\Sigma=\sum_{i<k}\phi(s_h^i,a_h^i)\phi(s_h^i,a_h^i)^\top$, the exact
uncertainty oracle satisfies
\begin{align}
 D^2(s,a;\Sigma)
 &=\sup_{v\in\mathcal U}
   \frac{\langle v,\phi(s,a)\rangle^2}{\lambda+v^\top\Sigma v}
 \le1,\label{eq:detail-linear-unit}\\
 \lambda+v^\top\Sigma v
 &\ge\frac{\lambda}{4W^2}\|v\|_2^2+v^\top\Sigma v
 =v^\top\!\left(\Sigma+\frac{\lambda}{4W^2}I\right)v.
 \label{eq:detail-linear-denominator}
\end{align}
Here \eqref{eq:detail-linear-unit} uses $\lambda=H^2$ and the output bound.
Writing $A=\Sigma+\lambda I/(4W^2)$, Cauchy--Schwarz gives
\[
 \langle v,\phi(s,a)\rangle^2
 =\langle A^{1/2}v,A^{-1/2}\phi(s,a)\rangle^2
 \le(v^\top Av)(\phi(s,a)^\top A^{-1}\phi(s,a)).
\]
Combining these inequalities, including the zero-vector case separately,
shows that
\begin{equation}
 D^2(s,a;\Sigma)
 \le\min\!\left\{1,
 \phi(s,a)^\top\left(\Sigma+\frac{\lambda}{4W^2}I\right)^{-1}\phi(s,a)
                  \right\}.
 \label{eq:detail-linear-elliptical}
\end{equation}
For a fixed stage, abbreviate $\phi_i=\phi(s_h^i,a_h^i)$, let
$a_0=\lambda/(4W^2)$, and set
$A_i=a_0I+\sum_{t=1}^i\phi_t\phi_t^\top$.
For any sequence of observations, the matrix determinant identity and
$\min\{1,x\}\le2\log(1+x)$ for $x\ge0$ imply
\begin{align}
 \sum_{i=1}^K D^2(s_h^i,a_h^i;\Sigma_{i-1})
 &\le\sum_{i=1}^K\min\{1,\phi_i^\top A_{i-1}^{-1}\phi_i\}\notag\\
 &\le2\sum_{i=1}^K\log(1+\phi_i^\top A_{i-1}^{-1}\phi_i)\notag\\
 &=2\log\frac{\det A_K}{\det A_0}\notag\\
 &\le2d\log\!\left(1+\frac{K}{d a_0}\right).
 \label{eq:detail-linear-potential}
\end{align}
The last inequality follows from
$\operatorname{tr}(A_K)\le d a_0+K$ and the arithmetic--geometric mean
inequality for the eigenvalues of $A_K$. Taking the supremum over sequences
in Definition~\ref{def:generalized-eluder} gives
$\dim_K(\mathcal F_h)=\widetilde{\mathcal O}(d)$.

To control the bonus covering number, let $\Sigma,\Sigma'$ be two empirical
Gram matrices and fix $v\in\mathcal U$. Direct subtraction gives
\begin{align}
 &\left|
 \frac{\langle v,\phi(s,a)\rangle^2}{\lambda+v^\top\Sigma v}
 -\frac{\langle v,\phi(s,a)\rangle^2}{\lambda+v^\top\Sigma'v}
 \right|\notag\\
 &\qquad=
 \frac{\langle v,\phi(s,a)\rangle^2
       |v^\top(\Sigma'-\Sigma)v|}
      {(\lambda+v^\top\Sigma v)(\lambda+v^\top\Sigma'v)}\notag\\
 &\qquad\le
 \frac{H^2\cdot4W^2\|\Sigma'-\Sigma\|_{\mathrm{op}}}{H^4}
 =\frac{4W^2}{H^2}\|\Sigma'-\Sigma\|_{\mathrm{op}}.
 \label{eq:detail-linear-cover-square}
\end{align}
Using $|\sup_v a_v-\sup_v b_v|\le\sup_v|a_v-b_v|$ and
$|\sqrt x-\sqrt y|\le\sqrt{|x-y|}$ for $x,y\ge0$, we obtain
\begin{equation}
 \|D(\cdot;\Sigma)-D(\cdot;\Sigma')\|_\infty
 \le\frac{2W}{H}\|\Sigma-\Sigma'\|_{\mathrm{op}}^{1/2}.
 \label{eq:detail-linear-cover}
\end{equation}
The empirical Gram matrices lie in an operator-norm ball of radius $K$ in
the $d(d+1)/2$-dimensional space of symmetric matrices. A cover at radius
$r=\kappa^2H^2/(4W^2)$ therefore induces a $\kappa$-cover of the exact oracle.
The usual finite-dimensional packing bound, obtained by comparing the
volumes of disjoint radius-$r/2$ balls with a radius-$(K+r/2)$ ball, gives
\begin{align*}
 \log\mathcal N_\kappa(\mathcal B)
 &\le\frac{d(d+1)}2
     \log\!\left(1+\frac{cKW^2}{\kappa^2H^2}\right)
 =\widetilde{\mathcal O}(d^2),\\
 \log\mathcal N_\kappa(\mathcal F)
 &\le d\log\!\left(1+\frac{cW}{\kappa}\right)
 =\widetilde{\mathcal O}(d),
\end{align*}
for an absolute constant $c$. For the second line, we use
$\|\langle w-w',\phi(\cdot)\rangle\|_\infty\le\|w-w'\|_2$ and the
same packing argument for the coefficient set. Representatives can be
chosen within the respective sets by taking maximal separated subsets.
These are the uncertainty and logarithmic covering bounds used in the
linear specialization.

\subsection{Optimism and Pessimism Lemmas}
\begin{lemma}\label{lem:bern-error}
On $\mathcal E$, for all $h,k,s,a$,
\[
 |\hat f_h^k(s,a)-\mathcal T_hV_{h+1}^k(s,a)|
 \le\alpha_kD_{\mathcal F_h}((s,a);z_{[k-1],h})\le b_h^k(s,a),
\]
and the same statement holds with $\hat f,V$ replaced by $\check f,\check V$.
\end{lemma}
\begin{lemma}\label{lem:opt-pess}
On $\mathcal E$, in the well-specified setting, for all $h,k,s,a$,
\[
 0\le\check Q_h^k\le Q_h^\star\le Q_h^k\le\hat Q_h^k\le H-h+1,
 \qquad \check V_h^k\le V_h^\star\le V_h^k.
\]
\end{lemma}
\begin{lemma}[Optimism with $Q_\rf$]\label{lem:inclusion-opt}
On $\mathcal E$, if the inclusion test holds, then
$Q_{\rf;h}(s,a)=Q_h^\star(s,a)$ and its projection onto
$[\check Q_h^k(s,a),\hat Q_h^k(s,a)]$ leaves it unchanged.
\end{lemma}
\begin{lemma}[Optimism]\label{lem:opt}
On $\mathcal E$, $Q_h^k(s,a)\ge Q_h^\star(s,a)$ and
$V_h^k(s)\ge V_h^\star(s)$ for every $h,k,s,a$.
\end{lemma}

\subsection{Supporting Lemmas}
\begin{lemma}\label{lem:sum-bonus}
For each stage $h$, the bonuses satisfy the deterministic bound
\[
 \sum_{k=1}^K c_h^k(s_h^k,a_h^k)
 \le4\sqrt C\,\alpha\sqrt{Kd_h}+8\alpha\kappa K.
\]
\end{lemma}
\begin{lemma}\label{lem:sum-bonus-subset}
For every subset $M\subseteq[K]$, including data-dependent subsets,
\[
 \sum_{k\in M}c_h^k(s_h^k,a_h^k)
 \le4\sqrt C\,\alpha\sqrt{|M|d_h}+8\alpha\kappa|M|.
\]
\end{lemma}
\begin{lemma}[Conditional regret decomposition]\label{lem:telescope}
On $\mathcal E$, in the well-specified setting,
\[
 \mathrm{Regret}(K)\le\sum_{k=1}^K
 \mathbb E\!\left[\sum_{h=1}^H I_h^k(s_h^k,a_h^k)c_h^k(s_h^k,a_h^k)
 \,\middle|\,\mathcal H_{k,0}\right].
\]
Each trajectory cost inside this expectation is nonnegative and at most $H^2$.
\end{lemma}
\begin{lemma}[Telescoping the confidence gap]\label{lem:telescope-gap}
On $\mathcal E$, interval projection gives $V_h^k\ge\check V_h^k$ and, along the greedy policy,
\[
 w_h^k(s,a)\le c_h^k(s,a)+P_h(V_{h+1}^k-\check V_{h+1}^k)(s,a),
 \qquad
 V_h^k(s)-\check V_h^k(s)\le w_h^k(s,\pi_h^k(s)).
\]
Consequently, for a fixed stage $h$ and a selection rule with
$\ind\{k\in M\}$ measurable with respect to $\mathcal H_{k,h}$,
\[
 \sum_{k\in M}w_h^k(s_h^k,a_h^k)
 \le2\sum_{k\in M}\sum_{j=h}^H c_j^k(s_j^k,a_j^k)
       +2H^2\log(1/\eta)
\]
on an additional event of probability at least $1-\eta$.
This probability assertion is for the specified predictable selection rule, not uniformly over arbitrary retrospectively chosen subsets.
\end{lemma}
\begin{lemma}[Upper bound on $|\mathcal M_h|$]\label{lem:mh-upper}
On $\mathcal E$, with an additional failure probability at most $\delta$, simultaneously for all $h$,
\begin{equation}
 |\mathcal M_h|\le\min\{K,C_1H^2\alpha^2d_{\mathcal F}/\beta^2\}
 \label{eqn:mh-upper}
\end{equation}
for an absolute constant $C_1$.
\end{lemma}
\begin{lemma}\label{lem:up-t1}
On the events of Lemma~\ref{lem:mh-upper},
\[
 \sum_{h=1}^H\sum_{k\in\mathcal M_h}c_h^k(s_h^k,a_h^k)
 \lesssim\min\{H^2\alpha^2d_{\mathcal F}/\beta,
                     \alpha H\sqrt{Kd_{\mathcal F}}\}+\alpha.
\]
\end{lemma}
\begin{lemma}\label{lem:bounded-vh}
Under Assumption~\ref{asm:bounded-coverage}, on $\mathcal E$ and an additional event of probability at least $1-\delta$,
\[
 |\mathcal V_h|\le2\rho K+\log(H/\delta),\qquad h\in[H].
\]
When $\rho=0$, no inaccurate pair is visited almost surely, so $|\mathcal V_h|=0$ on $\mathcal E$.
\end{lemma}
\begin{lemma}\label{lem:up-t3}
On the events of Lemma~\ref{lem:bounded-vh},
\[
 \sum_{h=1}^H\sum_{k\in\mathcal V_h}c_h^k(s_h^k,a_h^k)
 \lesssim \alpha H\sqrt{d_{\mathcal F}(\rho K+\ell)}+\alpha.
\]
The contribution of every certified pair to the positive Bellman residual is zero in the well-specified setting.
\end{lemma}

\subsection{Detailed Proof}
\begin{proof}[Proof of Theorem~\ref{thm:upper}]
We first obtain the adaptation bound and the ordinary online bound on the
same high-probability event. We use the notation $d_h,d_{\mathcal F},\ell$
and $\alpha$ introduced at the beginning of this appendix.

Define the nonnegative episode cost
\[
 X_k:=\sum_{h=1}^H
 I_h^k(s_h^k,a_h^k)c_h^k(s_h^k,a_h^k),\qquad
 m_k:=\mathbb E[X_k\mid\mathcal H_{k,0}].
\]
The inclusion indicators lie in $\{0,1\}$ and $0\le c_h^k\le H$, so
$0\le X_k\le H^2$, including on histories where regression confidence fails.
Apply the first inequality of Lemma~\ref{lem:nonnegative-predictable} with
$B=H^2$ and failure probability $\delta$. It gives an event $\mathcal G_X$
on which
\begin{equation}
 \sum_{k=1}^K m_k
 \le2\sum_{k=1}^K X_k+2H^2\log(1/\delta).
 \label{eq:detail-upper-cost-concentration}
\end{equation}
Let $\mathcal G_M$ and $\mathcal G_V$ be the additional concentration events
used in Lemmas~\ref{lem:mh-upper} and~\ref{lem:bounded-vh}, respectively.
Section~\ref{sec:detail-confidence} and those lemmas imply
\begin{align*}
 \Pr\big((\mathcal E\cap\mathcal G_M\cap\mathcal G_V
                         \cap\mathcal G_X)^c\big)
 &\le\Pr(\mathcal E^c)+\Pr(\mathcal G_M^c)
       +\Pr(\mathcal G_V^c)+\Pr(\mathcal G_X^c)\\
 &\le2\delta+\delta+\delta+\delta=5\delta.
\end{align*}
All subsequent inequalities hold on this intersection. No independence of
these events is required for the union bound.

Lemma~\ref{lem:telescope} and
\eqref{eq:detail-upper-cost-concentration} give
\begin{align}
 \mathrm{Regret}(K)
 &\overset{(a)}{\le}\sum_{k=1}^K m_k\notag\\
 &\overset{(b)}{\le}
  2\sum_{k=1}^K\sum_{h=1}^H
       I_h^k(s_h^k,a_h^k)c_h^k(s_h^k,a_h^k)+2H^2\ell\notag\\
 &\overset{(c)}{\le}
  2\sum_{h=1}^H\sum_{k\in\mathcal M_h}c_h^k(s_h^k,a_h^k)
  +2\sum_{h=1}^H\sum_{k\in\mathcal V_h}c_h^k(s_h^k,a_h^k)
  +2H^2\ell.
 \label{eq:detail-upper-partition}
\end{align}
Here (a) uses Lemma~\ref{lem:telescope}; (b) uses
\eqref{eq:detail-upper-cost-concentration} and
$\log(1/\delta)\le\ell$; and (c) uses the partition
\eqref{eqn:mh-set}--\eqref{eqn:vh-set}. Indeed, a failed inclusion test is
either at a wide interval or at a narrow interval in $\mathcal V_h$.
A wide interval may also pass the test, but replacing its indicator by one
only increases the nonnegative upper bound.

Apply Lemma~\ref{lem:up-t1} to the first sum and
Lemma~\ref{lem:up-t3} to the second sum in
\eqref{eq:detail-upper-partition}. For an absolute constant $C_2$, this gives
\begin{align}
 \mathrm{Regret}(K)
 &\le C_2\left(
       \frac{H^2\alpha^2d_{\mathcal F}}{\beta}
       +\alpha H\sqrt{d_{\mathcal F}(\rho K+\ell)}
       +\alpha+H^2\ell\right)\notag\\
 &\le C_2\left(
       \frac{H^2\alpha^2d_{\mathcal F}}{\beta}
       +\alpha H\sqrt{\rho Kd_{\mathcal F}}
       +\alpha H\sqrt{d_{\mathcal F}\ell}
       +\alpha+H^2\ell\right).
 \label{eq:detail-upper-with-remainders}
\end{align}
The second inequality uses $\sqrt{x+y}\le\sqrt x+\sqrt y$ for
$x,y\ge0$. To account explicitly for the remaining terms, choose the
absolute confidence constant large enough that
$\alpha\ge H\sqrt\ell\ge1$. Since $d_{\mathcal F}\ge1$ and $\beta\le H$,
\begin{align}
 \alpha H\sqrt{d_{\mathcal F}\ell}
 &\le\alpha^2\sqrt{d_{\mathcal F}}
 \le\alpha^2d_{\mathcal F}
 \le\frac{H^2\alpha^2d_{\mathcal F}}{\beta},\notag\\
 \alpha&\le\alpha^2
 \le\frac{H^2\alpha^2d_{\mathcal F}}{\beta},\notag\\
 H^2\ell&\le\alpha^2
 \le\frac{H^2\alpha^2d_{\mathcal F}}{\beta}.
 \label{eq:detail-remainder-absorption}
\end{align}
Substituting these three inequalities into
\eqref{eq:detail-upper-with-remainders}, we obtain
\begin{equation}
 \mathrm{Regret}(K)
 \lesssim\frac{H^2\alpha^2d_{\mathcal F}}{\beta}
           +\alpha H\sqrt{\rho Kd_{\mathcal F}}.
 \label{eq:detail-adaptation-bound}
\end{equation}
Thus the logarithmic concentration terms have been bounded explicitly,
not discarded.

Starting from (b) in \eqref{eq:detail-upper-partition}, use $I_h^k\le1$
and Lemma~\ref{lem:sum-bonus} to obtain
\begin{align}
 \mathrm{Regret}(K)
 &\le2\sum_{h=1}^H\sum_{k=1}^Kc_h^k(s_h^k,a_h^k)+2H^2\ell\notag\\
 &\le8\sqrt C\,\alpha\sqrt K\sum_{h=1}^H\sqrt{d_h}
       +16\alpha\kappa KH+2H^2\ell\notag\\
 &\le8\sqrt C\,\alpha H\sqrt{Kd_{\mathcal F}}
       +16\alpha+2H^2\ell.
 \label{eq:detail-online-with-remainders}
\end{align}
The last inequality uses $\kappa KH=1$ and Cauchy--Schwarz:
\[
 \sum_{h=1}^H\sqrt{d_h}
 \le\sqrt{H\sum_{h=1}^H d_h}
 \le H\sqrt{d_{\mathcal F}}.
\]
If $K\ge\ell$, then
\[
 \alpha\le\alpha H\sqrt{Kd_{\mathcal F}},\qquad
 H^2\ell\le\alpha H\sqrt\ell
              \le\alpha H\sqrt{Kd_{\mathcal F}},
\]
so \eqref{eq:detail-online-with-remainders} is of order
$\alpha H\sqrt{Kd_{\mathcal F}}$.
If $K<\ell$, we instead use rewards in $[0,1]$ to write
\begin{equation}
 \mathrm{Regret}(K)\le HK
 \le H\sqrt{K\ell}
 \le\alpha\sqrt K
 \le\alpha H\sqrt{Kd_{\mathcal F}}.
 \label{eq:detail-small-budget}
\end{equation}
Therefore, for every budget $K$,
\begin{equation}
 \mathrm{Regret}(K)\lesssim\alpha H\sqrt{Kd_{\mathcal F}}.
 \label{eq:detail-online-bound}
\end{equation}
This argument reuses the same event $\mathcal G_X$; it does not require
an additional concentration event or the special case
$|\mathcal M_h|=K$.

Because \eqref{eq:detail-adaptation-bound} and
\eqref{eq:detail-online-bound} hold simultaneously,
\begin{equation}
 \mathrm{Regret}(K)
 \lesssim\min\left\{
 \frac{H^2\alpha^2d_{\mathcal F}}{\beta}
   +\alpha H\sqrt{\rho Kd_{\mathcal F}},\quad
 \alpha H\sqrt{Kd_{\mathcal F}}\right\}.
 \label{eq:detail-upper-alpha}
\end{equation}
Let
$\mathcal L_K:=\log(8HK^2\mathcal N_{\mathcal F}^3
\mathcal N_{\mathcal B}/\delta)$, so
$\alpha=c_\alpha H\sqrt{\mathcal L_K}$. The individual terms become
\begin{align*}
 \frac{H^2\alpha^2d_{\mathcal F}}{\beta}
 &=\frac{c_\alpha^2H^4d_{\mathcal F}\mathcal L_K}{\beta},\\
 \alpha H\sqrt{\rho Kd_{\mathcal F}}
 &=c_\alpha H^2\sqrt{\rho Kd_{\mathcal F}\mathcal L_K},\\
 \alpha H\sqrt{Kd_{\mathcal F}}
 &=c_\alpha H^2\sqrt{Kd_{\mathcal F}\mathcal L_K}.
\end{align*}
Finally,
$\mathcal L_K=\mathcal O(\log(8HK^2/\delta)
+\log(e\mathcal N_{\mathcal F}\mathcal N_{\mathcal B}))$.
Substituting into \eqref{eq:detail-upper-alpha} gives the stated leading
regret orders with the positive-complexity convention fixed in this
appendix. The probability is at least $1-5\delta$ by the union bound established at the beginning of the proof.
\end{proof}

\begin{proof}[Proof of Corollary~\ref{cor:sampling-complexity}]
We first identify the value of the returned mixture and then solve the two
regret bounds for the required number of episodes.

Conditional on the completed training history, let $J$ be independent and
uniform on $[K]$. By definition, $\bar\pi$ executes $\pi^J$ for an entire
new episode. The law of total expectation gives
\begin{align}
 V_1^{\bar\pi}(s_1)
 &=\mathbb E_J[V_1^{\pi^J}(s_1)]
 =\frac1K\sum_{k=1}^K V_1^{\pi^k}(s_1),\notag\\
 V_1^\star(s_1)-V_1^{\bar\pi}(s_1)
 &=\frac1K\sum_{k=1}^K
       \bigl(V_1^\star(s_1)-V_1^{\pi^k}(s_1)\bigr)
 =\frac{\mathrm{Regret}(K)}K.
 \label{eq:detail-mixture-identity}
\end{align}
The mixing is across entire policies at the start of an episode, rather
than an independent choice of a new policy at every stage.

To make the logarithmic dependence explicit, write
$C_K:=d_{\mathcal F}\mathcal L_K$, with $\mathcal L_K$ as in the proof of
Theorem~\ref{thm:upper}. That proof and
\eqref{eq:detail-mixture-identity} imply, for an absolute constant $C_3$,
\begin{equation}
 V_1^\star(s_1)-V_1^{\bar\pi}(s_1)
 \le C_3\left(\frac{H^4C_K}{\beta K}
                   +H^2\sqrt{\frac{\rho C_K}{K}}\right).
 \label{eq:detail-mixture-adaptation}
\end{equation}
It suffices to make each summand at most $\epsilon/2$. The two requirements
are equivalent to
\begin{align*}
 C_3\frac{H^4C_K}{\beta K}\le\frac\epsilon2
 &\quad\Longleftrightarrow\quad
 K\ge\frac{2C_3H^4C_K}{\beta\epsilon},\\
 C_3H^2\sqrt{\frac{\rho C_K}{K}}\le\frac\epsilon2
 &\quad\Longleftrightarrow\quad
 K\ge\frac{4C_3^2\rho H^4C_K}{\epsilon^2}.
\end{align*}
In particular, both hold if $K$ is at least the sum of the right-hand sides.
This explains why the adaptation-overhead term contains $1/\epsilon$.

The online branch of Theorem~\ref{thm:upper} similarly gives
\begin{align*}
 V_1^\star(s_1)-V_1^{\bar\pi}(s_1)
 \le C_3H^2\sqrt{\frac{C_K}{K}}\le\epsilon
 \quad\text{whenever}\quad
 K\ge\frac{C_3^2H^4C_K}{\epsilon^2}.
\end{align*}
The algorithm's regret satisfies both bounds on the same event. Therefore
it is enough for the budget to satisfy either of these sufficient
requirements, giving
\[
 K=\widetilde{\mathcal O}\!\left(\min\left\{
 \frac{H^4d_{\mathcal F}\log(\mathcal N_{\mathcal F}\mathcal N_{\mathcal B})}
      {\beta\epsilon}
 +\frac{\rho H^4d_{\mathcal F}\log(\mathcal N_{\mathcal F}\mathcal N_{\mathcal B})}
      {\epsilon^2},\quad
 \frac{H^4d_{\mathcal F}\log(\mathcal N_{\mathcal F}\mathcal N_{\mathcal B})}
      {\epsilon^2}\right\}\right).
\]
All complexity quantities are evaluated at the selected budget $K$,
since the covers have resolution $(KH)^{-1}$ and the eluder quantity also
depends on $K$. The success probability is inherited from
Theorem~\ref{thm:upper}; no additional evaluation samples or concentration
inequality are needed for the value identity. This is a guarantee for the
mixture and does not assert a last-iterate guarantee.
\end{proof}

\subsection{Detailed Proof of Supporting Lemmas}
\begin{proof}[Proof of Lemma~\ref{lem:bern-error}]
Fix $h,k$ and $(s,a)$, and suppose that $\mathcal E$ holds.
By Assumption~\ref{asm:bellman-completeness}, both Bellman targets
$g_h^k:=\mathcal T_hV_{h+1}^k$ and
$\check g_h^k:=\mathcal T_h\check V_{h+1}^k$ belong to $\mathcal F_h$.
The regression minimizers also belong to $\mathcal F_h$, because their
finite cover is a subset of this class. Thus the pair
$(\hat f_h^k,g_h^k)$ is one of the pairs in the supremum defining
$D_{\mathcal F_h}^2$. It follows that
\begin{align}
 &|\hat f_h^k(s,a)-g_h^k(s,a)|^2\notag\\
 &\quad\le D_{\mathcal F_h}^2((s,a);z_{[k-1],h})
 \left(\lambda+\sum_{i<k}
     |\hat f_h^k(s_h^i,a_h^i)-g_h^k(s_h^i,a_h^i)|^2\right)\notag\\
 &\quad\le\alpha_k^2D_{\mathcal F_h}^2((s,a);z_{[k-1],h}).
 \label{eq:detail-pointwise-upper-error}
\end{align}
The first inequality uses Definition~\ref{def:generalized-eluder}; the
second uses $\mathcal E^{\hat f}$. Taking square roots and then using
Assumption~\ref{def:oracle} and the rounding rule in
Algorithm~\ref{alg:main-gfa}, we obtain
\begin{align}
 |\hat f_h^k(s,a)-\mathcal T_hV_{h+1}^k(s,a)|
 &\le\alpha_kD_{\mathcal F_h}((s,a);z_{[k-1],h})\notag\\
 &\le\alpha_k\bar D_{\mathcal F_h}((s,a);z_{[k-1],h})\notag\\
 &\le\alpha_k(\bar D_{h,\kappa}^k(s,a)+\kappa)
 =b_h^k(s,a).
 \label{eq:detail-pointwise-upper-bonus}
\end{align}
For the lower regression, use the pair $(\check f_h^k,\check g_h^k)$ and
$\mathcal E^{\check f}$ in exactly the same definition:
\begin{align*}
 &|\check f_h^k(s,a)-\check g_h^k(s,a)|^2\\
 &\quad\le D_{\mathcal F_h}^2((s,a);z_{[k-1],h})
 \left(\lambda+\sum_{i<k}
     |\check f_h^k(s_h^i,a_h^i)-\check g_h^k(s_h^i,a_h^i)|^2\right)\\
 &\quad\le\alpha_k^2D_{\mathcal F_h}^2((s,a);z_{[k-1],h}).
\end{align*}
Taking square roots and repeating the oracle and rounding inequalities
in \eqref{eq:detail-pointwise-upper-bonus} proves the lower statement.

We record a useful consequence of these bounds. Projection in
Algorithm~\ref{alg:main-gfa} ensures
$\check Q_h^k\le Q_h^k\le\hat Q_h^k$ on every history, and hence
$\check V_h^k\le V_h^k$. Both Bellman targets belong to $[0,H_h]$, where
$H_h=H-h+1$. The just-proved error bounds imply
\begin{align*}
 \check f_h^k-b_h^k&\le\check g_h^k\le g_h^k
                                     \le\hat f_h^k+b_h^k,\\
 \check f_h^k-b_h^k&\ge\check g_h^k-2b_h^k,\qquad
 \hat f_h^k+b_h^k\le g_h^k+2b_h^k.
\end{align*}
Clipping to $[0,H_h]$ preserves the first line's comparisons with the
Bellman targets. In particular, the clipped raw upper estimate is at least
the lower estimate, so the extra ordering maximum is inactive on
$\mathcal E$. Consequently, for any fixed supplied reference,
\begin{equation}
 \check g_h^k-2b_h^k
 \le\check Q_h^k\le\check g_h^k
 \le g_h^k\le\hat Q_h^k\le g_h^k+2b_h^k.
 \label{eq:detail-confidence-sandwich}
\end{equation}
For the two outside bounds, the lower clipping can only increase
$\check f_h^k-b_h^k\le H_h$, and the upper clipping can only decrease
$\hat f_h^k+b_h^k\ge0$. This consequence uses the ordering of the planning
values, not optimism relative to $Q^\star$, and will also be used in the
misspecified analysis.
\end{proof}

\begin{proof}[Joint proof of Lemmas~\ref{lem:opt-pess}--\ref{lem:opt}]
We prove confidence, certification, and optimism together by backward
induction. This avoids using a stage-$h$ certification statement before
establishing the corresponding confidence bounds. Fix an episode $k$ and
work on $\mathcal E$.

At stage $H+1$, Algorithm~\ref{alg:main-gfa} sets
$V_{H+1}^k=\check V_{H+1}^k=0$, and the Bellman terminal condition gives
$V_{H+1}^\star=0$. Thus
$\check V_{H+1}^k\le V_{H+1}^\star\le V_{H+1}^k$.

Suppose inductively that, at stage $h+1$, for every state $s'$,
\begin{equation}
 \check V_{h+1}^k(s')\le V_{h+1}^\star(s')\le V_{h+1}^k(s').
 \label{eq:detail-optimism-hypothesis}
\end{equation}
We establish all the stage-$h$ claims at an arbitrary pair $(s,a)$.

First, for the upper estimator, Lemma~\ref{lem:bern-error} yields
\begin{align}
 \hat f_h^k(s,a)+b_h^k(s,a)-Q_h^\star(s,a)
 &\ge r_h(s,a)+[P_hV_{h+1}^k](s,a)
       -b_h^k(s,a)+b_h^k(s,a)-Q_h^\star(s,a)\notag\\
 &=[P_h(V_{h+1}^k-V_{h+1}^\star)](s,a)
 \ge0.
 \label{eq:detail-optimism-upper}
\end{align}
The equality uses the Bellman optimality equation, and the final inequality
uses \eqref{eq:detail-optimism-hypothesis} and positivity of the transition
operator. Since $Q_h^\star(s,a)\in[0,H_h]$, clipping does not reverse the
comparison, and the ordering maximum can only increase the upper estimate.
Thus $\hat Q_h^k(s,a)\ge Q_h^\star(s,a)$.

Similarly, for the lower estimator,
\begin{align}
 \check f_h^k(s,a)-b_h^k(s,a)-Q_h^\star(s,a)
 &\le r_h(s,a)+[P_h\check V_{h+1}^k](s,a)
       +b_h^k(s,a)-b_h^k(s,a)-Q_h^\star(s,a)\notag\\
 &=[P_h(\check V_{h+1}^k-V_{h+1}^\star)](s,a)
 \le0.
 \label{eq:detail-optimism-lower}
\end{align}
Here we again use Lemma~\ref{lem:bern-error}, the Bellman optimality
equation, and the induction hypothesis, respectively. Clipping preserves
this comparison because $Q_h^\star(s,a)\in[0,H_h]$. We have therefore
proved
\begin{equation}
 0\le\check Q_h^k(s,a)\le Q_h^\star(s,a)
                   \le\hat Q_h^k(s,a)\le H_h.
 \label{eq:detail-optimism-interval}
\end{equation}

Next, consider the two branches of the reference test. If the test accepts,
\eqref{eq:detail-optimism-interval} and the inclusion condition imply
\begin{align*}
 Q_h^\star(s,a)
 &\in[\check Q_h^k(s,a),\hat Q_h^k(s,a)]\\
 &\subseteq[Q_{\rf;h}(s,a)-\beta/2,
                         Q_{\rf;h}(s,a)+\beta/2].
\end{align*}
Consequently,
$|Q_{\rf;h}(s,a)-Q_h^\star(s,a)|\le\beta/2<\beta$.
Definition~\ref{def:beta-sep} excludes any nonzero discrepancy of this
magnitude, so $Q_{\rf;h}(s,a)=Q_h^\star(s,a)$. This common value lies inside
the confidence interval; its projection is unchanged, giving
$Q_h^k(s,a)=Q_h^\star(s,a)$. If the test rejects, the algorithm uses
$Q_h^k(s,a)=\hat Q_h^k(s,a)\ge Q_h^\star(s,a)$.
Together, the two cases give
\[
 0\le\check Q_h^k(s,a)\le Q_h^\star(s,a)
       \le Q_h^k(s,a)\le\hat Q_h^k(s,a)\le H_h.
\]
Finally, maximizing the pointwise inequalities over actions yields
\begin{align*}
 \check V_h^k(s)
 &=\max_a\check Q_h^k(s,a)
 \le\max_aQ_h^\star(s,a)=V_h^\star(s),\\
 V_h^k(s)
 &=\max_aQ_h^k(s,a)
 \ge\max_aQ_h^\star(s,a)=V_h^\star(s).
\end{align*}
This closes the backward induction. The full chain proves
Lemma~\ref{lem:opt-pess}, the accepted-reference case proves
Lemma~\ref{lem:inclusion-opt}, and the two-branch optimism proves
Lemma~\ref{lem:opt}. No additional probability loss is incurred beyond
$\mathcal E$.
\end{proof}

\begin{proof}[Proof of Lemmas~\ref{lem:sum-bonus} and~\ref{lem:sum-bonus-subset}]
Fix a stage $h$ and an arbitrary subset $M\subseteq[K]$. For the present
calculation, write
$D_k:=D_{\mathcal F_h}((s_h^k,a_h^k);z_{[k-1],h})$.
The rounding rule and Assumption~\ref{def:oracle} give
\begin{align}
 b_h^k(s_h^k,a_h^k)
 &=\alpha_k(\bar D_{h,\kappa}^k(s_h^k,a_h^k)+\kappa)\notag\\
 &\le\alpha_k\bigl(
    \bar D_{\mathcal F_h}((s_h^k,a_h^k);z_{[k-1],h})+2\kappa\bigr)\notag\\
 &\le\alpha(\sqrt C\,D_k+2\kappa).
 \label{eq:detail-rounded-bonus-upper}
\end{align}
Also, the output bound $0\le f\le H$ and $\lambda=H^2$ imply
\begin{align}
 D_k^2
 &=\sup_{f_1,f_2\in\mathcal F_h}
 \frac{|f_1(s_h^k,a_h^k)-f_2(s_h^k,a_h^k)|^2}
      {H^2+\sum_{i<k}|f_1(s_h^i,a_h^i)-f_2(s_h^i,a_h^i)|^2}
 \le1.
 \label{eq:detail-uncertainty-at-most-one}
\end{align}
Since $H\ge1$, the truncation in Definition~\ref{def:generalized-eluder}
therefore leaves these uncertainty terms unchanged. In particular,
\begin{equation}
 \sum_{k\in M}D_k^2
 \le\sum_{k=1}^K D_k^2
 =\sum_{k=1}^K\min\{H^2,D_k^2\}
 \le\dim_K(\mathcal F_h)=d_h.
 \label{eq:detail-subset-uncertainty}
\end{equation}
It follows that
\begin{align}
 \sum_{k\in M}c_h^k(s_h^k,a_h^k)
 &\overset{(a)}{\le}4\sum_{k\in M}b_h^k(s_h^k,a_h^k)\notag\\
 &\overset{(b)}{\le}4\sqrt C\,\alpha\sum_{k\in M}D_k
                              +8\alpha\kappa|M|\notag\\
 &\overset{(c)}{\le}4\sqrt C\,\alpha
             \sqrt{|M|\sum_{k\in M}D_k^2}+8\alpha\kappa|M|\notag\\
 &\overset{(d)}{\le}4\sqrt C\,\alpha\sqrt{|M|d_h}
                              +8\alpha\kappa|M|.
 \label{eq:detail-subset-bonus}
\end{align}
Here (a) uses $c_h^k=\min\{4b_h^k,H\}$; (b) uses
\eqref{eq:detail-rounded-bonus-upper}; (c) is Cauchy--Schwarz; and (d) uses
\eqref{eq:detail-subset-uncertainty}. If $M=\varnothing$, all these sums
are zero. Every step is deterministic for a realized sequence of
observations, so $M$ may be data-dependent. This proves
Lemma~\ref{lem:sum-bonus-subset}. Taking $M=[K]$ proves
Lemma~\ref{lem:sum-bonus}.
\end{proof}

\begin{proof}[Proof of Lemma~\ref{lem:telescope}]
Fix an episode $k$. Its planning functions and greedy policy are fixed
conditional on $\mathcal H_{k,0}$. Define
\[
 e_h^k(s,a):=Q_h^k(s,a)-r_h(s,a)-[P_hV_{h+1}^k](s,a),\qquad
 \Delta_h^k(s):=V_h^k(s)-V_h^{\pi^k}(s).
\]
For $a=\pi_h^k(s)$, the greedy choice and the policy Bellman equation give
\begin{align}
 \Delta_h^k(s)
 &=Q_h^k(s,a)-Q_h^{\pi^k}(s,a)\notag\\
 &=Q_h^k(s,a)-r_h(s,a)-[P_hV_{h+1}^{\pi^k}](s,a)\notag\\
 &=Q_h^k(s,a)-r_h(s,a)-[P_hV_{h+1}^k](s,a)
       +[P_h(V_{h+1}^k-V_{h+1}^{\pi^k})](s,a)\notag\\
 &=e_h^k(s,a)+[P_h\Delta_{h+1}^k](s,a).
 \label{eq:detail-bellman-difference-recursion}
\end{align}
Evaluate this identity along the episode and take conditional expectations.
The tower property yields
\begin{align*}
 \mathbb E[\Delta_h^k(s_h^k)\mid\mathcal H_{k,0}]
 &=\mathbb E[e_h^k(s_h^k,a_h^k)\mid\mathcal H_{k,0}]
   +\mathbb E[\Delta_{h+1}^k(s_{h+1}^k)\mid\mathcal H_{k,0}].
\end{align*}
Summing this equality over $h=1,\ldots,H$ cancels the intermediate value
differences. Since $s_1^k=s_1$ and $\Delta_{H+1}^k=0$, we obtain
\begin{equation}
 V_1^k(s_1)-V_1^{\pi^k}(s_1)
 =\mathbb E\!\left[\sum_{h=1}^H e_h^k(s_h^k,a_h^k)
                              \,\middle|\,\mathcal H_{k,0}\right].
 \label{eq:detail-conditional-bellman-identity}
\end{equation}
This identity is algebraic and remains valid without optimism.

We next bound the residual on $\mathcal E$ in the well-specified case.
If the inclusion test fails, $Q_h^k=\hat Q_h^k$, and
Lemma~\ref{lem:bern-error}, specifically
\eqref{eq:detail-confidence-sandwich}, gives
\[
 e_h^k\le2b_h^k.
\]
Moreover, $Q_h^k\le H_h\le H$ and
$r_h+P_hV_{h+1}^k\ge0$, so $e_h^k\le H$. Hence
$e_h^k\le\min\{2b_h^k,H\}\le c_h^k$ on this branch.
If the test accepts, Lemma~\ref{lem:inclusion-opt} gives
$Q_h^k=Q_h^\star$, and Lemma~\ref{lem:opt} implies
\begin{align*}
 e_h^k(s,a)
 &=Q_h^\star(s,a)-r_h(s,a)-[P_hV_{h+1}^k](s,a)\\
 &=[P_h(V_{h+1}^\star-V_{h+1}^k)](s,a)\le0.
\end{align*}
Combining the two branches, for every $(s,a)$,
\begin{equation}
 e_h^k(s,a)\le I_h^k(s,a)c_h^k(s,a).
 \label{eq:detail-residual-indicator}
\end{equation}
Finally, use Lemma~\ref{lem:opt} at the initial state, then
\eqref{eq:detail-conditional-bellman-identity} and
\eqref{eq:detail-residual-indicator}:
\begin{align*}
 \mathrm{Regret}(K)
 &=\sum_{k=1}^K\bigl(V_1^\star(s_1)-V_1^{\pi^k}(s_1)\bigr)\\
 &\le\sum_{k=1}^K\bigl(V_1^k(s_1)-V_1^{\pi^k}(s_1)\bigr)\\
 &\le\sum_{k=1}^K\mathbb E\!\left[
       \sum_{h=1}^H I_h^k(s_h^k,a_h^k)c_h^k(s_h^k,a_h^k)
                              \,\middle|\,\mathcal H_{k,0}\right].
\end{align*}
The conditional expectations are with respect to the ordinary episode law;
we do not condition that law on $\mathcal E$. The use of $\mathcal E$ is
only to establish the pointwise comparisons for the already planned
functions. Each trajectory cost lies in $[0,H^2]$, as claimed.
\end{proof}

\begin{proof}[Proof of Lemma~\ref{lem:telescope-gap}]
We first establish a pointwise width recursion, then telescope its
conditional expectations, and finally control predictably selected sums.
The argument uses regression confidence and interval projection, but not
exact optimism relative to $Q^\star$.

Let $\Delta V_h^k:=V_h^k-\check V_h^k$. The projection and the ordered
interval in Algorithm~\ref{alg:main-gfa} imply
$\Delta V_h^k\ge0$ on every history. On $\mathcal E$,
\eqref{eq:detail-confidence-sandwich}, which follows from
Lemma~\ref{lem:bern-error}, yields
\begin{align}
 w_h^k(s,a)
 &=\hat Q_h^k(s,a)-\check Q_h^k(s,a)\notag\\
 &\le\bigl(r_h(s,a)+[P_hV_{h+1}^k](s,a)+2b_h^k(s,a)\bigr)\notag\\
 &\qquad-\bigl(r_h(s,a)+[P_h\check V_{h+1}^k](s,a)
                                          -2b_h^k(s,a)\bigr)\notag\\
 &=4b_h^k(s,a)+[P_h\Delta V_{h+1}^k](s,a).
 \label{eq:detail-width-unclipped}
\end{align}
If $4b_h^k(s,a)\le H$, the bonus in this inequality equals $c_h^k(s,a)$.
If $4b_h^k(s,a)>H$, use $w_h^k(s,a)\le H=c_h^k(s,a)$ and
$P_h\Delta V_{h+1}^k\ge0$. Thus both cases give
\begin{equation}
 w_h^k(s,a)
 \le c_h^k(s,a)+[P_h\Delta V_{h+1}^k](s,a).
 \label{eq:detail-width-recursion}
\end{equation}
For the greedy action $a=\pi_h^k(s)$,
\begin{align}
 \Delta V_h^k(s)
 &=Q_h^k(s,a)-\max_{a'}\check Q_h^k(s,a')\notag\\
 &\le Q_h^k(s,a)-\check Q_h^k(s,a)
 \le\hat Q_h^k(s,a)-\check Q_h^k(s,a)=w_h^k(s,a).
 \label{eq:detail-greedy-gap}
\end{align}
The first inequality evaluates the lower maximum at $a$; the second uses
$Q_h^k\le\hat Q_h^k$.

For $h<H$, applying \eqref{eq:detail-width-recursion} at the visited pair
and then \eqref{eq:detail-greedy-gap} at the next state gives
\begin{align*}
 w_h^k(s_h^k,a_h^k)
 &\le c_h^k(s_h^k,a_h^k)
    +\mathbb E[\Delta V_{h+1}^k(s_{h+1}^k)\mid\mathcal H_{k,h}]\\
 &\le c_h^k(s_h^k,a_h^k)
    +\mathbb E[w_{h+1}^k(s_{h+1}^k,a_{h+1}^k)
                                      \mid\mathcal H_{k,h}]\\
 &\le\mathbb E\!\left[
       c_h^k(s_h^k,a_h^k)+c_{h+1}^k(s_{h+1}^k,a_{h+1}^k)
       +\Delta V_{h+2}^k(s_{h+2}^k)
                                      \,\middle|\,\mathcal H_{k,h}\right].
\end{align*}
The last line uses the same recursion at stage $h+1$ and the tower
property. Continuing through stage $H$ and using
$\Delta V_{H+1}^k=0$ yields
\begin{equation}
 w_h^k(s_h^k,a_h^k)
 \le\mathbb E\!\left[\sum_{j=h}^H c_j^k(s_j^k,a_j^k)
                                  \,\middle|\,\mathcal H_{k,h}\right].
 \label{eq:detail-width-suffix-expectation}
\end{equation}
In particular, at $h=H$ this says $w_H^k\le c_H^k$; terminal-stage
estimation uncertainty is included rather than set to zero.

Fix a stage $h$ and a selection rule
$J_k:=\ind\{k\in M\}$ that is $\mathcal H_{k,h}$-measurable. Define
\[
 X_k^{(h)}:=J_k\sum_{j=h}^Hc_j^k(s_j^k,a_j^k),\qquad
 m_k^{(h)}:=\mathbb E[X_k^{(h)}\mid\mathcal H_{k,h}].
\]
These selected costs satisfy $0\le X_k^{(h)}\le H(H-h+1)\le H^2$.
Because $J_k$ is known before the suffix is collected,
\begin{align*}
 J_kw_h^k(s_h^k,a_h^k)
 &\le J_k\mathbb E\!\left[\sum_{j=h}^Hc_j^k(s_j^k,a_j^k)
                                  \,\middle|\,\mathcal H_{k,h}\right]\\
 &=m_k^{(h)}.
\end{align*}
For clarity, one may take the pre-observation filtration for cost $k$ to
be $\mathcal H_{k,h}$ and the next filtration to include the complete
suffix and the next episode's prefix. These histories are nested, and
$X_k^{(h)}$ is measurable by the next filtration. Thus the first inequality
of Lemma~\ref{lem:nonnegative-predictable}, with $B=H^2$, applies and gives
\begin{align*}
 \sum_{k\in M}w_h^k(s_h^k,a_h^k)
 &\le\sum_{k=1}^K m_k^{(h)}\\
 &\le2\sum_{k=1}^K X_k^{(h)}+2H^2\log(1/\eta)\\
 &=2\sum_{k\in M}\sum_{j=h}^Hc_j^k(s_j^k,a_j^k)
                                      +2H^2\log(1/\eta).
\end{align*}
The concentration event is defined for the costs without conditioning on
$\mathcal E$; its intersection with $\mathcal E$ gives the stated result.
This is a guarantee for the chosen predictable selection rule, not a
simultaneous guarantee over arbitrary subsets selected after observing
their suffixes.
\end{proof}

\begin{proof}[Proof of Lemma~\ref{lem:mh-upper}]
Fix $h$ and let $n:=|\mathcal M_h|$. If $n=0$, the conclusion is immediate.
Otherwise, the defining width threshold gives
\begin{equation}
 \frac\beta4n
 \le\sum_{k\in\mathcal M_h}w_h^k(s_h^k,a_h^k).
 \label{eq:detail-wide-width-lower}
\end{equation}
Membership in $\mathcal M_h$ is known from the planning functions and
$(s_h^k,a_h^k)$ before the stage-$h$ transition. We may therefore apply
Lemma~\ref{lem:telescope-gap} with this selection rule and failure
probability $\delta/H$. A union bound over $h$ costs at most $\delta$.
On the resulting event intersected with $\mathcal E$,
\begin{align}
 \frac\beta4n
 &\overset{(a)}{\le}
  2\sum_{j=h}^H\sum_{k\in\mathcal M_h}c_j^k(s_j^k,a_j^k)
                                      +2H^2\ell\notag\\
 &\overset{(b)}{\le}
  8\sqrt C\,\alpha\sqrt n\sum_{j=h}^H\sqrt{d_j}
       +16\alpha\kappa(H-h+1)n+2H^2\ell\notag\\
 &\overset{(c)}{\le}
  8\sqrt C\,\alpha\sqrt{n(H-h+1)\sum_{j=h}^Hd_j}
       +16\alpha\kappa(H-h+1)n+2H^2\ell\notag\\
 &\overset{(d)}{\le}
  8\sqrt C\,\alpha H\sqrt{nd_{\mathcal F}}+16\alpha+2H^2\ell.
 \label{eq:detail-wide-quadratic}
\end{align}
Here (a) combines \eqref{eq:detail-wide-width-lower} with
Lemma~\ref{lem:telescope-gap} and $\log(H/\delta)\le\ell$.
For (b), apply Lemma~\ref{lem:sum-bonus-subset} separately at each suffix
stage $j$, always to the same realized subset $\mathcal M_h$; the lemma
allows this subset to be data-dependent. Step (c) is Cauchy--Schwarz over
the suffix stages. Finally, (d) uses
\[
 H-h+1\le H,\quad \sum_{j=h}^Hd_j\le Hd_{\mathcal F},\quad
 n\le K,\quad \kappa HK=1.
\]

To solve \eqref{eq:detail-wide-quadratic} explicitly, set
$A=8\sqrt C\,\alpha H\sqrt{d_{\mathcal F}}$ and
$B_0=16\alpha+2H^2\ell$. Young's inequality gives
\[
 A\sqrt n\le\frac\beta8n+\frac{2A^2}{\beta}.
\]
Substituting into $(\beta/4)n\le A\sqrt n+B_0$ and rearranging yields
\begin{align}
 \frac\beta8n&\le\frac{2A^2}{\beta}+B_0,\notag\\
 n&\le\frac{16A^2}{\beta^2}+\frac{8B_0}{\beta}\notag\\
  &=\frac{1024C\alpha^2H^2d_{\mathcal F}}{\beta^2}
       +\frac{128\alpha}{\beta}+\frac{16H^2\ell}{\beta}.
 \label{eq:detail-wide-explicit}
\end{align}
The latter two terms are lower order under the appendix's parameter
conventions. More explicitly, $\alpha\ge1$, $d_{\mathcal F}\ge1$,
$\beta\le H$, and $\alpha^2\ge H^2\ell$ imply
\[
 \frac{\alpha}{\beta}
 \le\frac{\alpha^2H^2d_{\mathcal F}}{\beta^2},\qquad
 \frac{H^2\ell}{\beta}
 \le\frac{\alpha^2}{\beta}
 \le\frac{\alpha^2H^2d_{\mathcal F}}{\beta^2}.
\]
Absorb these terms into an absolute constant $C_1$, and combine with the
trivial bound $n\le K$. This proves
\[
 |\mathcal M_h|
 \le\min\{K,C_1H^2\alpha^2d_{\mathcal F}/\beta^2\}
\]
simultaneously for all stages. The argument retains the sum of
suffix-stage complexities and includes $j=H$ throughout.
\end{proof}

\begin{proof}[Proof of Lemma~\ref{lem:up-t1}]
Apply Lemma~\ref{lem:sum-bonus-subset} at each stage with
$M=\mathcal M_h$. Since $|\mathcal M_h|\le K$ and $\kappa HK=1$,
\begin{align}
 S_M&:=\sum_{h=1}^H\sum_{k\in\mathcal M_h}c_h^k(s_h^k,a_h^k)\notag\\
 &\le4\sqrt C\,\alpha\sum_{h=1}^H\sqrt{|\mathcal M_h|d_h}
       +8\alpha\kappa\sum_{h=1}^H|\mathcal M_h|\notag\\
 &\le4\sqrt C\,\alpha\sum_{h=1}^H\sqrt{|\mathcal M_h|d_h}
       +8\alpha.
 \label{eq:detail-wide-cost-start}
\end{align}
Use first the nontrivial cardinality estimate in
Lemma~\ref{lem:mh-upper}:
\begin{align*}
 \sum_{h=1}^H\sqrt{|\mathcal M_h|d_h}
 &\le\frac{\sqrt{C_1}H\alpha\sqrt{d_{\mathcal F}}}{\beta}
                  \sum_{h=1}^H\sqrt{d_h}\\
 &\le\frac{\sqrt{C_1}H^2\alpha d_{\mathcal F}}{\beta}.
\end{align*}
The final inequality is Cauchy--Schwarz,
$\sum_h\sqrt{d_h}\le H\sqrt{d_{\mathcal F}}$.
Substituting into \eqref{eq:detail-wide-cost-start} gives
\[
 S_M\le4\sqrt{CC_1}\,
         \frac{H^2\alpha^2d_{\mathcal F}}{\beta}+8\alpha.
\]
Independently, use the trivial cardinality bound in the same expression:
\begin{align*}
 S_M
 &\le4\sqrt C\,\alpha\sqrt K\sum_{h=1}^H\sqrt{d_h}+8\alpha\\
 &\le4\sqrt C\,\alpha H\sqrt{Kd_{\mathcal F}}+8\alpha.
\end{align*}
Both inequalities hold on the event in Lemma~\ref{lem:mh-upper}, so
we may take their minimum. This yields
\[
 S_M\lesssim\min\left\{
 \frac{H^2\alpha^2d_{\mathcal F}}{\beta},\quad
 \alpha H\sqrt{Kd_{\mathcal F}}\right\}+\alpha,
\]
as required.
\end{proof}

\begin{proof}[Proof of Lemma~\ref{lem:bounded-vh}]
We first identify which reference errors can generate a narrow rejected
interval, and then apply concentration to the true error indicators.

Suppose $\mathcal E$ holds and fix a visited pair with
$Q_{\rf;h}(s_h^k,a_h^k)=Q_h^\star(s_h^k,a_h^k)$ and
$w_h^k(s_h^k,a_h^k)\le\beta/4$.
For readability, suppress this pair's arguments below.
Lemma~\ref{lem:opt-pess} gives $\check Q_h^k\le Q_h^\star\le\hat Q_h^k$.
Therefore
\begin{align*}
 \check Q_h^k
 &=\hat Q_h^k-w_h^k
 \ge Q_h^\star-\beta/4
 =Q_{\rf;h}-\beta/4
 \ge Q_{\rf;h}-\beta/2,\\
 \hat Q_h^k
 &=\check Q_h^k+w_h^k
 \le Q_h^\star+\beta/4
 =Q_{\rf;h}+\beta/4
 \le Q_{\rf;h}+\beta/2.
\end{align*}
Thus the entire interval is inside the reference neighborhood, so the
test accepts. Taking the contrapositive and using the definition of
$\mathcal V_h$ gives, on $\mathcal E$,
\begin{equation}
 \ind\{k\in\mathcal V_h\}
 \le Y_{k,h}:=
 \ind\{Q_{\rf;h}(s_h^k,a_h^k)\ne Q_h^\star(s_h^k,a_h^k)\}.
 \label{eq:detail-narrow-error-domination}
\end{equation}

Conditional on $\mathcal H_{k,0}$, the algorithm's policy $\pi^k$ is fixed,
and its stage-$h$ pair has occupancy $d_h^{\pi^k}$. Consequently,
\begin{align}
 \mathbb E[Y_{k,h}\mid\mathcal H_{k,0}]
 &=\mathbb E_{(s,a)\sim d_h^{\pi^k}}
      \big[\ind\{Q_{\rf;h}(s,a)\ne Q_h^\star(s,a)\}\big]\notag\\
 &\le\rho.
 \label{eq:detail-true-error-mean}
\end{align}
The last inequality is precisely Assumption~\ref{asm:bounded-coverage},
which holds for every policy and hence for the policy selected after any
training history. Since $Y_{k,h}\in[0,1]$, the second inequality of
Lemma~\ref{lem:nonnegative-predictable}, with $B=1$ and failure probability
$\delta/H$, gives
\begin{align*}
 \sum_{k=1}^K Y_{k,h}
 &\le2\sum_{k=1}^K\mathbb E[Y_{k,h}\mid\mathcal H_{k,0}]
                                           +\log(H/\delta)\\
 &\le2\rho K+\log(H/\delta).
\end{align*}
A union bound makes this simultaneous for all $h$ with failure probability
at most $\delta$. Intersect this event with $\mathcal E$ and apply
\eqref{eq:detail-narrow-error-domination} to conclude
\[
 |\mathcal V_h|\le\sum_{k=1}^K Y_{k,h}
                         \le2\rho K+\log(H/\delta).
\]
Concentration was applied to $Y_{k,h}$ without conditioning on the future
confidence event. If $\rho=0$, the nonnegative variable $Y_{k,h}$ has zero
conditional expectation by \eqref{eq:detail-true-error-mean}, so it is zero
almost surely. A finite union over $h,k$ then gives $|\mathcal V_h|=0$ on
$\mathcal E$.
\end{proof}

\begin{proof}[Proof of Lemma~\ref{lem:up-t3}]
By Lemma~\ref{lem:sum-bonus-subset}, followed by
Lemma~\ref{lem:bounded-vh},
\begin{align}
 S_V&:=\sum_{h=1}^H\sum_{k\in\mathcal V_h}c_h^k(s_h^k,a_h^k)\notag\\
 &\le4\sqrt C\,\alpha\sum_{h=1}^H\sqrt{|\mathcal V_h|d_h}
                 +8\alpha\kappa\sum_{h=1}^H|\mathcal V_h|\notag\\
 &\le4\sqrt C\,\alpha
        \sqrt{2\rho K+\log(H/\delta)}\sum_{h=1}^H\sqrt{d_h}
                 +8\alpha\kappa HK\notag\\
 &\le4\sqrt C\,\alpha H
        \sqrt{d_{\mathcal F}(2\rho K+\log(H/\delta))}+8\alpha\notag\\
 &\lesssim\alpha H\sqrt{d_{\mathcal F}(\rho K+\ell)}+\alpha.
 \label{eq:detail-narrow-cost-bound}
\end{align}
The third line also uses $|\mathcal V_h|\le K$ for the rounding term; the
fourth uses Cauchy--Schwarz and $\kappa HK=1$; and the last uses
$\log(H/\delta)\le\ell$.

For the additional assertion, if a pair passes the test,
Lemma~\ref{lem:inclusion-opt} gives $Q_h^k=Q_h^\star$ and
Lemma~\ref{lem:opt} gives $V_{h+1}^k\ge V_{h+1}^\star$. Hence
\begin{align*}
 Q_h^k(s,a)-r_h(s,a)-[P_hV_{h+1}^k](s,a)
 &=[P_h(V_{h+1}^\star-V_{h+1}^k)](s,a)\le0.
\end{align*}
Its positive part is therefore zero. This remains true whether or not the
accepted interval belongs to the wide set $\mathcal M_h$.
\end{proof}

\section{Proof of Theorem~\ref{thm:upper-mis}}
\label{app:proof-mis}
\subsection{Sets}
The algorithm now receives $\tilde Q_\rf$, and its acceptance test is centered at this observed reference. Define
\begin{align}
 \widetilde{\mathcal M}_h&:=\{k:w_h^k(s_h^k,a_h^k)>\beta/4\},\label{eqn:mh-set-mis}\\
 \widetilde{\mathcal V}_h&:=\{k\notin\widetilde{\mathcal M}_h:
 [\check Q_h^k(s_h^k,a_h^k),\hat Q_h^k(s_h^k,a_h^k)]
 \not\subseteq[\tilde Q_{\rf;h}(s_h^k,a_h^k)-\beta/2,
                \tilde Q_{\rf;h}(s_h^k,a_h^k)+\beta/2]\}.
 \label{eqn:vh-set-mis}
\end{align}
These sets use the actual algorithmic test, not an enlarged neighborhood around the unobserved ideal reference. Let $\widetilde I_h^k$ indicate failure of this test. The regression event $\mathcal E$ and the nonnegative costs $c_h^k$ are defined as in Appendix~\ref{app:proof-upper}, conditional on the fixed observed reference. The finite-family concentration proof is unchanged.

\subsection{Optimism Lemmas under Misspecification}
\begin{lemma}[Certification under small misspecification]\label{lem:inclusion-opt-mis}
On $\mathcal E$, if $0\le\tau\le\beta/8$, every accepted pair satisfies
$Q_{\rf;h}(s,a)=Q_h^\star(s,a)$ and
\[
 |Q_h^k(s,a)-Q_h^\star(s,a)|\le\tau.
\]
Here $Q_h^k$ is the projection of $\tilde Q_{\rf;h}$ onto the confidence interval, as specified in Algorithm~\ref{alg:main-gfa}.
\end{lemma}
\begin{lemma}[Approximate optimism]\label{lem:opt-mis}
On $\mathcal E$, pessimism $\check Q_h^k\le Q_h^\star$ holds for every fixed observed reference. If $\tau\le\beta/8$, then
\[
 \hat Q_h^k\ge Q_h^\star-\tau,\qquad
 Q_h^k\ge Q_h^\star-\tau,\qquad V_h^k\ge V_h^\star-\tau.
\]
For an arbitrary observed reference, the weaker bounds
\[
 Q_h^k\ge Q_h^\star-(H-h+1)\beta/2,\qquad
 V_h^k\ge V_h^\star-(H-h+1)\beta/2
\]
hold. These statements do not require the learner to know $\tau$.
\end{lemma}

\subsection{Misspecification Lemmas}
\begin{lemma}\label{lem:bounded-vh-mis}
If $\tau\le\beta/8$, then on $\mathcal E$ and an additional event of probability at least $1-\delta$,
\[
 |\widetilde{\mathcal V}_h|\le2\rho K+\log(H/\delta),\qquad h\in[H].
\]
\end{lemma}
\begin{lemma}[Decomposition]\label{lem:telescope-mis}
On $\mathcal E$, for $\tau\le\beta/8$,
\[
 \mathrm{Regret}(K)\le(2H+1)K\tau+
 \sum_{k=1}^K\mathbb E\!\left[
   \sum_{h=1}^H\widetilde I_h^k(s_h^k,a_h^k)c_h^k(s_h^k,a_h^k)
   \,\middle|\,\mathcal H_{k,0}\right].
\]
For arbitrary $\tilde Q_\rf$, we also have both
\begin{align*}
 \mathrm{Regret}(K)&\le\frac34HK\beta+
  \sum_k\mathbb E\!\left[\sum_h\ind\{k\in\widetilde{\mathcal M}_h\}c_h^k
                  \,\middle|\,\mathcal H_{k,0}\right],\\
 \mathrm{Regret}(K)&\le\frac12HK\beta+
  \sum_k\mathbb E\!\left[\sum_hc_h^k
                  \,\middle|\,\mathcal H_{k,0}\right],
\end{align*}
where trajectory costs are evaluated at the visited state-action pairs.
\end{lemma}
\begin{lemma}[Upper bound on $|\widetilde{\mathcal M}_h|$]\label{lem:mh-upper-mis}
For any observed reference, on $\mathcal E$ and an additional event of probability at least $1-\delta$,
\begin{equation}
 |\widetilde{\mathcal M}_h|\le\min\{K,C_1H^2\alpha^2d_{\mathcal F}/\beta^2\},
 \qquad h\in[H].\label{eqn:mh-upper-mis}
\end{equation}
\end{lemma}
\begin{lemma}\label{lem:up-t1-mis}
On the events of Lemma~\ref{lem:mh-upper-mis},
\[
 \sum_h\sum_{k\in\widetilde{\mathcal M}_h}c_h^k(s_h^k,a_h^k)
 \lesssim\min\{H^2\alpha^2d_{\mathcal F}/\beta,
                     \alpha H\sqrt{Kd_{\mathcal F}}\}+\alpha.
\]
\end{lemma}
\begin{lemma}\label{lem:up-t3-mis}
If $\tau\le\beta/8$, then on the events of Lemma~\ref{lem:bounded-vh-mis},
\[
 \sum_h\sum_{k\in\widetilde{\mathcal V}_h}c_h^k(s_h^k,a_h^k)
 \lesssim\alpha H\sqrt{d_{\mathcal F}(\rho K+\ell)}+\alpha.
\]
For any observed reference, the Bellman residual on a narrow interval is at most $\beta/4$, regardless of whether the reference is accepted.
\end{lemma}

\subsection{Detailed Proof}
\begin{proof}[Proof of Theorem~\ref{thm:upper-mis}]
The proof distinguishes two ranges of $\tau$. When the error is small,
certification can still identify exact entries of the ideal reference.
For the remaining range, a bound proportional to $HK\beta$ suffices
because $\beta<8\tau$. Throughout, the algorithm receives only
$\tilde Q_\rf$; this case distinction is for the analysis and is not used
by the learner.

First, consider $0<\tau\le\beta/8$. Define the nonnegative episode cost
\[
 \widetilde X_k:=\sum_{h=1}^H
   \widetilde I_h^k(s_h^k,a_h^k)c_h^k(s_h^k,a_h^k),\qquad
 \widetilde m_k:=\mathbb E[\widetilde X_k\mid\mathcal H_{k,0}].
\]
As before, $0\le\widetilde X_k\le H^2$ on every history. Apply
Lemma~\ref{lem:nonnegative-predictable} with $B=H^2$ and failure
probability $\delta$. Intersect this event with $\mathcal E$ and the
additional events of Lemmas~\ref{lem:mh-upper-mis} and
\ref{lem:bounded-vh-mis}. The union bound gives total failure probability
at most $2\delta+\delta+\delta+\delta=5\delta$.
On this intersection,
\begin{align}
 \mathrm{Regret}(K)
 &\overset{(a)}{\le}(2H+1)K\tau+\sum_{k=1}^K\widetilde m_k\notag\\
 &\overset{(b)}{\le}(2H+1)K\tau
      +2\sum_{k=1}^K\widetilde X_k+2H^2\ell\notag\\
 &\overset{(c)}{\le}(2H+1)K\tau
      +2\sum_{h=1}^H\sum_{k\in\widetilde{\mathcal M}_h}
                          c_h^k(s_h^k,a_h^k)\notag\\
 &\qquad+2\sum_{h=1}^H\sum_{k\in\widetilde{\mathcal V}_h}
                          c_h^k(s_h^k,a_h^k)+2H^2\ell\notag\\
 &\overset{(d)}{\lesssim}HK\tau
       +\frac{H^2\alpha^2d_{\mathcal F}}{\beta}
       +\alpha H\sqrt{d_{\mathcal F}(\rho K+\ell)}
       +\alpha+H^2\ell.
 \label{eq:detail-mis-small-raw}
\end{align}
Here (a) uses the first decomposition in Lemma~\ref{lem:telescope-mis};
(b) uses the cost concentration inequality; (c) uses
\eqref{eqn:mh-set-mis}--\eqref{eqn:vh-set-mis}; and (d) uses
Lemmas~\ref{lem:up-t1-mis} and~\ref{lem:up-t3-mis}, together with
$(2H+1)K\tau\le3HK\tau$.
Applying $\sqrt{\rho K+\ell}\le\sqrt{\rho K}+\sqrt\ell$ and the three
explicit remainder bounds in \eqref{eq:detail-remainder-absorption} gives
\begin{equation}
 \mathrm{Regret}(K)
 \lesssim\frac{H^2\alpha^2d_{\mathcal F}}{\beta}
       +\alpha H\sqrt{\rho Kd_{\mathcal F}}+HK\tau.
 \label{eq:detail-mis-small-adaptation}
\end{equation}

For the online branch, use $\widetilde I_h^k\le1$ directly in (b), then
apply Lemma~\ref{lem:sum-bonus}:
\begin{align}
 \mathrm{Regret}(K)
 &\le(2H+1)K\tau
       +2\sum_{h=1}^H\sum_{k=1}^Kc_h^k(s_h^k,a_h^k)+2H^2\ell\notag\\
 &\le(2H+1)K\tau
       +8\sqrt C\,\alpha\sqrt K\sum_{h=1}^H\sqrt{d_h}
       +16\alpha\kappa KH+2H^2\ell\notag\\
 &\le3HK\tau+8\sqrt C\,\alpha H\sqrt{Kd_{\mathcal F}}
                         +16\alpha+2H^2\ell.
 \label{eq:detail-mis-small-online-raw}
\end{align}
For $K\ge\ell$, the last two terms are absorbed exactly as in
\eqref{eq:detail-online-with-remainders}. For $K<\ell$, the deterministic
bound \eqref{eq:detail-small-budget} applies to the true regret, whether
or not the planning values are optimistic. Thus
\begin{equation}
 \mathrm{Regret}(K)\lesssim
          \alpha H\sqrt{Kd_{\mathcal F}}+HK\tau.
 \label{eq:detail-mis-small-online}
\end{equation}
Both \eqref{eq:detail-mis-small-adaptation} and
\eqref{eq:detail-mis-small-online} hold on the same event; their minimum
proves the required bound for this error range.

Next, consider $\beta/8<\tau\le\beta/2$. In this range, we do not rely on certification of the ideal reference or
on an inaccurate-region count. Define the wide-interval and all-bonus
costs
\begin{align*}
 X_k^M&:=\sum_{h=1}^H\ind\{k\in\widetilde{\mathcal M}_h\}
                                  c_h^k(s_h^k,a_h^k),\\
 X_k^{\mathrm{all}}&:=\sum_{h=1}^H c_h^k(s_h^k,a_h^k).
\end{align*}
Both costs lie in $[0,H^2]$. Apply the first inequality of
Lemma~\ref{lem:nonnegative-predictable} to each cost sequence with failure
probability $\delta/2$. Together with $\mathcal E$ and the additional
wide-interval event in Lemma~\ref{lem:mh-upper-mis}, the total failure
probability is at most
$2\delta+\delta+\delta/2+\delta/2=4\delta\le5\delta$.

The first generic decomposition of Lemma~\ref{lem:telescope-mis} and
Lemma~\ref{lem:up-t1-mis} imply
\begin{align}
 \mathrm{Regret}(K)
 &\le\frac34HK\beta
       +\sum_{k=1}^K\mathbb E[X_k^M\mid\mathcal H_{k,0}]\notag\\
 &\le\frac34HK\beta+2\sum_{k=1}^K X_k^M
                                  +2H^2\log(2/\delta)\notag\\
 &\lesssim HK\beta
       +\frac{H^2\alpha^2d_{\mathcal F}}{\beta}+\alpha+H^2\ell\notag\\
 &\lesssim HK\beta+\frac{H^2\alpha^2d_{\mathcal F}}{\beta}.
 \label{eq:detail-mis-large-adaptation}
\end{align}
The last line uses \eqref{eq:detail-remainder-absorption}.
Independently, the second generic decomposition in
Lemma~\ref{lem:telescope-mis} and Lemma~\ref{lem:sum-bonus} give
\begin{align}
 \mathrm{Regret}(K)
 &\le\frac12HK\beta
       +\sum_{k=1}^K\mathbb E[X_k^{\mathrm{all}}\mid\mathcal H_{k,0}]\notag\\
 &\le\frac12HK\beta+2\sum_{k=1}^K X_k^{\mathrm{all}}
                                  +2H^2\log(2/\delta)\notag\\
 &\lesssim HK\beta+\alpha H\sqrt{Kd_{\mathcal F}}
                                  +\alpha+H^2\ell.
 \label{eq:detail-mis-large-online-raw}
\end{align}
For $K\ge\ell$ absorb the last two terms into
$\alpha H\sqrt{Kd_{\mathcal F}}$, and for $K<\ell$ use
\eqref{eq:detail-small-budget}. Combining the two resulting bounds gives
\begin{align}
 \mathrm{Regret}(K)
 &\lesssim\min\left\{
       \frac{H^2\alpha^2d_{\mathcal F}}{\beta},\quad
       \alpha H\sqrt{Kd_{\mathcal F}}\right\}+HK\beta\notag\\
 &\lesssim\min\left\{
       \frac{H^2\alpha^2d_{\mathcal F}}{\beta}
         +\alpha H\sqrt{\rho Kd_{\mathcal F}},\quad
       \alpha H\sqrt{Kd_{\mathcal F}}\right\}+HK\tau.
 \label{eq:detail-mis-large-final}
\end{align}
For the second inequality, the added square-root term is nonnegative, and
$\beta<8\tau$ in this range. The absolute factor eight is absorbed in the
implicit constant.

Finally, the two cases cover the full interval $0<\tau\le\beta/2$.
Substitute $\alpha=c_\alpha H\sqrt{\mathcal L_K}$ as in the proof of
Theorem~\ref{thm:upper}. This yields the same two leading regret terms,
with the additional $\mathcal O(HK\tau)$ term, and proves the result with
probability at least $1-5\delta$ under the appendix's complexity convention.
Neither case requires the algorithm to know $\tau$.
\end{proof}

\subsection{Detailed Proof for Supporting Lemmas}
\begin{proof}[Joint proof of Lemmas~\ref{lem:inclusion-opt-mis} and~\ref{lem:opt-mis}]
Fix an episode and work on $\mathcal E$. We first prove pessimism without
any restriction on the observed reference. We then establish small-error
certification and approximate optimism jointly, followed by the weaker
bound for arbitrary references.

At stage $H+1$, $\check V_{H+1}^k=V_{H+1}^\star=0$.
Suppose $\check V_{h+1}^k\le V_{h+1}^\star$. By
Lemma~\ref{lem:bern-error} and the Bellman optimality equation,
\begin{align*}
 \check f_h^k(s,a)-b_h^k(s,a)
 &\le r_h(s,a)+[P_h\check V_{h+1}^k](s,a)\\
 &\le r_h(s,a)+[P_hV_{h+1}^\star](s,a)
 =Q_h^\star(s,a).
\end{align*}
As $Q_h^\star(s,a)\in[0,H_h]$, clipping the left side to $[0,H_h]$
preserves its upper bound. Therefore $\check Q_h^k\le Q_h^\star$ and,
after maximizing over actions, $\check V_h^k\le V_h^\star$. Backward
induction proves these inequalities for every fixed observed reference.

Assume $0\le\tau\le\beta/8$. The terminal continuation is zero, so
$V_{H+1}^k\ge V_{H+1}^\star-\tau$. Suppose inductively that
\begin{equation}
 V_{h+1}^k(s')\ge V_{h+1}^\star(s')-\tau
 \qquad\text{for every }s'.
 \label{eq:detail-mis-small-hypothesis}
\end{equation}
By \eqref{eq:detail-confidence-sandwich} and the fact that each
$P_h(\cdot\mid s,a)$ is a probability distribution,
\begin{align}
 \hat Q_h^k(s,a)
 &\ge r_h(s,a)+[P_hV_{h+1}^k](s,a)\notag\\
 &\ge r_h(s,a)+[P_hV_{h+1}^\star](s,a)-\tau
 =Q_h^\star(s,a)-\tau.
 \label{eq:detail-mis-upper-approximate}
\end{align}
Notice that this is approximate, rather than exact, optimism.

If the test accepts at $(s,a)$, its containment condition gives
\[
 \tilde Q_{\rf;h}-\beta/2\le\check Q_h^k
 \le\hat Q_h^k\le\tilde Q_{\rf;h}+\beta/2,
\]
where the common arguments $(s,a)$ are suppressed.
Using the pessimism established above and \eqref{eq:detail-mis-upper-approximate}, respectively,
we obtain both sides of the following interval:
\begin{align}
 \tilde Q_{\rf;h}
 &\le\check Q_h^k+\beta/2\le Q_h^\star+\beta/2,\notag\\
 \tilde Q_{\rf;h}
 &\ge\hat Q_h^k-\beta/2\ge Q_h^\star-\tau-\beta/2.
 \label{eq:detail-mis-observed-interval}
\end{align}
Definition~\ref{def:qref-mis} states
$|Q_{\rf;h}-\tilde Q_{\rf;h}|\le\tau$. Combining this with
\eqref{eq:detail-mis-observed-interval} yields
\begin{align*}
 Q_h^\star-\beta/2-2\tau
 &\le Q_{\rf;h}\le Q_h^\star+\beta/2+\tau,\\
 |Q_{\rf;h}-Q_h^\star|
 &\le\beta/2+2\tau\le3\beta/4<\beta.
\end{align*}
By Definition~\ref{def:beta-sep}, the ideal reference must therefore be
exact: $Q_{\rf;h}=Q_h^\star$. It follows that
$|\tilde Q_{\rf;h}-Q_h^\star|\le\tau$.

It remains to check that projection preserves this error bound even when
$Q_h^\star$ is not inside the interval. Let
$l=\check Q_h^k$, $u=\hat Q_h^k$, and $q=\tilde Q_{\rf;h}$.
The inequalities established above say
\[
 l\le Q_h^\star,\qquad u\ge Q_h^\star-\tau,\qquad
 Q_h^\star-\tau\le q\le Q_h^\star+\tau.
\]
The two endpoints needed for projection satisfy
\begin{align*}
 \max\{l,q\}&\ge q\ge Q_h^\star-\tau,\qquad
 u\ge Q_h^\star-\tau,\\
 \max\{l,q\}&\le Q_h^\star+\tau.
\end{align*}
Therefore
\begin{align}
 Q_h^\star-\tau
 &\le\min\{u,\max\{l,q\}\}
 =\Pi_{[l,u]}(q)=Q_h^k
 \le Q_h^\star+\tau.
 \label{eq:detail-mis-projection-error}
\end{align}
If the test rejects, the algorithm instead sets
$Q_h^k=\hat Q_h^k\ge Q_h^\star-\tau$ by
\eqref{eq:detail-mis-upper-approximate}. Thus both branches satisfy
$Q_h^k(s,a)\ge Q_h^\star(s,a)-\tau$. Maximizing over actions gives
\[
 V_h^k(s)=\max_aQ_h^k(s,a)
 \ge\max_a\bigl(Q_h^\star(s,a)-\tau\bigr)
 =V_h^\star(s)-\tau.
\]
This closes the induction and proves small-error certification and all
small-error claims in Lemma~\ref{lem:opt-mis}.

For a possibly arbitrary $\tilde Q_\rf$, we prove by backward induction
that $V_h^k\ge V_h^\star-H_h\beta/2$. The terminal case again holds.
Assume $V_{h+1}^k\ge V_{h+1}^\star-(H-h)\beta/2$.
Equation~\eqref{eq:detail-confidence-sandwich} gives
\[
 \hat Q_h^k\ge r_h+P_hV_{h+1}^k
 \ge Q_h^\star-(H-h)\beta/2.
\]
On an accepted pair, the actual inclusion test implies
$\tilde Q_{\rf;h}\ge\hat Q_h^k-\beta/2$. Thus
\begin{align*}
 Q_h^k
 &=\min\{\hat Q_h^k,
                   \max\{\check Q_h^k,\tilde Q_{\rf;h}\}\}\\
 &\ge\min\{\hat Q_h^k,\hat Q_h^k-\beta/2\}
 =\hat Q_h^k-\beta/2\\
 &\ge Q_h^\star-(H-h+1)\beta/2.
\end{align*}
On a rejected pair, $Q_h^k=\hat Q_h^k$ satisfies an even stronger bound.
Maximizing over actions proves the value-function inequality and completes
the induction. This last step does not use $\tau$, exact certification,
or the coverage assumption.
\end{proof}

\begin{proof}[Proof of Lemma~\ref{lem:bounded-vh-mis}]
Work first on $\mathcal E$ and suppose $\tau\le\beta/8$.
Consider a visited pair at which the ideal reference is exact,
$Q_{\rf;h}=Q_h^\star$, and the interval is narrow, $w_h^k\le\beta/4$.
We show separately that both endpoints satisfy the actual acceptance test
centered at $\tilde Q_{\rf;h}$.
Suppress the visited pair's arguments. For the lower endpoint,
\begin{align}
 \check Q_h^k
 &=\hat Q_h^k-w_h^k\notag\\
 &\overset{(a)}{\ge}\hat Q_h^k-\beta/4\notag\\
 &\overset{(b)}{\ge}Q_h^\star-\tau-\beta/4\notag\\
 &\overset{(c)}{\ge}\tilde Q_{\rf;h}-2\tau-\beta/4\notag\\
 &\overset{(d)}{\ge}\tilde Q_{\rf;h}-\beta/2.
 \label{eq:detail-mis-narrow-lower}
\end{align}
Here (a) is the narrow-width condition; (b) is
Lemma~\ref{lem:opt-mis}; (c) follows from
$Q_{\rf;h}=Q_h^\star$ and Definition~\ref{def:qref-mis}; and (d) uses
$2\tau\le\beta/4$.
For the upper endpoint,
\begin{align}
 \hat Q_h^k
 &=\check Q_h^k+w_h^k\notag\\
 &\overset{(a)}{\le}\check Q_h^k+\beta/4\notag\\
 &\overset{(b)}{\le}Q_h^\star+\beta/4\notag\\
 &\overset{(c)}{\le}\tilde Q_{\rf;h}+\tau+\beta/4\notag\\
 &\overset{(d)}{\le}\tilde Q_{\rf;h}+\beta/2.
 \label{eq:detail-mis-narrow-upper}
\end{align}
The justifications are the width condition, pessimism in
Lemma~\ref{lem:opt-mis}, Definition~\ref{def:qref-mis}, and
$\tau\le\beta/8\le\beta/4$, respectively.
Thus the actual test accepts whenever the ideal reference is exact and
the interval is narrow. Consequently, on $\mathcal E$,
\begin{equation}
 \ind\{k\in\widetilde{\mathcal V}_h\}
 \le Y_{k,h}:=\ind\{Q_{\rf;h}(s_h^k,a_h^k)
                              \ne Q_h^\star(s_h^k,a_h^k)\}.
 \label{eq:detail-mis-error-domination}
\end{equation}

We now apply concentration to $Y_{k,h}$, not to an indicator conditioned
on $\mathcal E$. Even though $\pi^k$ is selected using the observed
reference, Assumption~\ref{asm:bounded-coverage} applies to every policy.
Therefore
\[
 \mathbb E[Y_{k,h}\mid\mathcal H_{k,0}]
 =\mathbb E_{(s,a)\sim d_h^{\pi^k}}
       [\ind\{Q_{\rf;h}(s,a)\ne Q_h^\star(s,a)\}]
 \le\rho.
\]
By the second inequality of Lemma~\ref{lem:nonnegative-predictable},
with $B=1$ and failure probability $\delta/H$,
\begin{align*}
 \sum_{k=1}^K Y_{k,h}
 &\le2\sum_{k=1}^K\mathbb E[Y_{k,h}\mid\mathcal H_{k,0}]
                                            +\log(H/\delta)\\
 &\le2\rho K+\log(H/\delta).
\end{align*}
A union bound over stages costs at most $\delta$. Intersecting this event
with $\mathcal E$ and using \eqref{eq:detail-mis-error-domination}
proves the lemma. This is the same true-error counting mechanism as in
Lemma~\ref{lem:bounded-vh}, with the endpoint calculations above supplying
the required misspecified-reference implication.
\end{proof}

\begin{proof}[Proof of Lemma~\ref{lem:telescope-mis}]
Fix an episode $k$, and define
$e_h^k:=Q_h^k-r_h-P_hV_{h+1}^k$.
We first derive residual bounds valid for any observed reference.

On $\mathcal E$, projection gives $Q_h^k\le\hat Q_h^k$, and
\eqref{eq:detail-confidence-sandwich} gives
$\hat Q_h^k\le r_h+P_hV_{h+1}^k+2b_h^k$. Thus
\[
 e_h^k\le2b_h^k.
\]
Since $Q_h^k\le H$ and $r_h+P_hV_{h+1}^k\ge0$, we also have
$e_h^k\le H$, yielding $e_h^k\le\min\{2b_h^k,H\}\le c_h^k$.
For the width bound, \eqref{eq:detail-confidence-sandwich} implies
\[
 \check Q_h^k\le r_h+P_h\check V_{h+1}^k
                        \le r_h+P_hV_{h+1}^k,
\]
where the second comparison uses $\check V_{h+1}^k\le V_{h+1}^k$ from
projection. Consequently,
\begin{align}
 e_h^k
 &=Q_h^k-(r_h+P_hV_{h+1}^k)
 \le Q_h^k-\check Q_h^k
 \le\hat Q_h^k-\check Q_h^k=w_h^k.
 \label{eq:detail-mis-residual-width}
\end{align}
We have established, without a restriction on $\tau$,
\begin{equation}
 e_h^k\le c_h^k,\qquad e_h^k\le w_h^k.
 \label{eq:detail-mis-generic-residual}
\end{equation}

When $\tau\le\beta/8$, certification gives a sharper residual bound.
On a pair accepted by the reference test,
Lemma~\ref{lem:inclusion-opt-mis} gives $Q_h^k\le Q_h^\star+\tau$,
and Lemma~\ref{lem:opt-mis} gives
$V_{h+1}^k\ge V_{h+1}^\star-\tau$. Therefore
\begin{align}
 e_h^k(s,a)
 &\le Q_h^\star(s,a)+\tau-r_h(s,a)
                       -[P_hV_{h+1}^k](s,a)\notag\\
 &\le Q_h^\star(s,a)+\tau-r_h(s,a)
                       -[P_hV_{h+1}^\star](s,a)+\tau\notag\\
 &=2\tau.
 \label{eq:detail-mis-accepted-residual}
\end{align}
The final equality uses the Bellman optimality equation. On a rejected
pair, \eqref{eq:detail-mis-generic-residual} gives $e_h^k\le c_h^k$.
Combining the two branches produces the pointwise bound
\begin{equation}
 e_h^k(s,a)\le\widetilde I_h^k(s,a)c_h^k(s,a)+2\tau.
 \label{eq:detail-mis-small-residual}
\end{equation}

The conditional Bellman identity
\eqref{eq:detail-conditional-bellman-identity}, proved in the proof of
Lemma~\ref{lem:telescope}, does not require optimism and hence still holds.
Lemma~\ref{lem:opt-mis} further gives
$V_1^\star(s_1)\le V_1^k(s_1)+\tau$ in the small-error regime. Thus
\begin{align}
 V_1^\star(s_1)-V_1^{\pi^k}(s_1)
 &\le\tau+V_1^k(s_1)-V_1^{\pi^k}(s_1)\notag\\
 &=\tau+\mathbb E\!\left[
       \sum_{h=1}^H e_h^k(s_h^k,a_h^k)
                               \,\middle|\,\mathcal H_{k,0}\right]\notag\\
 &\le(2H+1)\tau+\mathbb E\!\left[
       \sum_{h=1}^H\widetilde I_h^k(s_h^k,a_h^k)c_h^k(s_h^k,a_h^k)
                               \,\middle|\,\mathcal H_{k,0}\right].
 \label{eq:detail-mis-small-episode}
\end{align}
Summing over $k$ proves the first assertion.

For an arbitrary observed reference, the generic part of
Lemma~\ref{lem:opt-mis} yields
$V_1^\star(s_1)\le V_1^k(s_1)+H\beta/2$.
If $w_h^k\le\beta/4$, \eqref{eq:detail-mis-residual-width} gives
$e_h^k\le\beta/4$. If $w_h^k>\beta/4$, use $e_h^k\le c_h^k$ instead.
These two cases imply
\[
 e_h^k(s,a)
 \le\ind\{w_h^k(s,a)>\beta/4\}c_h^k(s,a)+\beta/4.
\]
Along the realized trajectory the indicator equals
$\ind\{k\in\widetilde{\mathcal M}_h\}$, so conditional Bellman telescoping
now gives
\begin{align*}
 V_1^\star(s_1)-V_1^{\pi^k}(s_1)
 &\le\frac{H\beta}{2}
     +\mathbb E\!\left[\sum_{h=1}^He_h^k(s_h^k,a_h^k)
                                  \,\middle|\,\mathcal H_{k,0}\right]\\
 &\le\frac{H\beta}{2}+\frac{H\beta}{4}
     +\mathbb E\!\left[\sum_{h=1}^H
        \ind\{k\in\widetilde{\mathcal M}_h\}c_h^k(s_h^k,a_h^k)
                                  \,\middle|\,\mathcal H_{k,0}\right]\\
 &=\frac34H\beta
     +\mathbb E\!\left[\sum_{h=1}^H
        \ind\{k\in\widetilde{\mathcal M}_h\}c_h^k(s_h^k,a_h^k)
                                  \,\middle|\,\mathcal H_{k,0}\right].
\end{align*}
Alternatively, using only $e_h^k\le c_h^k$ gives
\[
 V_1^\star(s_1)-V_1^{\pi^k}(s_1)
 \le\frac12H\beta+
 \mathbb E\!\left[\sum_{h=1}^H c_h^k(s_h^k,a_h^k)
                              \,\middle|\,\mathcal H_{k,0}\right].
\]
Summing these two inequalities separately over episodes proves both
remaining assertions. The nonnegative costs inside all decompositions
are defined under the ordinary interaction law, without conditioning on
the future event $\mathcal E$.
\end{proof}

\begin{proof}[Proof of Lemmas~\ref{lem:mh-upper-mis} and~\ref{lem:up-t1-mis}]
We give the details to identify why the wide-interval analysis remains
valid even without exact optimism. For any observed reference,
Algorithm~\ref{alg:main-gfa} still guarantees
$\check Q_h^k\le Q_h^k\le\hat Q_h^k$ by projection. The confidence
sandwich \eqref{eq:detail-confidence-sandwich} therefore applies on
$\mathcal E$, as do the width recursion and predictable-suffix bound of
Lemma~\ref{lem:telescope-gap}.

Fix $h$ and put $n_h:=|\widetilde{\mathcal M}_h|$.
Its membership indicator is $\mathcal H_{k,h}$-measurable. Apply
Lemma~\ref{lem:telescope-gap} with failure probability $\delta/H$ and
then Lemma~\ref{lem:sum-bonus-subset} at every suffix stage. On an
additional event with failure probability at most $\delta$ across all
stages,
\begin{align}
 \frac\beta4n_h
 &\le\sum_{k\in\widetilde{\mathcal M}_h}w_h^k(s_h^k,a_h^k)\notag\\
 &\le2\sum_{j=h}^H\sum_{k\in\widetilde{\mathcal M}_h}
                           c_j^k(s_j^k,a_j^k)+2H^2\ell\notag\\
 &\le8\sqrt C\,\alpha\sqrt{n_h}\sum_{j=h}^H\sqrt{d_j}
       +16\alpha\kappa(H-h+1)n_h+2H^2\ell\notag\\
 &\le8\sqrt C\,\alpha H\sqrt{n_hd_{\mathcal F}}+16\alpha+2H^2\ell.
 \label{eq:detail-mis-wide-quadratic}
\end{align}
The first inequality is the definition of
$\widetilde{\mathcal M}_h$; the next two use the cited lemmas; and the
last uses Cauchy--Schwarz, $n_h\le K$, and $\kappa HK=1$.
Set $A=8\sqrt C\,\alpha H\sqrt{d_{\mathcal F}}$ and
$B_0=16\alpha+2H^2\ell$. As in
\eqref{eq:detail-wide-explicit}, Young's inequality gives
\begin{align*}
 \frac\beta4 n_h
 &\le A\sqrt{n_h}+B_0
 \le\frac\beta8 n_h+\frac{2A^2}{\beta}+B_0,\\
 n_h
 &\le\frac{16A^2}{\beta^2}+\frac{8B_0}{\beta}
 \le\frac{C_1H^2\alpha^2d_{\mathcal F}}{\beta^2}.
\end{align*}
The last step uses the same explicitly verified remainder absorption as
in the proof of Lemma~\ref{lem:mh-upper}. Together with $n_h\le K$,
this proves Lemma~\ref{lem:mh-upper-mis} with the same absolute constant
chosen large enough.

For the cost assertion, apply Lemma~\ref{lem:sum-bonus-subset} again:
\begin{align*}
 \sum_{h=1}^H\sum_{k\in\widetilde{\mathcal M}_h}c_h^k(s_h^k,a_h^k)
 &\le4\sqrt C\,\alpha\sum_{h=1}^H\sqrt{n_hd_h}
                     +8\alpha\kappa\sum_{h=1}^H n_h\\
 &\le4\sqrt C\,\alpha
   \min\!\left\{\sqrt K,
           \frac{\sqrt{C_1}H\alpha\sqrt{d_{\mathcal F}}}{\beta}\right\}
                    \sum_{h=1}^H\sqrt{d_h}+8\alpha\\
 &\lesssim\min\!\left\{
        \alpha H\sqrt{Kd_{\mathcal F}},
        \frac{H^2\alpha^2d_{\mathcal F}}{\beta}\right\}+\alpha.
\end{align*}
The second line uses the two uniform bounds on $n_h$ and $\kappa HK=1$;
the last uses $\sum_h\sqrt{d_h}\le H\sqrt{d_{\mathcal F}}$.
This proves Lemma~\ref{lem:up-t1-mis}. No step used the accuracy of
$\tilde Q_\rf$, the value of $\tau$, or the coverage coefficient $\rho$.
\end{proof}

\begin{proof}[Proof of Lemma~\ref{lem:up-t3-mis}]
For $\tau\le\beta/8$, Lemma~\ref{lem:sum-bonus-subset}, followed by
Lemma~\ref{lem:bounded-vh-mis}, gives
\begin{align*}
 &\sum_{h=1}^H\sum_{k\in\widetilde{\mathcal V}_h}c_h^k(s_h^k,a_h^k)\\
 &\quad\le4\sqrt C\,\alpha
       \sum_{h=1}^H\sqrt{|\widetilde{\mathcal V}_h|d_h}
       +8\alpha\kappa\sum_{h=1}^H|\widetilde{\mathcal V}_h|\\
 &\quad\le4\sqrt C\,\alpha\sqrt{2\rho K+\log(H/\delta)}
          \sum_{h=1}^H\sqrt{d_h}+8\alpha\kappa HK\\
 &\quad\le4\sqrt C\,\alpha H
          \sqrt{d_{\mathcal F}(2\rho K+\log(H/\delta))}+8\alpha\\
 &\quad\lesssim\alpha H\sqrt{d_{\mathcal F}(\rho K+\ell)}+\alpha.
\end{align*}
The third displayed line also uses
$|\widetilde{\mathcal V}_h|\le K$ for the rounding term. The next uses
Cauchy--Schwarz and $\kappa HK=1$; the last uses
$\log(H/\delta)\le\ell$.

For the second assertion, no small-error condition is needed.
Equation~\eqref{eq:detail-mis-residual-width}, established in the proof of
Lemma~\ref{lem:telescope-mis}, gives
\begin{align*}
 Q_h^k(s,a)-r_h(s,a)-[P_hV_{h+1}^k](s,a)
 &\le Q_h^k(s,a)-\check Q_h^k(s,a)\\
 &\le\hat Q_h^k(s,a)-\check Q_h^k(s,a)
 =w_h^k(s,a).
\end{align*}
Whenever the interval is narrow, its final width is at most $\beta/4$.
This proves the claimed residual bound regardless of which branch of
the reference test is selected.
\end{proof}

\section{Proof of Theorem~\ref{thm:lb}}
\label{app:proof-lower}
\subsection{Construction}
\label{sec:construction}
We use the time-inhomogeneous two-state absorbing construction already
described in the lower-bound argument, following the construction strategy
in \citet{zhou2021nearly}. Let $d\ge3$, $m=d-2$, and
$\mathcal A=\{-1,+1\}^m$. The states are a transient state $x$ and a good
absorbing state $x_G$, and every episode starts at $x$. The rewards are
\[
 r_h(x,a)=0,\qquad r_h(x_G,a)=1.
\]
For $0<\epsilon\le\zeta H/384$, define
\begin{equation}
 p_0=\frac1{2H},\qquad
 \Delta=\frac1{mH^2},\qquad
 \gamma=128\epsilon,\qquad
 u=\gamma\Delta=\frac{128\epsilon}{mH^2}.
 \label{eq:detail-lower-parameters}
\end{equation}
For every stage, select a sign vector $\theta_h\in\{-1,+1\}^m$ and let
$\bmu_h=\Delta\theta_h$. The transition probabilities from $x$ are
\begin{align*}
 P_h(x_G\mid x,a)
 &=p_0+\gamma\langle\bmu_h,a\rangle
 =p_0+u\langle\theta_h,a\rangle,\\
 P_h(x\mid x,a)&=1-p_0-u\langle\theta_h,a\rangle,
\end{align*}
while $P_h(x_G\mid x_G,a)=1$.

We first check that these are valid probabilities. Since
$|\langle\theta_h,a\rangle|\le m$, the accuracy restriction and
$\zeta<1/2$ give
\begin{align}
 mu
 =\frac{128\epsilon}{H^2}
 \le\frac{128\zeta H}{384H^2}
 =\frac{\zeta}{3H}
 <\frac1{6H}.
 \label{eq:detail-lower-perturbation}
\end{align}
Thus, uniformly over stages and actions,
\begin{equation}
 \frac1{3H}=p_0-\frac1{6H}
 \le P_h(x_G\mid x,a)
 \le p_0+\frac1{6H}=\frac2{3H}.
 \label{eq:detail-lower-probability-range}
\end{equation}
The remaining probability is assigned to $x$, so the transition kernels
are well-defined.

Next, we verify the $d$-dimensional linear representation. Let
\[
 \bphi(x,a)=\frac{(1,a^\top,0)^\top}{\sqrt{m+1}},\qquad
 \bphi(x_G,a)=e_d,
\]
and define the vector-valued transition measures by
\begin{align*}
 \boldsymbol\nu_h(x)
 &=\sqrt{m+1}\,(1-p_0,-u\theta_h^\top,0)^\top,\\
 \boldsymbol\nu_h(x_G)
 &=(\sqrt{m+1}\,p_0,\sqrt{m+1}\,u\theta_h^\top,1)^\top.
\end{align*}
Their dot products recover all four transition cases:
\begin{align*}
 \langle\bphi(x,a),\boldsymbol\nu_h(x_G)\rangle
 &=p_0+u\langle\theta_h,a\rangle,\qquad
 \langle\bphi(x,a),\boldsymbol\nu_h(x)\rangle
 =1-p_0-u\langle\theta_h,a\rangle,\\
 \langle\bphi(x_G,a),\boldsymbol\nu_h(x_G)\rangle
 &=1,\qquad
 \langle\bphi(x_G,a),\boldsymbol\nu_h(x)\rangle=0.
\end{align*}
The reward parameter is $e_d$, since
$\langle\bphi(x,a),e_d\rangle=0$ and
$\langle\bphi(x_G,a),e_d\rangle=1$.
Furthermore,
\begin{align*}
 \|\bphi(x,a)\|_2^2
 &=\frac{1+\|a\|_2^2}{m+1}=1,
 \qquad \|\bphi(x_G,a)\|_2=\|e_d\|_2=1,\\
 \boldsymbol\nu_h(\mathcal S)
 &=(\sqrt{m+1},0,\ldots,0,1)^\top,
 \qquad\|\boldsymbol\nu_h(\mathcal S)\|_2=\sqrt{m+2}=\sqrt d.
\end{align*}
This verifies the stated feature and parameter normalization.

Finally, let $M_\rf$ have the same rewards and features but no perturbation,
that is, $u=0$, and supply the common reference
$Q_\rf=Q^{\star,M_\rf}$ to the learner.
For each fixed accuracy $\epsilon$, take the finite family indexed by
\[
 \Theta:=\bigl(\{-1,+1\}^m\bigr)^H,
 \qquad |\Theta|=2^{mH},
\]
with the uniform prior on the sign arrays
$\theta=(\theta_1,\ldots,\theta_H)$.
The corresponding family is denoted by $\mathcal M_\zeta$, suppressing its
dependence on $\epsilon$. The supplied reference is the same for every
$\theta$, so it cannot reveal the unknown signs. Lemma~\ref{lem:upper-gamma}
verifies that all these models satisfy the required reference-closeness
condition.

\subsection{Supporting Lemmas}
\begin{lemma}[Reference closeness]\label{lem:upper-gamma}
Every MDP in this construction satisfies
\[
 \max_{h,s,a}|Q_{\rf;h}(s,a)-Q_h^\star(s,a)|\le64\epsilon\le\zeta H.
\]
\end{lemma}
\begin{proof}
Let $V_{\rf,h}:=V_h^{\star,M_\rf}$, let $P_{\rf,h}$ be the reference
transition kernel, and fix a model $M$ in the constructed family.
Define
\[
 \delta_h:=\max_{s,a}|Q_{\rf;h}(s,a)-Q_h^\star(s,a;M)|,
 \qquad \delta_{H+1}:=0.
\]
We suppress $M$ from the notation for its optimal values.
The rewards of the reference and target models are identical. Adding
and subtracting $P_{\rf,h}V_{h+1}^\star$ in their Bellman equations gives
\begin{align}
 &|Q_{\rf;h}(s,a)-Q_h^\star(s,a)|\notag\\
 &\quad=|[P_{\rf,h}V_{\rf,h+1}](s,a)
                           -[P_hV_{h+1}^\star](s,a)|\notag\\
 &\quad\le|[P_{\rf,h}(V_{\rf,h+1}-V_{h+1}^\star)](s,a)|
            +|[(P_{\rf,h}-P_h)V_{h+1}^\star](s,a)|.
 \label{eq:detail-reference-bellman-split}
\end{align}
For the first term, the transition probabilities sum to one, and the
maximum over actions is $1$-Lipschitz in the sup norm. Thus
\begin{align}
 |[P_{\rf,h}(V_{\rf,h+1}-V_{h+1}^\star)](s,a)|
 &\le\max_{s'}|V_{\rf,h+1}(s')-V_{h+1}^\star(s')|\notag\\
 &\le\max_{s',a'}|Q_{\rf;h+1}(s',a')-Q_{h+1}^\star(s',a')|
 =\delta_{h+1}.
 \label{eq:detail-reference-recursive-term}
\end{align}
For the second term, the kernels coincide at $x_G$. At $x$, the
construction gives the exact identity
\begin{align*}
 [(P_{\rf,h}-P_h)V_{h+1}^\star](x,a)
 =u\langle\theta_h,a\rangle
       \bigl(V_{h+1}^\star(x)-V_{h+1}^\star(x_G)\bigr).
\end{align*}
Since $|\langle\theta_h,a\rangle|\le m$ and both continuation values are
in $[0,H-h]$, this implies
\begin{equation}
 |[(P_{\rf,h}-P_h)V_{h+1}^\star](s,a)|\le mu(H-h)
 \qquad\text{for either state }s.
 \label{eq:detail-reference-transition-term}
\end{equation}
Combining \eqref{eq:detail-reference-bellman-split}--
\eqref{eq:detail-reference-transition-term} and maximizing over $(s,a)$,
we obtain
\[
 \delta_h\le\delta_{h+1}+mu(H-h).
\]
Unrolling this recursion through the terminal condition gives
\begin{align*}
 \delta_h
 &\le mu\sum_{j=h}^H(H-j)
   =mu\,\frac{(H-h)(H-h+1)}2\\
 &\le mu\,\frac{H(H-1)}2
   =\frac{128\epsilon}{H^2}\frac{H(H-1)}2\\
 &=64\epsilon\frac{H-1}{H}
 \le64\epsilon
 \le\frac{\zeta H}{6}
 \le\zeta H.
\end{align*}
The first equality sums the arithmetic progression; the second line uses
\eqref{eq:detail-lower-parameters}; and the penultimate comparison uses
$\epsilon\le\zeta H/384$. This proves the uniform closeness bound for
every sign array in the family.
\end{proof}

\begin{lemma}[Decomposition]\label{lem:lb-decompose}
Let $L=\lfloor H/2\rfloor$. For each $h\le L$, the probability of reaching $x$ at stage $h$ is at least $1/3$ under every policy, and
\[
 V_{h+1}^\star(x_G)-V_{h+1}^\star(x)\ge H/6.
\]
If $N_{h,j}(\theta)$ counts episodes in which the learner visits $x$ at stage $h$ and takes an action whose $j$th sign differs from $\theta_{h,j}$, then
\[
 \mathbb E_\theta[\mathrm{Regret}(K)]
 \ge\frac{Hu}{3}\sum_{h=1}^{L}\sum_{j=1}^m
                         \mathbb E_\theta[N_{h,j}(\theta)].
\]
\end{lemma}
\begin{proof}
Fix a sign array $\theta$. We derive a lower bound on visits to $x$, a
lower bound on the value gained by entering $x_G$, and then the resulting
regret decomposition.

By \eqref{eq:detail-lower-probability-range}, under any chosen action,
\[
 P_j(x\mid x,a)\ge1-\frac2{3H}.
\]
This bound also holds after conditioning on any previous interaction
history. Starting from $x$ and applying the tower property at each stage,
\begin{align}
 \Pr_\theta(s_h=x)
 &\ge\left(1-\frac2{3H}\right)^{h-1}\notag\\
 &\ge\left(1-\frac2{3H}\right)^H
 \ge1-H\frac2{3H}=\frac13.
 \label{eq:detail-lower-reach}
\end{align}
The second inequality uses $h-1\le H$ and a base in $(0,1)$; the last is
Bernoulli's inequality $(1-t)^H\ge1-Ht$. In particular, the result holds
conditionally on the training history preceding each episode, even when
the learner is randomized or history-dependent.

At stage $h$, the action $a=\theta_h$ maximizes the probability of moving
to $x_G$, since $\max_a\langle\theta_h,a\rangle=m$.
It is optimal because entering $x_G$ cannot decrease future rewards.
The optimal transition probability is the same at every stage:
$q:=p_0+mu\le2/(3H)$.
Let $g_h:=V_h^\star(x_G)-V_h^\star(x)$ and $g_{H+1}=0$.
The Bellman equations give
\begin{align*}
 g_h
 &=\bigl(1+V_{h+1}^\star(x_G)\bigr)
     -\bigl(qV_{h+1}^\star(x_G)+(1-q)V_{h+1}^\star(x)\bigr)\\
 &=1+(1-q)g_{h+1}.
\end{align*}
Backward substitution yields
\begin{equation}
 g_{h+1}=1+(1-q)+\cdots+(1-q)^{H-h-1}
         =\sum_{t=0}^{H-h-1}(1-q)^t.
 \label{eq:detail-good-state-gap}
\end{equation}
For $h\le L:=\lfloor H/2\rfloor$, the number of summands is at least
$H/2$. Hence
\begin{align}
 g_{h+1}
 &\ge(H-h)(1-q)^{H-h-1}\notag\\
 &\ge(H-h)\left(1-\frac2{3H}\right)^H\notag\\
 &\ge\frac{H-h}{3}\ge\frac H6.
 \label{eq:detail-good-state-gap-lower}
\end{align}
The first inequality bounds every summand by the smallest one; the
second uses $q\le2/(3H)$; the third uses the same Bernoulli inequality
as \eqref{eq:detail-lower-reach}; and the last uses $h\le H/2$.

Write $p_h(a)=p_0+u\langle\theta_h,a\rangle$.
The optimal and action-conditioned Bellman equations at $x$ give
\begin{align}
 V_h^\star(x)-Q_h^\star(x,a)
 &=\bigl(q-p_h(a)\bigr)
                    \bigl(V_{h+1}^\star(x_G)-V_{h+1}^\star(x)\bigr)\notag\\
 &=u\bigl(m-\langle\theta_h,a\rangle\bigr)g_{h+1}\notag\\
 &=2u\sum_{j=1}^m\ind\{a_j\ne\theta_{h,j}\}\,g_{h+1}\notag\\
 &\ge\frac{Hu}{3}\sum_{j=1}^m\ind\{a_j\ne\theta_{h,j}\},
 \qquad h\le L.
 \label{eq:detail-sign-disadvantage}
\end{align}
For the third line, each coordinate satisfies
$1-\theta_{h,j}a_j=2\ind\{a_j\ne\theta_{h,j}\}$.
The last line uses \eqref{eq:detail-good-state-gap-lower}.
At $x_G$, every action has zero disadvantage.

To spell out the performance-difference identity, define
$A_h^\star(s,a):=V_h^\star(s)-Q_h^\star(s,a)\ge0$.
For any learner and any episode, conditional on its initial history,
\begin{align*}
 \mathbb E_\theta\!\left[\sum_{h=1}^H A_h^\star(s_h,a_h)\right]
 &=\mathbb E_\theta\!\left[\sum_{h=1}^H
       \bigl(V_h^\star(s_h)-r_h(s_h,a_h)
                            -[P_hV_{h+1}^\star](s_h,a_h)\bigr)\right]\\
 &=\mathbb E_\theta\!\left[\sum_{h=1}^H
       \bigl(V_h^\star(s_h)-r_h(s_h,a_h)
                                      -V_{h+1}^\star(s_{h+1})\bigr)\right]\\
 &=V_1^\star(x)-\mathbb E_\theta\!\left[\sum_{h=1}^H r_h(s_h,a_h)\right].
\end{align*}
The second equality is the tower property for the transition kernel;
the third telescopes and uses $V_{H+1}^\star=0$.
Consequently, after summing over episodes and taking expectation over
training randomness,
\begin{align*}
 \mathbb E_\theta[\mathrm{Regret}(K)]
 &=\mathbb E_\theta\!\left[\sum_{k=1}^K\sum_{h=1}^H
                         A_h^\star(s_h^k,a_h^k)\right]\\
 &\ge\frac{Hu}{3}\sum_{h=1}^L\sum_{j=1}^m
     \mathbb E_\theta\!\left[\sum_{k=1}^K
       \ind\{s_h^k=x,\ a_{h,j}^k\ne\theta_{h,j}\}\right]\\
 &=\frac{Hu}{3}\sum_{h=1}^L\sum_{j=1}^m
                        \mathbb E_\theta[N_{h,j}(\theta)].
\end{align*}
Here we discard only nonnegative disadvantages at later stages and use
\eqref{eq:detail-sign-disadvantage} at the earlier stages. The count in
brackets is precisely $N_{h,j}(\theta)$ in the lemma. This proves the
claimed decomposition, including for history-dependent learners.
\end{proof}

\subsection{Detailed Proof}
\begin{proof}[Proof of Theorem~\ref{thm:lb}]
Let $\Theta=(\{-1,+1\}^m)^H$ be the sign-array family in
Section~\ref{sec:construction}. For fixed $\theta\in\Theta$, denote by
$\mathbb P_\theta$ the law of the complete $K$-episode interaction history,
including states, actions, rewards, and learner randomization. We show
that neighboring sign arrays are statistically difficult to distinguish,
and then use Lemma~\ref{lem:lb-decompose} to convert sign errors into regret.

Fix a coordinate $(h,j)$ and let $\theta'$ differ from $\theta$ only in
that coordinate. The common reference $Q_\rf$ is identical in the two
models. Conditional on a realized history, the learner's action rule is
also the same in the two models. The rewards coincide, and the transition
kernels differ only at stage $h$ when the state is $x$. By the chain rule
of relative entropy, all other conditional contributions vanish, leaving
\begin{align}
 \mathrm{KL}(\mathbb P_\theta\,\|\,\mathbb P_{\theta'})
 &=\sum_{k=1}^K
   \mathbb E_\theta\!\left[
    \ind\{s_h^k=x\}
    \mathrm{kl}\bigl(p_h^\theta(a_h^k),p_h^{\theta'}(a_h^k)\bigr)
                         \right],
 \label{eq:detail-history-kl-chain}
\end{align}
where $\mathrm{kl}$ denotes Bernoulli relative entropy and
$p_h^\theta(a)=p_0+u\langle\theta_h,a\rangle$.
The changed coordinate gives
\[
 |p_h^\theta(a)-p_h^{\theta'}(a)|
 =|2u\theta_{h,j}a_j|=2u.
\]
For $p,q\in(0,1)$, the elementary inequality $\log t\le t-1$ implies
\begin{align}
 \mathrm{KL}(p,q)
 &=p\log\frac pq+(1-p)\log\frac{1-p}{1-q}\notag\\
 &\le p\left(\frac pq-1\right)
       +(1-p)\left(\frac{1-p}{1-q}-1\right)\notag\\
 &=\frac{(p-q)^2}{q(1-q)}.
 \label{eq:detail-bernoulli-kl}
\end{align}
By \eqref{eq:detail-lower-probability-range}, both parameters lie in
$[1/(3H),2/(3H)]$. As $H\ge3$,
\[
 q\ge\frac1{3H},\qquad 1-q\ge1-\frac2{3H}\ge\frac34,
 \qquad q(1-q)\ge\frac1{4H}.
\]
Substituting into \eqref{eq:detail-bernoulli-kl} and then
\eqref{eq:detail-history-kl-chain} yields
\begin{align}
 \mathrm{KL}\bigl(p_h^\theta(a),p_h^{\theta'}(a)\bigr)
 &\le\frac{(2u)^2}{1/(4H)}=16Hu^2,\notag\\
 \mathrm{KL}(\mathbb P_\theta\,\|\,\mathbb P_{\theta'})
 &\le16Hu^2\sum_{k=1}^K\mathbb E_\theta[\ind\{s_h^k=x\}]
 \le16HKu^2.
 \label{eq:detail-full-history-kl}
\end{align}
This bound accounts for all online observations, not only reward
observations or a single stage viewed in isolation.

Suppose
\begin{equation}
 K\le\frac1{288Hu^2}
   =\frac{m^2H^3}{288\cdot128^2\epsilon^2}.
 \label{eq:detail-lower-budget}
\end{equation}
Pinsker's inequality~\citep{lattimore2020bandit} and
\eqref{eq:detail-full-history-kl} imply
\begin{align}
 \mathrm{TV}(\mathbb P_\theta,\mathbb P_{\theta'})
 &\le\sqrt{\frac12
                \mathrm{KL}(\mathbb P_\theta\,\|\,\mathbb P_{\theta'})}
 \le\sqrt{8HKu^2}\notag\\
 &\le\sqrt{\frac8{288}}=\frac16.
 \label{eq:detail-neighbor-tv}
\end{align}
The equality in \eqref{eq:detail-lower-budget} substitutes
$u=128\epsilon/(mH^2)$ from \eqref{eq:detail-lower-parameters}.

For a fixed complete history $\omega$, let
$T_h(\omega):=\sum_{k=1}^K\ind\{s_h^k=x\}$.
The arrays $\theta$ and $\theta'$ prescribe opposite signs at $(h,j)$.
At every visit to $x$, exactly one of these two prescriptions disagrees
with the observed sign $a_{h,j}^k$. Hence, as functions of the same history,
\begin{equation}
 N_{h,j}(\theta;\omega)+N_{h,j}(\theta';\omega)=T_h(\omega).
 \label{eq:detail-paired-count-identity}
\end{equation}
All three counts lie in $[0,K]$.
For any history-measurable function $F\in[0,K]$, the layer-cake identity
and the definition of total variation give
\begin{align*}
 |\mathbb E_\theta F-\mathbb E_{\theta'}F|
 &=\left|\int_0^K
     \bigl(\mathbb P_\theta(F>t)-\mathbb P_{\theta'}(F>t)\bigr)
          \,\mathrm dt\right|\\
 &\le K\,\mathrm{TV}(\mathbb P_\theta,\mathbb P_{\theta'}).
\end{align*}
Apply this to $F=N_{h,j}(\theta';\cdot)$, then use
\eqref{eq:detail-paired-count-identity}:
\begin{align}
 &\mathbb E_\theta N_{h,j}(\theta)
       +\mathbb E_{\theta'}N_{h,j}(\theta')\notag\\
 &\quad=\mathbb E_\theta\bigl[N_{h,j}(\theta)+N_{h,j}(\theta')\bigr]
       +(\mathbb E_{\theta'}-\mathbb E_\theta)N_{h,j}(\theta')\notag\\
 &\quad\ge\mathbb E_\theta T_h
               -K\,\mathrm{TV}(\mathbb P_\theta,\mathbb P_{\theta'})\notag\\
 &\quad\ge\frac K3-\frac K6=\frac K6.
 \label{eq:detail-paired-error-bound}
\end{align}
The last line uses the reach-probability conclusion of
Lemma~\ref{lem:lb-decompose} for each episode and
\eqref{eq:detail-neighbor-tv}.

The map that flips coordinate $(h,j)$ is a bijection of $\Theta$ and its
own inverse. Summing \eqref{eq:detail-paired-error-bound} over all
$\theta\in\Theta$ therefore gives
\begin{align*}
 2\sum_{\theta\in\Theta}\mathbb E_\theta N_{h,j}(\theta)
 &=\sum_{\theta\in\Theta}
    \left(\mathbb E_\theta N_{h,j}(\theta)
       +\mathbb E_{\theta'}N_{h,j}(\theta')\right)\\
 &\ge|\Theta|\frac K6.
\end{align*}
Dividing by $2|\Theta|$ yields, for every coordinate,
\begin{equation}
 \mathbb E_{\theta\sim\mathrm{Unif}(\Theta)}
         \mathbb E_\theta N_{h,j}(\theta)\ge\frac K{12}.
 \label{eq:detail-prior-error-count}
\end{equation}
Now average the regret decomposition in Lemma~\ref{lem:lb-decompose}
over the same prior:
\begin{align}
 \mathbb E_{\theta\sim\mathrm{Unif}(\Theta)}
                   \mathbb E_\theta[\mathrm{Regret}(K)]
 &\ge\frac{Hu}{3}\sum_{h=1}^L\sum_{j=1}^m
       \mathbb E_{\theta\sim\mathrm{Unif}(\Theta)}
                      \mathbb E_\theta N_{h,j}(\theta)\notag\\
 &\ge\frac{Hu}{3}\,Lm\frac K{12}
 =\frac{HuLmK}{36}\notag\\
 &\ge\frac{umH^2K}{108}
 =\frac{128}{108}\epsilon K
 >\epsilon K.
 \label{eq:detail-prior-regret-lower}
\end{align}
The second line uses \eqref{eq:detail-prior-error-count}; the third uses
$L=\lfloor H/2\rfloor\ge H/3$ for integer $H\ge3$; and the final equality
substitutes $u=128\epsilon/(mH^2)$.

Equation~\eqref{eq:detail-prior-regret-lower} contradicts expected average
regret at most $\epsilon$ whenever the budget satisfies
\eqref{eq:detail-lower-budget}. Hence the required number of episodes is
at least a constant multiple of $m^2H^3/\epsilon^2$.
Since $m=d-2\ge d/3$ for $d\ge3$, this is
$\Omega(H^3d^2/\epsilon^2)$.
Lemma~\ref{lem:upper-gamma} verifies
$\|Q_\rf-Q^{\star,M}\|_\infty\le\zeta H$ for every member of the same
family, and the construction supplies an identical reference to every
learner. Taking $c_0=1/384$ establishes all claims in
Theorem~\ref{thm:lb}.
\end{proof}

\section{Technical Lemmas}
\label{app:technical-lemmas}
\begin{lemma}[Hoeffding's Inequality]
\label{lem:hoeffding}
Let $X_1,\dots,X_n$ be independent random variables such that
$a_i \le X_i \le b_i$ almost surely for all $i\in[n]$, and define
$S_n := \sum_{i=1}^n X_i$ with $\mu := \mathbb{E}[S_n]$.
Then, with probability at least $1-\delta$, we have
\begin{align*}
S_n - \mu
\le
\sqrt{\frac{1}{2}\Big(\sum_{i=1}^n (b_i-a_i)^2\Big)\log\!\frac{1}{\delta}}.
\end{align*}
\end{lemma}

\begin{lemma}[Azuma--Hoeffding Inequality]
\label{lem:azuma_hoeffding}
Let $\{X_t\}_{t=0}^T$ be a martingale adapted to a filtration 
$\{\mathcal{F}_t\}_{t=0}^T$, and define the martingale difference
sequence $D_t = X_t - X_{t-1}$. 
Assume that the increments are almost surely bounded:
\[
|D_t| \le c_t \quad \text{for all } t=1,\dots,T.
\]
Then for any $\delta \in (0,1)$, with probability at least $1-\delta$,
\[
X_T - X_0 
\;\le\;
\sqrt{2 \log\!\left(\frac{1}{\delta}\right)
\sum_{t=1}^T c_t^2 } .
\]
\end{lemma}
% \begin{proof}
% For each martingale difference $D_t$, the martingale property implies
% $\mathbb E[D_t\mid\mathcal F_{t-1}]=0$, and its conditional support lies
% in $[-c_t,c_t]$. The conditional version of
% \eqref{eq:detail-hoeffding-mgf}, established in the proof of
% Lemma~\ref{lem:hoeffding}, therefore gives
% \[
%  \mathbb E[\exp(\lambda D_t)\mid\mathcal F_{t-1}]
%  \le\exp\!\left(\frac{\lambda^2(2c_t)^2}{8}\right)
%  =\exp(\lambda^2c_t^2/2).
% \]
% By successive applications of the tower property,
% \begin{align*}
%  \mathbb E\exp\!\left(\lambda\sum_{t=1}^T D_t\right)
%  &\le\exp\!\left(\frac{\lambda^2}{2}\sum_{t=1}^T c_t^2\right).
% \end{align*}
% Let $A:=\sum_{t=1}^T c_t^2$. If $A=0$, all increments vanish almost
% surely. For $A>0$, Markov's inequality gives, for $x>0$,
% \begin{align*}
%  \Pr(X_T-X_0\ge x)
%  &=\Pr\!\left(\sum_{t=1}^T D_t\ge x\right)\\
%  &\le\inf_{\lambda>0}\exp(-\lambda x+\lambda^2 A/2)
%  =\exp(-x^2/(2A)),
% \end{align*}
% where the infimum is attained at $\lambda=x/A$.
% Set $x=\sqrt{2A\log(1/\delta)}$ to obtain the desired inequality.
% This proof does not require independent increments; conditional centering
% and the martingale filtration suffice.
% \end{proof}

\begin{lemma}[Nonnegative predictable costs]\label{lem:nonnegative-predictable}
Let $0\le X_i\le B$ be adapted observations and let
$m_i=\mathbb E[X_i\mid\mathcal G_{i-1}]$ for their pre-observation filtration. For each $\eta\in(0,1)$, each of the following holds with probability at least $1-\eta$:
\[
 \sum_i m_i\le2\sum_iX_i+2B\log(1/\eta),\qquad
 \sum_iX_i\le2\sum_i m_i+B\log(1/\eta).
\]
The same statements apply to a predictable selection of nonnegative costs by including the selection indicator in $X_i$.
\end{lemma}
\begin{proof}
If $B=0$, every observation and conditional mean is zero, so both
inequalities hold. Assume $B>0$ and write
\[
 U_i:=X_i/B\in[0,1],\qquad
 \mu_i:=\mathbb E[U_i\mid\mathcal G_{i-1}]=m_i/B.
\]
We prove the two one-sided inequalities separately. Each has failure
probability at most $\eta$; simultaneous use of both may be obtained by
splitting a failure budget between them.

The graph of a convex function lies below the chord connecting two points
on its graph. Applying this to $u\mapsto e^{-u}$ on $[0,1]$ gives
\[
 e^{-u}\le(1-u)e^0+u e^{-1}=1-(1-e^{-1})u.
\]
Let $a:=1-e^{-1}$. Taking conditional expectations and using
$1-v\le e^{-v}$ yields
\begin{align}
 \mathbb E[e^{-U_i}\mid\mathcal G_{i-1}]
 &\le1-a\mu_i\le e^{-a\mu_i},\notag\\
 \mathbb E[e^{a\mu_i-U_i}\mid\mathcal G_{i-1}]&\le1.
 \label{eq:detail-cost-lower-mgf}
\end{align}
Since $\mu_i$ is $\mathcal G_{i-1}$-measurable, the process
\[
 Z_n^-:=\exp\!\left(a\sum_{i=1}^n\mu_i-\sum_{i=1}^nU_i\right),
 \qquad Z_0^-=1,
\]
is a nonnegative supermartingale: indeed,
\[
 \mathbb E[Z_n^-\mid\mathcal G_{n-1}]
 =Z_{n-1}^-\mathbb E[e^{a\mu_n-U_n}\mid\mathcal G_{n-1}]
 \le Z_{n-1}^-.
\]
Thus $\mathbb EZ_n^-\le1$. By Markov's inequality,
$\Pr(Z_n^->1/\eta)\le\eta$. On the complementary event,
\begin{align*}
 a\sum_{i=1}^n\mu_i-\sum_{i=1}^nU_i&\le\log(1/\eta),\\
 \sum_{i=1}^n m_i
 &\le\frac1a\sum_{i=1}^n X_i+\frac{B}{a}\log(1/\eta)\\
 &\le2\sum_{i=1}^n X_i+2B\log(1/\eta),
\end{align*}
where $a^{-1}=(1-e^{-1})^{-1}<2$. This proves the first inequality.

The chord bound for $u\mapsto e^u$ gives
$e^u\le1+(e-1)u$ for $u\in[0,1]$.
Let $b:=e-1$. Then
\begin{align*}
 \mathbb E[e^{U_i}\mid\mathcal G_{i-1}]
 &\le1+b\mu_i\le e^{b\mu_i},\\
 \mathbb E[e^{U_i-b\mu_i}\mid\mathcal G_{i-1}]&\le1.
\end{align*}
The same tower-property argument shows that
\[
 Z_n^+:=\exp\!\left(\sum_{i=1}^nU_i-b\sum_{i=1}^n\mu_i\right)
\]
is a nonnegative supermartingale with initial value one.
Except on an event of probability $\eta$, Markov's inequality yields
\begin{align*}
 \sum_{i=1}^nU_i-b\sum_{i=1}^n\mu_i&\le\log(1/\eta),\\
 \sum_{i=1}^nX_i
 &\le b\sum_{i=1}^n m_i+B\log(1/\eta)\\
 &\le2\sum_{i=1}^n m_i+B\log(1/\eta),
\end{align*}
since $b=e-1<2$. This proves the second inequality.

Let $J_i\in\{0,1\}$ be $\mathcal G_{i-1}$-measurable and define
$X_i':=J_iX_i$. Then $0\le X_i'\le B$ and
\[
 \mathbb E[X_i'\mid\mathcal G_{i-1}]
 =J_i\mathbb E[X_i\mid\mathcal G_{i-1}]=J_im_i.
\]
Applying either of the preceding arguments to $X_i'$ proves the
corresponding selected-cost inequality. Predictability is essential to
this last equality; selection after observing $X_i$ is not covered by
this assertion.
\end{proof}

\section{Verification for Example~\ref{ex:misspecified-reference}}
\label{app:misspecified-reference}

We verify Example~\ref{ex:misspecified-reference} through an explicit
construction. Let $H\geq 2$, $0<\beta\leq b\leq 1$,
$\rho\in[0,1]$, and $p:=\rho/H$.
Consider a base linear MDP with state space $\mathcal X$, finite
action space $\mathcal A$, transition kernel
\begin{align*}
    \bar P_h(\cdot\mid x,a)
    &=
    \left\langle
        \bar{\boldsymbol\phi}(x,a),
        \bar{\boldsymbol\mu}_h(\cdot)
    \right\rangle,
\end{align*}
and reward
\begin{align*}
    \bar r_h(x,a)
    &=
    \left\langle
        \bar{\boldsymbol\phi}(x,a),
        \boldsymbol\theta_h
    \right\rangle
    \in[0,1-b].
\end{align*}

\paragraph{Construction of the misspecification.}
Augment the state space to
$\mathcal S:=\mathcal X\times\{0,1\}$, where
$s=(x,c)$ and $c$ is a corruption indicator.
Let $\nu_p$ be the Bernoulli distribution with
$\nu_p(\{1\})=p$.
Define the transition kernel by
\begin{align*}
    P_h(\mathrm d x',\mathrm d c'\mid (x,c),a)
    &=
    \bar P_h(\mathrm d x'\mid x,a)\,
    \nu_p(\mathrm d c').
\end{align*}
Thus, the next corruption indicator is sampled independently of
the current state and action.

The augmented transition kernel retains the linear representation
with
\begin{align*}
    \boldsymbol\phi((x,c),a)
    &:=
    \bar{\boldsymbol\phi}(x,a),\\
    \boldsymbol\mu_h(\mathrm d x',\mathrm d c')
    &:=
    \bar{\boldsymbol\mu}_h(\mathrm d x')\,
    \nu_p(\mathrm d c').
\end{align*}
In particular,
\begin{align*}
    P_h(\cdot\mid s,a)
    &=
    \left\langle
        \boldsymbol\phi(s,a),
        \boldsymbol\mu_h(\cdot)
    \right\rangle.
\end{align*}

For each $h\in\{2,\ldots,H\}$, choose a function
$g_h:\mathcal A\to[\beta,b]$ and define
\begin{align*}
    \xi_h((x,c),a)
    &:=
    \begin{cases}
        0, & h=1,\\
        c\,g_h(a), & 2\leq h\leq H.
    \end{cases}
\end{align*}
The reward is
\begin{align*}
    r_h((x,c),a)
    &:=
    \bar r_h(x,a)+\xi_h((x,c),a)
    \in[0,1].
\end{align*}
For $h\geq2$, this residual is not representable using
$\boldsymbol\phi$: the states $(x,0)$ and $(x,1)$ have identical
features for every action, whereas their residuals differ by
$g_h(a)\geq\beta$.

\paragraph{Idealized reference and separation.}
Let $V^\star$ be the optimal value function of the augmented MDP,
with $V_{H+1}^\star\equiv0$.
For this existence construction, define the idealized linear
reference
\begin{align*}
    \boldsymbol w_h
    &:=
    \boldsymbol\theta_h
    +
    \int_{\mathcal S}
        V_{h+1}^\star(s')\,
        \mathrm d\boldsymbol\mu_h(s'),\\
    Q_{\rf;h}(s,a)
    &:=
    \left\langle
        \boldsymbol\phi(s,a),\boldsymbol w_h
    \right\rangle.
\end{align*}
This definition uses the true continuation value
$V_{h+1}^\star$; it does not assert that the reference can be
computed from the linear model components alone.

By the Bellman optimality equation,
\begin{align*}
    Q_h^\star(s,a)
    &=
    r_h(s,a)+[P_hV_{h+1}^\star](s,a)\\
    &=
    \xi_h(s,a)
    +
    \left\langle
        \boldsymbol\phi(s,a),
        \boldsymbol\theta_h
        +
        \int_{\mathcal S}
            V_{h+1}^\star(s')\,
            \mathrm d\boldsymbol\mu_h(s')
    \right\rangle\\
    &=
    \xi_h(s,a)+Q_{\rf;h}(s,a).
\end{align*}

Define the corrupted subsets by
\begin{align*}
    \mathcal S'_1&:=\varnothing,
    &
    \mathcal S'_h&:=\mathcal X\times\{1\},
    \qquad 2\leq h\leq H.
\end{align*}
The reference is exact outside
$\mathcal S'_h\times\mathcal A$, while
\begin{align*}
    \left|
        Q_h^\star(s,a)-Q_{\rf;h}(s,a)
    \right|
    =
    \xi_h(s,a)
    \geq\beta
\end{align*}
on $\mathcal S'_h\times\mathcal A$.
Hence $Q_{\rf}$ satisfies
Definition~\ref{def:beta-sep}.

\paragraph{Policy-uniform coverage.}
For every $h\in\{2,\ldots,H\}$ and every $(s,a)$,
the transition construction gives
\begin{align*}
    P_{h-1}(\mathcal S'_h\mid s,a)=p.
\end{align*}
Consequently, under any policy $\pi$ and any initial-state
distribution,
\begin{align*}
    \Pr_\pi(s_h\in\mathcal S'_h)=p,
    \qquad 2\leq h\leq H.
\end{align*}
Since $\xi_1\equiv0$, the reference is exact at the first stage.
Therefore,
\begin{align*}
    &\mathbb E_{(s,a)\sim d_h^\pi}
    \left[
        \mathbf 1
        \left\{
            Q_h^\star(s,a)\neq Q_{\rf;h}(s,a)
        \right\}
    \right]\\
    &\qquad=
    \begin{cases}
        0, & h=1,\\
        \rho/H, & 2\leq h\leq H,
    \end{cases}
    \ \leq\rho.
\end{align*}
This verifies Assumption~\ref{asm:bounded-coverage}.
In fact, the construction also satisfies the stronger bound
\begin{align*}
    \sum_{h=1}^H
    \mathbb E_{(s,a)\sim d_h^\pi}
    \left[
        \mathbf 1
        \left\{
            Q_h^\star(s,a)\neq Q_{\rf;h}(s,a)
        \right\}
    \right]
    =
    \frac{(H-1)\rho}{H}
    \leq\rho.
\end{align*}

\paragraph{Reference misspecification versus online realizability.}
The misspecification above concerns the reference feature map
$\boldsymbol\phi$, not necessarily the function class used for
online learning.
To satisfy Assumption~\ref{asm:bellman-completeness}, the online
class must also represent the reward residual.
For this construction, this can be achieved by using the richer
feature map
\begin{align*}
    \boldsymbol\psi((x,c),a)
    :=
    \left(
        \bar{\boldsymbol\phi}(x,a),
        c\,\mathbf e_a
    \right),
\end{align*}
where $\mathbf e_a$ is the canonical basis vector indexed by
$a\in\mathcal A$.
Indeed, the reward and every Bellman backup
$r_h+P_hV$ are linear in $\boldsymbol\psi$.
With sufficiently large coefficient and output bounds, the
corresponding online class therefore contains the required
Bellman backups, even though the reference uses only
$\boldsymbol\phi$.

\begin{figure}[t]
  \centering
  \includegraphics[width=\linewidth]{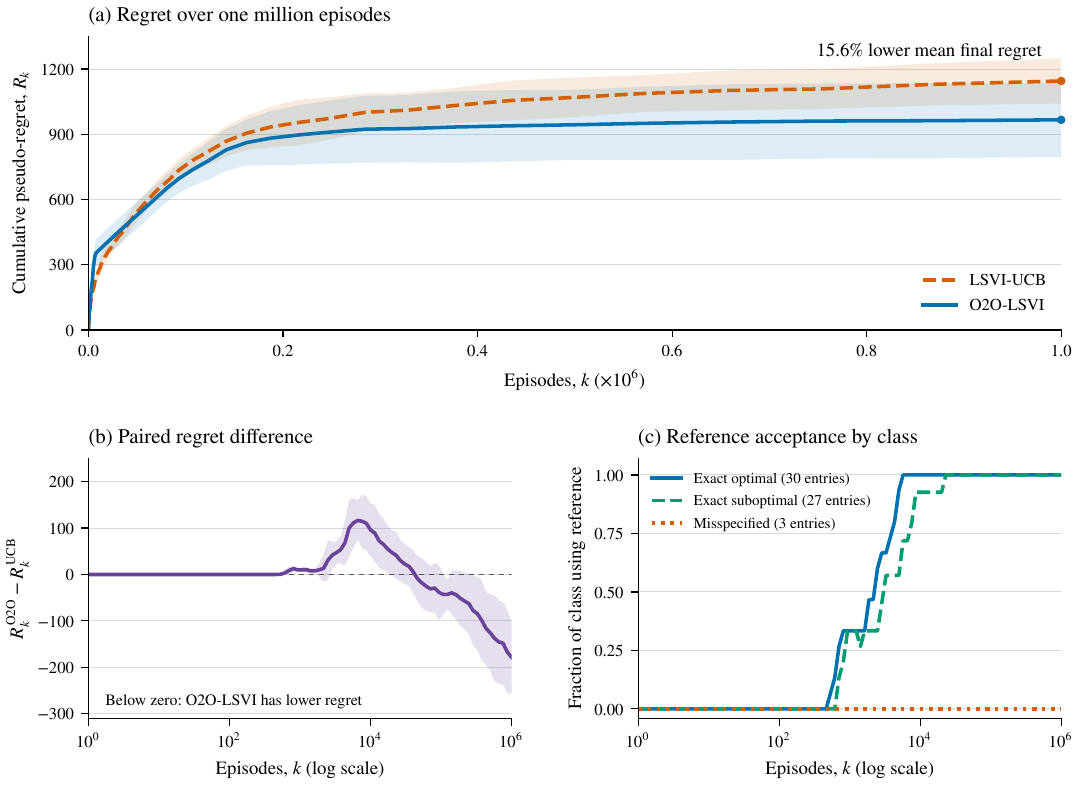}
  \caption{\textbf{Results on Linear MDP.} O2O-LSVI and LSVI-UCB with a fixed true action gap of $0.10$,
  reference tolerance $\beta_{\mathrm{ref}}=0.40$, and reference error
  $0.44$ on $3$ of $60$ entries. Both learners use $d=5$ features.
  (a) Mean cumulative pseudo-regret across five seeds, with a linear
  episode axis. (b) Mean paired regret difference; negative values favor
  O2O-LSVI. Shaded regions in (a) and (b) are pointwise $95\%$
  Student-$t$ intervals with four degrees of freedom, not simultaneous
  confidence bands. (c) Mean fraction of entries within each reference
  class satisfying the containment test at a checkpoint; these are
  table-entry fractions, not visitation-weighted or cumulative rates.
  Panels (b) and (c) use logarithmic episode axes. Lines connect saved
  checkpoints; panel (a) also includes the known origin $R_0=0$.}
  \label{fig:sparse-reference-experiment}
\end{figure}

\section{Empirical Results}
\label{app:empirical-results}
\subsection{MuJoCo Results}
\paragraph{Results and Analysis.}
We conduct an empirical study of O2O-LSVI on standard MuJoCo AntMaze environments, leveraging D4RL datasets \citep{fu2020d4rl} for offline RL pretraining. In the offline phase, we adopt the Cal-QL \citep{nakamoto2023cal} pretraining strategy and use an ensemble of Q-networks. During online adaptation, we instantiate the confidence interval in the inclusion test~\eqref{eqn:inclusion} using the ensemble standard deviation as the interval width. We compare our method against direct adaptation via Cal-QL \citep{nakamoto2023cal} and the offline RL baselines CQL \citep{kumar2020conservative} and IQL \citep{kostrikov2021offline}. As shown in Table~\ref{tab:results}, our method matches or outperforms these baselines. These final-score comparisons do not isolate online sample efficiency, since the reported offline starting scores differ.
\begin{table}[H]
\centering
\begin{tabular}{c|cc|cc}
\toprule
\multirow{2}{*}{\textbf{Settings}} 
& \multicolumn{2}{c|}{\textbf{Offline RL}} 
& \multicolumn{2}{c}{\textbf{Offline-to-Online Adaptation}} \\
\cmidrule(lr){2-3} \cmidrule(lr){4-5}
& \textbf{CQL} & \textbf{IQL} & \textbf{Cal-QL} & \textbf{Ours} \\
\midrule
Umaze & 94.0$\pm$1.6 & 77.0$\pm$0.7 & 76.8$\pm$7.5$\rightarrow$99.8$\pm$0.4 & 85.8$\pm$3.3$\rightarrow$99.8$\pm$0.4 \\
Medium-Play & 59.0$\pm$11.2 & 71.8$\pm$3.0 & 71.8$\pm$3.3$\rightarrow$98.8$\pm$1.6 & 70.3$\pm$2.0$\rightarrow$99.3$\pm$1.3 \\
Large-Play & 28.8$\pm$7.8 & 38.5$\pm$8.7 & 31.8$\pm$8.9$\rightarrow$97.3$\pm$1.8 & 35.3$\pm$4.0$\rightarrow$98.5$\pm$1.6 \\
\bottomrule
\end{tabular}
\vspace{.5em}
\caption{\textbf{Empirical Results on AntMaze.} We compare our method with existing offline-to-online adaptation and offline RL baselines on AntMaze. Our method achieves performance comparable to or better than existing baselines. Results are reported as D4RL scores averaged over four random seeds.}
\label{tab:results}
\end{table}

\paragraph{Experimental Details.}
We build our method upon the Cal-QL offline pretraining procedure \citep{nakamoto2023cal}. During online fine tuning, we form the confidence interval
\[
[\mathrm{mean}(Q_{\mathrm{ens}})-\mathrm{std}(Q_{\mathrm{ens}}),\ \mathrm{mean}(Q_{\mathrm{ens}})+\mathrm{std}(Q_{\mathrm{ens}})]
\]
from a Q-function ensemble, and use the resulting uncertainty estimate to regulate the Q-function during the SAC policy update \citep{haarnoja2018soft}. The reported experiments use numerical reference-switching implementations, not an explicit enumeration of the finite covers in the theoretical algorithm. The finite-cover and projection safeguards in the revised analysis are not a claim that the original experimental runs were re-executed with these changes. Ensemble intervals are heuristic uncertainty estimates rather than the proved confidence sets. The hyperparameter settings are provided in Table~\ref{tab:hyperparam}. We implement our method based on the CORL repository \citep{tarasov2023corl}, which also serves as the codebase for the offline RL baselines CQL \citep{kumar2020conservative}, IQL \citep{kostrikov2021offline}, and Cal-QL \citep{nakamoto2023cal}.
\begin{table}[h]
    \centering
    \begin{tabular}{c|c}
    \toprule
        Hyperparameter & Value \\
    \midrule
        Ensemble Size & 5 \\
        $\beta$ & 50 \\
        CQL Conservative Coefficient $\alpha$ (Offline \& Online) & 5.0 \\
        Discount factor & 0.99 \\
        Q Network Learning Rate & 3e-4 \\
        Policy Network Learning Rate & 1e-4 \\
        Batch Size & 256 \\
        Soft Target Update Portion & 0.005 \\
        Number of Layers for Q Function & 5 \\
        Number of Layers for Policy Network & 3 \\
        Hidden Dimension (Q Function and Policy) & 256 \\
        CQL Temperature & 1.0 \\
    \bottomrule
    \end{tabular}
    \caption{\textbf{Hyperparameters.} We demonstrate the hyperparameters used in the empirical experiment with the practical implementation of O2O-LSVI.}
    \label{tab:hyperparam}
\end{table}

\subsection{Experimental Result with Linear MDP}
\label{sec:sparse-reference-experiment}

\paragraph{Experimental setup.}
We compare O2O-LSVI and LSVI-UCB in a linear MDP with $10$ states,
$2$ actions, horizon $H=3$, and initial state $0$. Rewards are Bernoulli
with mean
\begin{equation}
  r_h(s,a)=\frac12+\frac{\gamma}{2}\sigma_h(2a-1),
  \qquad \gamma=0.10.
  \label{eq:sparse-true-reward}
\end{equation}
The stage signs $\sigma_h\in\{-1,+1\}$ are fixed within each run.
Transitions follow stage-dependent, nonuniform distributions independent
of the current state and action, giving a fixed action gap of $0.10$
throughout the experiment.

We construct an oracle reference by subtracting $0.44$ from the rewards
at the three reachable suboptimal entries with lowest state occupancy,
denoted by $\mathcal D$, and planning with the true transitions. This
preserves the optimal policy and gives
\begin{equation}
  V_h^{\mathrm{ref}}(s)=V_h^*(s),\qquad
  Q_h^*(s,a)-Q_h^{\mathrm{ref}}(s,a)
  =0.44\,\mathbf 1\{(h,s,a)\in\mathcal D\}.
  \label{eq:sparse-reference-identity}
\end{equation}
The $60$ entries comprise $30$ exact optimal, $27$ exact suboptimal,
and $3$ misspecified suboptimal references, giving a table-entry fraction of $0.05$, not the occupancy coefficient $\rho$. Here the policy-uniform coefficient is
$\rho=\max_h\sum_s p_h(s)\ind\{\exists a:(h,s,a)\in\mathcal D\}$,
where $p_h$ is the stage-$h$ state distribution. Its numerical value requires the actual transition probabilities and is not inferred from the entry count.
Both learners use the same norm-bounded, $d=5$ feature map: two base
coordinates augmented with three sparse indicators. True rewards,
reference rewards, and transitions remain exactly linear in this map.
Only O2O-LSVI additionally receives $Q^{\mathrm{ref}}$.

Both learners start without online samples and use stagewise ridge
regression with $\lambda=1$, remaining-horizon upper clipping, and
the common confidence multiplier
\begin{equation}
  \alpha_k=0.15H
    \sqrt{d\log\!\left(\frac{2dH\max\{k,2\}}{0.05}\right)}.
  \label{eq:sparse-bonus}
\end{equation}
The effective scale is $0.15=0.5\times0.3$. O2O-LSVI uses the reference
only when its entire pessimistic--optimistic interval lies within
$Q^{\mathrm{ref}}\pm0.20$, otherwise retaining its optimistic estimate.
Thus the containment parameter is $\beta_{\mathrm{ref}}=0.40$, while
the reference error is $1.10\beta_{\mathrm{ref}}=0.44$.

We run $K=10^6$ episodes per algorithm for each of five seeds
($3\times10^6$ steps per algorithm per seed), using paired common
random numbers and shared random tie-breaking. We evaluate
cumulative pseudo-regret using exact dynamic programming,
\begin{equation}
  R_k=\sum_{i=1}^k
    \bigl[V_1^*(0)-V_1^{\pi_i}(0)\bigr],
  \label{eq:sparse-regret}
\end{equation}
and report $92$ saved logarithmic checkpoints per seed. Regret curves show
means and pointwise $95\%$ Student-$t$ intervals with four degrees of
freedom; differences are paired within seed before aggregation.

\paragraph{Results and interpretation.}
Figure~\ref{fig:sparse-reference-experiment}(a) shows that O2O-LSVI achieves substantially lower final regret than LSVI-UCB, reducing the mean regret from $1145.2$ to $966.4$ ($15.6\%$), with lower regret across all five seeds. As shown in Figure~\ref{fig:sparse-reference-experiment}(b), this improvement emerges after an initial transient cost: O2O-LSVI initially incurs higher regret but eventually overtakes LSVI-UCB and accumulates regret more slowly thereafter.

Figure~\ref{fig:sparse-reference-experiment}(c) further illustrates the behavior of the reference-switching mechanism. The algorithm eventually accepts all correctly specified reference entries while consistently rejecting the misspecified ones. Notably, exact but suboptimal entries are also retained, which can help reduce optimistic overestimation of competing actions.

Because the reference preserves an optimal greedy policy, a frozen-reference greedy baseline has zero pseudo-regret in this construction. Thus this experiment illustrates certification and reference-switching behavior relative to learning from scratch, not recovery from a suboptimal reference policy. Overall, these results demonstrate a finite-sample benefit of reference switching relative to LSVI-UCB in this fixed-gap setting. However, they should not be interpreted as evidence of an asymptotic improvement, since both regret curves continue to grow and the experiment considers only a small number of seeds and a single reference configuration.

\section{Limitations}
The paper’s theoretical analysis relies on Assumptions~\ref{asm:bellman-completeness} and~\ref{asm:bounded-coverage}. When these assumptions do not hold, our analysis has limited ability to characterize the algorithm’s behavior. Moreover, our algorithm is tailored to settings with a $\beta$-separable structure; without such structure, the proposed method may not enjoy a provable advantage over standard online RL.

\end{document}